\documentclass{article}
\usepackage{iclr2027_conference,times}

\usepackage{amsmath,amsfonts,bm}

 \newcommand{\IKL}{\mathrm{IKL}}
\newcommand{\dd}{\mathrm{d}}

\def\eqref#1{equation~\ref{#1}}

\def\1{\bm{1}}

\DeclareMathAlphabet{\mathsfit}{\encodingdefault}{\sfdefault}{m}{sl}
\SetMathAlphabet{\mathsfit}{bold}{\encodingdefault}{\sfdefault}{bx}{n}

\newcommand{\E}{\mathbb{E}}

\DeclareMathOperator*{\argmin}{arg\,min}

\usepackage{amsmath,amssymb,amsthm}
\newtheorem{proposition}{Proposition}

\newtheorem{lemma}{Lemma} 
\newtheorem{theorem}{Theorem} 
\theoremstyle{remark} 

\usepackage{tikz}
\usetikzlibrary{arrows.meta, positioning, calc, fit, backgrounds}
\usepackage{pgfplots}
\pgfplotsset{compat=1.18}
\usepackage{enumitem}
\usepackage{algorithm}
\usepackage{algpseudocode}
\usepackage{booktabs}

\usepackage{hyperref}
\usepackage{url}
\usepgfplotslibrary{groupplots}
\usepackage{xcolor}
\newif\ifpreprint
\preprinttrue
\ifpreprint
  \newcommand{\rev}[1]{#1}
  \newcommand{\revcolor}{}
\else
  \newcommand{\rev}[1]{{\color{blue}#1}}
  \newcommand{\revcolor}{\color{blue}}
\fi

\title{Reconstruction-Aware Distribution Matching for Inverse Problems with Unpaired Data}

\author{Thomas Ryckeboer\thanks{Equal contribution.} \\
Univ. Grenoble Alpes, Inria, CNRS, Grenoble INP, LJK, Grenoble, France
\AND
Giacomo Meanti\footnotemark[1] \\
MaLGa Center, DIBRIS, Universit\`a degli Studi di Genova, Genoa, Italy
\AND
Julien Mairal\footnotemark[1] \\
Univ. Grenoble Alpes, Inria, CNRS, Grenoble INP, LJK, Grenoble, France
}

\newcommand\vsp{{\vspace*{-0.1cm}}}

\ifpreprint\iclrfinalcopy\fi
\begin{document}

\maketitle
\ifpreprint\lhead{Preprint.}\fi

\begin{abstract}
Recovering the forward operator of an imaging system from unpaired data
avoids both calibration hardware and the clean/degraded pairs that
supervision requires---pairs that, for a real lens, are often impossible to
acquire. Existing unpaired approaches that explicitly estimate the forward operator by distribution matching  
evaluate a candidate only through the measurements it generates, whose distribution should match that of the real ones.
Components suppressed by the operator are barely present in the degraded
images, and therefore barely constrain the operator under such a criterion. Yet
inversion precisely tries to recover them. As a result, two operators that are
nearly indistinguishable as forward models can invert very differently. We
address this by adding a comparison on the clean side: real degraded
images, once restored, should be distributed like clean ones. Restorations are computed by
a differentiable plug-and-play algorithm parameterized by the operator
being learned. On spatially varying PSF
calibration and blind super-resolution, the method recovers more
accurate operators than degraded-side matching alone, and closes much of the
gap to restoration with the true operator.

{\bfseries Code:} \url{https://github.com/inria-thoth/reco-aware-ddm}

\end{abstract}

\section{Introduction}\label{sec:intro}
Typical image inverse problems consist of recovering a clean image~$x$ from a measurement $y = Ax + \varepsilon$, where~$A$ describes the forward operator and~$\varepsilon$ is an additive noise term. 
Many methods assume knowledge of~$A$ or that paired examples $(x,y)$ are available for supervised training. 
In practice, neither assumption is easy to satisfy, even in apparently simple scenarios. For instance, when $A$ is a blur due to optics,
accurate estimation of $A$ typically requires a controlled acquisition setup often combined with specialized displays~\citep{kee2011modeling,bauer2018automatic} for photography or beads for microscopy~\citep{speiser2021deep}. Moreover, the procedure must be repeated for every optical system.  
Paired real data is also difficult to obtain and remains a bottleneck for real-world restoration: 
even for a simple denoising task, reference images are typically constructed from multiple exposures of the same scene, which can introduce brightness changes and subpixel misalignment~\citep{zamir2020cycleisp}. 
Many methods consequently train on synthetic degradation pipelines containing a broad range of corruptions \citep{wang2021realesrgan}. 
These pipelines are, however, hand-designed: there is no guarantee that they cover the acquisition system at hand, and performance drops when they do not.

When $A$ is unknown and paired data is unavailable, two common approaches are blind and unpaired methods. 
\emph{Blind} methods estimate $A$ and $x$ jointly from a single degraded image, using priors on both the clean image and the degradation operator
\citep{schuler2012blind,chung2023blinddps,laroche2024fast}. 
They require only the degraded image~$y$ as input, but the operator prior is fixed before observing the acquisition system and may be a poor match. 
\emph{Unpaired} methods~\citep{yuan2018unsupervised,sim2020optimal,wolf2021deflow,tang2024residual} make a different trade-off: they require a collection $\mathcal{Y}$ of degraded images from the same system, together with a separate collection $\mathcal{X}$ of clean images, and learn a mapping between the two distributions without requiring per-image correspondence. 
In most of these approaches, the degradation is represented implicitly by a network and cannot be plugged into another non-blind solver. 
\citet{meanti2025unsupervised} instead recover an \emph{explicit} forward operator through distribution matching, using conditional flow matching~\citep{lipman2023flow,liu2023flow}. 
The recovered operator can then be combined with any non-blind solver \citep{venkatakrishnan2013plug,zhang2021plug,zhu2023diffpir}, and is itself the quantity of interest in calibration problems. This is the setting we adopt.

Yet, \citet{meanti2025unsupervised} only evaluate an operator through the measurements it produces, matching the distribution of images from ${\mathcal X}$ degraded by a candidate operator $A$ to the real degraded images~${\mathcal Y}$. This can be problematic when the forward operator removes information. Components that are strongly suppressed by $A$ are barely present in the degraded images and therefore provide little information for estimating $A$. 
Yet reconstruction attempts to recover these same components, so an operator error that is difficult to detect on the degraded side can still have a substantial effect after inversion.
\citet{efrat2013accurate} study this sensitivity for super-resolution and show that an accurate blur model can be more important than a sophisticated image prior. 
Thus, two operators that produce similar degraded-image distributions can behave differently when used for reconstruction.
In this paper,  we make operator estimation reconstruction-aware, using only unpaired data.

\vspace*{-0.2cm}
\paragraph{Our approach}
In this work, we improve on the criterion of \citet{meanti2025unsupervised},
which operates only on the degraded-image space, with a second comparison after reconstruction.
This additional term encourages
real measurements restored with the candidate operator to
produce images distributed like clean ones. The two comparisons therefore
evaluate the same operator from opposite sides of the inverse problem: the first
through its forward action, and the second through the procedure
which will be used for downstream reconstruction.

To allow end-to-end gradient-based learning, the restoration algorithm must depend explicitly on the candidate operator and 
be differentiable with respect to it. 
Most plug-and-play algorithms~\citep{venkatakrishnan2013plug} satisfy both requirements: the operator enters explicitly through the data-fidelity step, while a fixed denoiser provides the image prior. 
We use DPIR~\citep{zhang2021plug}, which provides high-quality restoration with a small number of iterations.

Differentiating the reconstruction objective is nevertheless more expensive than differentiating the measurement objective, as it requires backpropagating through an iterative restoration algorithm which involves calling a deep denoiser at each step. 
Backpropagation therefore requires storing the activations of all denoiser evaluations, which is memory intensive.
We use truncated backpropagation instead, retaining gradients only through the final iterations of the solver, following a strategy that has proved effective in other iterative optimization settings~\citep{shaban2019truncated}.

A second difficulty lies in balancing the two objectives because their gradients can have very different scales: In the measurement branch, the operator is applied once to generate a synthetic degradation. 
In the reconstruction branch, it affects every iteration of the restoration algorithm, so small operator changes can be amplified or attenuated as they propagate through the solver. 
Consequently, the relative magnitude of the two gradients can vary substantially during training, making a fixed weight difficult to tune. 
We therefore downweight the reconstruction gradient whenever its norm exceeds that of the measurement gradient, so that it cannot dominate because of its numerical scale.

\vspace*{-0.2cm}
\paragraph{Contributions}
\begin{itemize}[leftmargin=*,itemsep=1pt,topsep=2pt,parsep=0pt]
   \item We show theoretically that degraded-side matching barely constrains how
the operator acts on strongly attenuated image components, even though errors
on these components still affect the reconstruction. The clean-side comparison
remains sensitive to these components and acts as a prior on the operator,
favoring operators whose restorations are distributed most like clean images.
    \item We introduce reconstruction-aware distribution matching for unpaired operator learning. We complement measurement matching with a clean-side comparison obtained by restoring real measurements using a differentiable plug-and-play algorithm parameterized by the operator.
    \item Across blind deblurring, space-varying PSF calibration, and blind super-resolution tasks, reconstruction-aware matching produces more accurate estimates than degraded-side matching alone and substantially reduces the gap to restoration with the true operator.
\end{itemize}

\section{Related Work}\label{sec:related}
\vspace*{-0.1cm}
\paragraph{Learning degradations from unpaired data.}
Unpaired methods learn from a collection of degraded images and a separate
collection of clean ones, without per-image correspondence.
RCOT~\citep{tang2024residual} and 
UOTIP~\citep{lee2026uotip} learn a transport map from measurements to clean
signals, respectively in the blind and non-blind setting.
DeFlow~\citep{wolf2021deflow} uses a conditional normalizing flow, and
cycle-consistency approaches~\citep{zhu2017cyclegan,yuan2018unsupervised,sim2020optimal}
train forward and backward networks jointly; \citet{chen20unsup,maeda2020unpaired} further develop these approaches for blind super-resolution. 
Such implicit models can synthesize training pairs, but they do not
provide an explicit operator that can be plugged into a non-blind solver.
In contrast, our goal is to learn an \emph{explicit} forward operator. A few papers share the same goal:
\citet{lim2020cyclegan} replace one of the two CycleGAN generators by a
linear blur kernel for microscopy applications; \citet{luo2022learning} and \citet{levac2025double} learn
distributions over blur kernels (and noise, for the former) with adversarial
objectives, the latter building on AmbientGAN~\citep{bora2018ambientgan};
and \citet{meanti2025unsupervised} estimate a single
operator by diffusion distribution matching.
In all of them, the operator is evaluated only through the measurements it
generates, never through a reconstruction computed with it.
\rev{For instance, the reconstructions of \citet{lim2020cyclegan} come from an inverse
network, which could produce clean-looking reconstructions even with a poor
kernel estimate. Our method does not train an inverse network: it optimizes the
reconstruction loss directly with respect to the operator
(Section~\ref{sec:exp:ablations}).}

\vspace*{-0.1cm}
\paragraph{Blind restoration from a single image.}
Recent diffusion-based methods perform joint inference over the image and degradation \citep{chung2023blinddps,murata2023gibbsddrm,laroche2024fast,sanghvi2024kernel,li2025blindpnp}; blind super-resolution can similarly estimate the blur kernel from a single image~\citep{bell2019kernelgan,yang2024dynamic,luo2022deep}.
Estimating a high-dimensional  operator, such as a space-varying PSF, is challenging from a single observation with strong priors. 
We instead exploit a collection of images acquired with the same system, which carries more information about the degradation.

\vspace*{-0.1cm}
\paragraph{Reconstruction-aware operator estimation.}
Several works take reconstruction into account for kernel estimation.
\citet{chen2025physics} learn spatially-varying PSFs from calibration patterns, and assess the estimation quality through the quality of a deblurring step. \rev{\citet{sanghvi2022photon} estimate the kernel of a single photon-limited image by back-propagating an unsupervised loss through a non-blind solver.} \citet{lunz2021learned} correct an inaccurate operator using information from reconstruction, and \citet{guan2024modelmismatch} correct model mismatch jointly with reconstruction. 
Yet, these approaches do not address the unpaired learning problem, which is the focus of our work.

\vspace*{-0.1cm}
\paragraph{Distribution matching and differentiable reconstruction.}
Score
distillation~\citep{poole2022dreamfusion} and its variational
extensions~\citep{wang2023prolificdreamer,luo2023diffinstruct} provide
gradients for matching a parameterized distribution to a pretrained generative
prior. \citet{meanti2025unsupervised} adapt these tools to explicit operator
learning with conditional flow
matching~\citep{lipman2023flow,liu2023flow}; we use the same
distribution-matching estimator for both our measurement- and
reconstruction-side losses. Plug-and-play
methods~\citep{venkatakrishnan2013plug,romano2017little} combine an explicit
data-fidelity term with a denoising prior. DPIR~\citep{zhang2021plug}
implements this principle with half-quadratic splitting.
We use DPIR because the forward operator remains explicit throughout
reconstruction, allowing the reconstruction loss to be differentiated with
respect to it.

\section{Method}\label{sec:method}
\definecolor{cleancol}{RGB}{38,84,124}
\definecolor{degcol}{RGB}{166,86,40}
\definecolor{opcol}{RGB}{20,20,20}

\begin{figure}[t]
\centering
\begin{tikzpicture}[
  font=\small,
  box/.style={
    draw, rounded corners=2.5pt, minimum width=26mm, minimum height=15mm,
    align=center, inner sep=3pt, line width=0.7pt},
  cleanbox/.style={box, draw=cleancol, fill=cleancol!7},
  degbox/.style={box, draw=degcol, fill=degcol!7},
  op/.style={-{Stealth[length=2.6mm,width=2mm]}, line width=1pt,
             draw=opcol},
  cmp/.style={{Stealth[length=2.2mm]}-{Stealth[length=2.2mm]}, dashed,
              dash pattern=on 2pt off 1.6pt, line width=0.7pt,
              draw=black!50},
  tie/.style={dotted, line width=0.7pt, draw=black!45},
  oplab/.style={fill=white, inner sep=2pt, font=\small},
  losslab/.style={fill=white, inner sep=2pt, font=\small, align=center},
]

\node[cleanbox] (X) at (0,2.35)
  {clean images\\[1pt] $X\sim P_X$};
\node[degbox] (YX) at (6.2,2.35)
  {synthetic measurements\\[1pt] $Y_\theta=A_\theta X+\varepsilon$};
\node[degbox] (Y) at (6.2,0)
  {real measurements\\[1pt] $Y\sim P_Y$};
\node[cleanbox] (RY) at (0,0)
  {restorations\\[1pt] $R_\theta(Y)$};

\draw[op] (X) -- node[oplab, above] {$A_\theta$}
                node[oplab, below, font=\scriptsize, text=black!60]
                     {forward branch} (YX);
\draw[op] (Y) -- node[oplab, below] {$R_\theta=\mathcal R(\cdot;A_\theta)$}
                node[oplab, above, font=\scriptsize, text=black!60]
                     {reconstruction branch} (RY);

\draw[cmp] (YX) -- node[losslab, right=1pt]
  {$F(\theta)$\\[-1pt]{\scriptsize measurement matching}} (Y);
\draw[cmp] (RY) -- node[losslab, left=1pt]
  {$G(\theta)$\\[-1pt]{\scriptsize reconstruction matching}} (X);

\node[draw, rounded corners=2pt, fill=black!4, draw=black!45,
      inner sep=4pt, align=center, font=\small] (share) at (3.1,1.175)
  {shared operator $\theta$\\[-1pt]
   {\scriptsize fixed reconstruction algorithm}};

\draw[tie] (share.north) -- (3.1,2.35);
\draw[tie] (share.south) -- (3.1,0);

\end{tikzpicture}
\caption{%
  \textbf{Reconstruction-aware operator learning.}
Measurement matching $F$ encourages clean images degraded with $A_\theta$ to reproduce the distribution
of real measurements. With the reconstruction-side objective $G$,
real measurements are restored using the same candidate operator, and the
resulting images are matched to the clean-image distribution. }
\label{fig:teaser}
\end{figure}
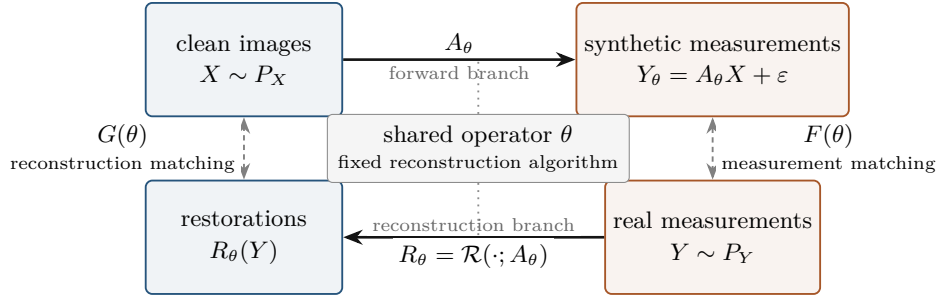

\vsp
\paragraph{Problem setting and formulation.}
We are given independent clean and degraded datasets
$\mathcal X$ and~$\mathcal Y$, with distributions $P_X$ and $P_Y$.
We assume that the clean set represents the latent image distribution underlying
the measurements,
\begin{equation}
    Y=A_{\theta_\star}X'+\varepsilon,
    \qquad
    X'\sim P_X,
    \qquad
    \varepsilon\sim\mathcal N(0,\sigma^2I),
    \label{eq:method:model}
\end{equation}
where $A_{\theta_\star}$ is the unknown forward operator parametrized by $\theta_\star$. The clean image $X'$
that generated a given measurement is never observed, so no pairing between
$\mathcal X$ and $\mathcal Y$ is available.

For a candidate operator $A_\theta$, we can nevertheless generate synthetic
measurements
\begin{equation}
    Y_\theta=A_\theta X+\varepsilon,
    \qquad X\sim P_X, \label{eq:method:synthetic}
\end{equation}
with distribution
$P_{Y,\theta}:=\operatorname{Law}(Y_\theta)$.
Given a distributional discrepancy $\mathcal D$, measurement-side matching
therefore minimizes
\begin{equation}
    F(\theta)
    :=
    \mathcal D\!\left(P_{Y,\theta}\Vert P_Y\right).
    \label{eq:method:F}
\end{equation}
Under the assumed model, $\theta_\star$ is a minimizer of $F$,
although it does not need to be the unique one.

Our main idea is to evaluate the same candidate operator through a 
reconstruction procedure.
Let
 $R_\theta(y):=\mathcal R(y;A_\theta)$
denote a non-blind reconstruction algorithm run with $A_\theta$,
differentiable with respect to $\theta$.
Applying it to real measurements induces
$    Q_\theta
    :=
    \operatorname{Law}\!\left(R_\theta(Y)\right),
    $ for $Y\sim P_Y.$
We then compare these reconstructions with the clean-image distribution,
and optimize
\begin{equation}
   L(\theta)=F(\theta)+\lambda G(\theta) ~~~~\text{with}~~~~G(\theta)
    :=
    \mathcal D\!\left(Q_\theta\Vert P_X\right).
    \label{eq:method:population-objective}
\end{equation}
Thus, $F$ asks whether $A_\theta$ reproduces the observed degradation
statistics, while $G$ asks how the same candidate behaves after inversion.
Figure~\ref{fig:teaser} illustrates the measurement and reconstruction branches.

\vsp
\paragraph{Complementary roles of the two objectives.}
Measurement matching is
tied to the forward model: with the correct operator, synthetic and
real measurements follow the same distribution. The reconstruction objective
plays a different role. It measures how the candidate operator affects the
distribution of the inverse solutions
by comparing restored measurements with clean
images. Even with the correct operator, these two distributions do not need to
coincide exactly, because reconstruction cannot fully recover information lost
   to degradation and noise. For most restorer choices, $G$ is even biased, but it nevertheless encourages a desirable property (the distribution of reconstructions should be similar to that of clean images) and remains sensitive to
operator errors that may be almost invisible in measurement space. We therefore
use $F$ and $G$ as complementary signals: $F$ constrains the forward
model, while $G$ acts as a prior on the components that $F$ barely constrains.

In Section~\ref{sec:theory}, we study the sensitivity of the two objectives to
an operator error on a single image component (say a frequency), with $\rho$ denoting the SNR of
that component after degradation. When $\rho\ll1$, the component is
strongly attenuated relative to the measurement noise and
measurement-side information decays as $\Theta(\rho^2)$, while the effect 
on reconstruction decays only as
$\Theta(\rho)$. 

\subsection{Practical distribution matching with IKL}
\label{sec:method:distribution-matching}

The formulation above is independent of the discrepancy $\mathcal D$.
Following \citet{meanti2025unsupervised}, we instantiate both branches with
   the integrated KL (IKL) loss function. 
   We present below its gradient, whereas the loss itself and further details are
   given in Appendix~\ref{sec:appendix:ikl}.
Specifically, consider matching a parameterized distribution $P_\theta$ to a fixed target
distribution $Q$. Let $Z_1(\theta)\sim P_\theta$ be a sample that is differentiable with respect to $\theta$,
$Z_0\sim\mathcal N(0,I)$, and $Z_t=(1-t)Z_0+tZ_1(\theta)$. Conditional flow
matching represents a distribution $P$ by a velocity field
\begin{equation}
    v_P(z,t)=\E[Z_1-Z_0\mid Z_t=z].
\end{equation}
Then, following~\citet{luo2023diffinstruct}, the gradient of the chosen IKL loss can be written
\begin{equation}
    \nabla_\theta\IKL(P_\theta\Vert Q)
    =
    \int_0^1
    \E\!\left[
        (J_\theta Z_1)^\top
        \bigl(v_{P_\theta}(Z_t,t)-v_Q(Z_t,t)\bigr)
    \right]\dd t ,
    \label{eq:method:ikl-gradient}
\end{equation}
where $J_\theta Z_1$ is the Jacobian of the generated sample with respect to
$\theta$. Following \citet{luo2023diffinstruct,meanti2025unsupervised}, we estimate $v_Q$ with a
\emph{teacher} flow trained once on the target data and then frozen, and
$v_{P_\theta}$ with an \emph{auxiliary} flow trained online on samples from the
current model distribution. The gradient~(\ref{eq:method:ikl-gradient}) is then estimated by sampling $t$ and $Z_0$, see Appendix~\ref{sec:appendix:ikl} for more details.

\vsp
\paragraph{Measurement and reconstruction branches.}
We apply the same estimator twice. For measurement matching,
$Z_1(\theta)=A_\theta X+\varepsilon$: the teacher is trained on real degraded
images and the auxiliary flow tracks the current synthetic measurements. For
reconstruction matching, $Z_1(\theta)=R_\theta(Y)$: the teacher is trained on
clean images and the auxiliary flow tracks the current restorations.

\subsection{Algorithmic Recipe}
\label{sec:method:restoration}
\label{sec:method:rec}

We next describe the main ingredients of our approach, presented in Algorithm~\ref{alg:method}. Appendix~\ref{sec:appendix:centering} details the operator constraints and parametrization, notably fixing the classical translation ambiguity in convolution operator estimation.

\vsp
\paragraph{Differentiable plug-and-play reconstruction}
We instantiate $R_\theta$ with DPIR~\citep{zhang2021plug}, a plug-and-play
half-quadratic splitting algorithm whose iterations alternate a data-fidelity
update involving $A_\theta$ with a fixed pretrained denoiser. Because the
restorer and its schedule are fixed, the reconstruction branch depends on
$\theta$ only through $A_\theta$. We run
$K$ DPIR iterations and differentiate only through the final
$K_{\rm bp}\leq K$, reducing activation memory from $O(K)$ to
$O(K_{\rm bp})$~\citep[see][]{monga2021algorithm,shaban2019truncated}. \rev{With $K_{\rm bp}=1$, this is one-step differentiation~\citep{bolte2023onestep},
a cheap estimate of the full derivative. In our experiments, longer horizons
did not improve the kernels, and even degraded them
(see Section~\ref{sec:exp:ablations}).} Additional
details about the solver and implementation are given in Appendix~\ref{sec:method:details}.

\vsp
\paragraph{Balancing the two stochastic gradients}
The theory motivates combining the two branches but does not determine their
relative scale. At operator-update
step $k$, let
\begin{equation}
    g_F \approx \nabla_\theta F(\theta_k),
    \qquad
    g_G \approx \nabla_\theta G(\theta_k),
\end{equation}
denote the stochastic IKL gradient estimates obtained from
Eq.~(\ref{eq:method:ikl-gradient}) for the measurement and reconstruction
branches, respectively. The former differentiates through one application of
$A_\theta$, whereas the latter differentiates through the iterative inverse
solver, so their norms can differ substantially.

Because we did not manage to obtain satisfactory results by using a fixed weighting strategy\rev{~(see Appendix~\ref{sec:appendix:exp:sensitivity})}, we use a gradient-norm balancing heuristic and
update the operator by using the gradient estimate
\begin{equation}
   g_k=g_F+\lambda\,c_k\,g_G, ~~~~\text{with}~~~~
    c_k
    =
    \min\!\left(
        1,
        \frac{\|g_F\|_2}{\|g_G\|_2+\epsilon}
    \right),
    \label{eq:method:balance}
\end{equation}
where $c_k$ normalizes the scale of the reconstruction gradient. The rule is
asymmetric: it downweights the reconstruction branch when its gradient is
larger than the measurement gradient, but never amplifies it above its raw
scale. Thus $G$ cannot dominate merely because differentiation through the
solver produces larger gradients, while $F$ retains its role as the
forward-model anchor. This is related in spirit to gradient normalization for
multi-loss optimization~\citep{chen2018gradnorm}.

In practice, we use $\lambda=1$, which ensures
$c_k\|g_G\|\le\|g_F\|$, hence $\langle g_F+c_k g_G,\, g_F\rangle\ge 0$: the combined
update never opposes the gradient of $F$, and $\lambda=1$ is the largest weight
with this property. The clean-side term can thus steer the descent of $F$, but never
trade data fit for its own preference. The update therefore does not minimize
Eq.~(\ref{eq:method:population-objective}) with a fixed predefined weight: when $c_k<1$, it is the gradient of
$F+c_kG$, with a weight $c_k\le 1$ that adapts during training.
With exact gradients, its stationary points are the points where
$\nabla F+\mu\nabla G=0$ for some $\mu\in[0,1]$.

\begin{algorithm}[t]
\caption{Reconstruction-aware operator learning.}
\label{alg:method}
\begin{algorithmic}[1]
  \Require clean set $\mathcal X$, degraded set $\mathcal Y$, frozen teachers
    $v_X^{\rm tc},v_Y^{\rm tc}$, reconstruction steps $K$.
  \State initialize $\theta$ and auxiliary fields
    $v_X^{\rm aux},v_Y^{\rm aux}$
  \For{$k=1,\ldots,N$}
      \State sample $X\sim\mathcal X$, $Y\sim\mathcal Y$ and $\varepsilon$
      \State update $v_Y^{\rm aux}$ on $A_\theta X+\varepsilon$ with $\theta$ fixed
      \State get new samples $X\sim\mathcal X$, $Y\sim\mathcal Y$ and $\varepsilon$
      \State estimate $g_F\approx\nabla_\theta F(\theta_k)$ using
      Eq.~(\ref{eq:method:ikl-gradient}) with
      $Z_1(\theta)=A_\theta X+\varepsilon$
      \State compute $R_{\theta_k}(Y)$ with $K$ DPIR iterations and estimate
      $g_G\approx\nabla_\theta G(\theta_k)$ using
      Eq.~(\ref{eq:method:ikl-gradient}).
   \State update $v_X^{\rm aux}$ on $R_\theta(Y)$ detached with $\theta$ fixed
    \State $c_k\gets
      \min\!\left(1,\|g_F\|_2/(\|g_G\|_2+\epsilon)\right)$
    \State update $\theta$ with a stochastic gradient descent solver (\emph{e.g.}, Adam) using $g_F+\lambda\,c_k\,g_G$
    \State enforce operator constraints
      (Appendix~\ref{sec:appendix:centering})
  \EndFor
  \State \Return $A_\theta$
\end{algorithmic}
\end{algorithm}

\vsp
\paragraph{Practical heuristics.}
Algorithm~\ref{alg:method} presents the method in its idealized form. A few heuristics are used in all our experiments and are described in
Appendix~\ref{sec:method:details}. Notably, the weight $\lambda$ is ramped
linearly from $0$ to $1$ over an initial fraction of training: early on the operator
is close to its initialization, so its reconstructions are poor and the
clean-side gradient says little about $\theta$. Then, we report an exponential
moving average of the operator parameters rather than the last iterate, which
removes the step-to-step fluctuation of a low-dimensional parameter driven by
stochastic gradients.

\section{Theory}
\label{sec:theory}
\vsp
\paragraph{Setup: Sensitivity analysis for one component.}
A forward operator can transmit different image components (\emph{e.g.}, a spatial frequency) with different
strengths. 
Here, we consider one scalar component and study the effects of the
loss functions~$F$ and~$G$ when estimating the operator's action on this
component, as well as the effect of an estimation error on the reconstruction.

After rescaling the clean component and the measurement noise to unit variance,
we write its observation as
\begin{equation}
    Y_\rho=\sqrt{\rho}\,X+N,
    \qquad
    N\sim\mathcal N(0,1),
    \qquad
    \mathbb E[X]=0,\quad \mathbb E[X^2]=1,
    \label{eq:theory:channel}
\end{equation}
where $\rho > 0$ is the SNR of component $X$ after degradation.
When $\rho\ll1$, the operator has attenuated the component such
that its contribution to the measurement is small relative to noise.

\vsp
\paragraph{Sensitivity analysis for estimating $\rho$.}
We define a perturbation of the component gain $\rho$ by
\begin{equation}
    \rho_\delta=\rho(1+\delta),\qquad \delta>-1.
    \label{eq:theory:perturbation}
\end{equation}
This is a relative perturbation: $\delta$ represents the
fractional  error regardless of how visible the component is.
For the analysis below, we consider the KL loss function for the measurement side:
\begin{equation}
    F_\rho(\delta)
    =
    \operatorname{KL}
    \!\left(P_\rho\Vert P_{\rho(1+\delta)}\right)~~~~\text{with}~~~~P_\gamma
    :=
    \operatorname{Law}(\sqrt{\gamma}X+N).
    \label{eq:theory:F}
\end{equation}

To define a loss function on the reconstruction side, we must define a reconstruction method. We use here the posterior mean
which represents an ``ideal'' MMSE restorer:
\begin{equation}
    \eta_\gamma(y)
    =
    \mathbb E[X\mid \sqrt{\gamma}X+N=y].
\end{equation}
Its reconstruction risk is then
\begin{equation}
    \mathcal R_\rho(\delta)
    =
    \mathbb E\!\left[
      \bigl(X-\eta_{\rho(1+\delta)}(Y_\rho)\bigr)^2
    \right].
    \label{eq:theory:risk}
\end{equation}
Finally,
we consider the reconstruction-side objective using the Wasserstein-2 loss function:
\begin{equation}
    G_\rho(\delta)
    =
    W_2^2(Q_{\rho,\delta},P_X),
    \qquad
    P_X=\operatorname{Law}(X)~~~\text{and}~~~Q_{\rho,\delta}
    =
    \operatorname{Law}
    \!\left(\eta_{\rho(1+\delta)}(Y_\rho)\right).
    \label{eq:theory:G}
\end{equation}
Note that we use ordinary KL, $W_2$, and posterior-mean reconstruction because they make
the sensitivity analysis tractable, whereas the practical method uses IKL and
DPIR.
Under this choice of loss functions~$F_\rho$ and $G_\rho$, we can now state our main sensitivity results for low-SNR components.
\begin{theorem}[Sensitivity of poorly observed components]
\label{thm:theory:weak}
Assume that $X$ is bounded. As $\rho\to 0$,
\begin{equation}
    F_\rho''(0)
    =
    \frac{\rho^2}{2}+o(\rho^2),
    \label{eq:theory:F-curvature}~~~~\text{and}~~~~
    \left.
    \frac{d^2}{d\delta^2}
    \bigl[\mathcal R_\rho(\delta)-\mathcal R_\rho(0)\bigr]
    \right|_{\delta=0}
    =
    \frac{\rho}{2}+o(\rho).
\end{equation}
Moreover, there exists a constant $\kappa_X\in(0,1)$, depending only on
$P_X$, such that
\begin{equation}
    G_\rho'(0)
    =
    -\kappa_X\sqrt{\rho}+O(\rho),
    \qquad
    G_\rho''(0)
    =
    \frac{\kappa_X}{2}\sqrt{\rho}+O(\rho).
    \label{eq:theory:G-derivatives}
\end{equation}
\end{theorem}
Appendix~\ref{app:theory} provides the proof of Theorem~\ref{thm:theory:weak}.

\vsp
\paragraph{Interpretation.}
Since $F_\rho$ is minimized at the true component gain ($\delta=0$), its curvature
$F_\rho''(0)$ measures the local sensitivity of measurement matching to a
relative operator error. More precisely, it is the Fisher information about
$\delta$ contained in one degraded observation. Thus, as $\rho$ decreases, the operator's action on that component
rapidly becomes difficult to estimate from degraded data alone.
Its importance for reconstruction decreases more slowly. The excess MSE caused
by the same relative operator error has curvature $\Theta(\rho)$, whereas the
measurement information is only $\Theta(\rho^2)$. Hence the degraded data can
become weakly informative about an operator component while errors in that
component still matter after inversion. 

The clean-side objective $G_\rho$ remains sensitive in this regime, with slope
and curvature of order $\sqrt{\rho}$ but this requires a cautious interpretation. 
$G$ is sensitive in the low-SNR regime, but it is biased since like any point-estimate restorer, including DPIR, the posterior mean does not
produce images distributed like clean ones, even with the true operator.
\rev{It is therefore not sufficient by itself, but our analysis shows that it
is informative where $F$ is nearly flat.}
$G$ should thus be seen as a prior on components that the degraded data barely
constrain. It acts as a selection criterion among operators that $F$ finds nearly equivalent,
encouraging solutions with a desirable property: reconstructions
should be distributed like clean images, leading in practice to better
estimation~(see Section~\ref{sec:exp}).

\rev{%
\vsp
\paragraph{From the scalar model to the practical objective.}
Recall that our method uses IKL and DPIR instead of KL and the posterior mean.
In Appendix~\ref{app:theory:practice}, we show that the measurement-side result
also holds for IKL: the curvature of $F_\rho$ is still of order $\rho^2$
(Proposition~\ref{prop:app:ikl}). We also show that the data-fidelity step of
DPIR, which is the only step involving the operator, behaves like the posterior
mean for weakly transmitted frequencies.}

In Appendix~\ref{subsec:unbiased_exp}, we study an unbiased but more expensive
variant, comparing restorations of real measurements with restorations of clean
images degraded by the candidate operator. It performs comparably, suggesting that
the bias of $G$ is not harmful.

\section{Experiments}\label{sec:exp}\label{sec:experiments}
\newcommand{\tbd}{\textsc{tbd}}
\newcommand{\best}[1]{\textbf{#1}}

We evaluate reconstruction-aware matching on four problems of increasing
realism. We follow the unpaired setting of~\citet{meanti2025unsupervised}, but
make it harder in every respect: nine uniform degradations instead of two, a
measured lens PSF field with strong chromatic aberration, blind
super-resolution without any high-resolution reference, and a real lens
calibrated from only fifteen photographs of a single scene, a simple setup that
does not require extensive data acquisition.

\vsp
\paragraph{Setup.}
The main controlled comparison is with degraded-side matching alone, i.e., the
approach of~\citet{meanti2025unsupervised} implemented in our pipeline with the
reconstruction branch disabled ($\lambda=0$). The configuration of this
\emph{measurement-only} baseline is fixed first and reused unchanged for our
method ($\lambda=1$), so that the reconstruction branch is the only
difference. All experiments share the same optimization recipe: the optimizer,
learning rates, batch size, network architectures, reconstruction weight
$\lambda=1$, and solver settings are identical across tasks, and the
task-specific settings (Table~\ref{tab:appendix:settings}) mainly concern the
operator and the data, such as the kernel size, boundary handling, and noise
level. We report the kernel error, i.e., the MAE between the estimated and
true kernels up to the translation ambiguity
(Appendix~\ref{sec:appendix:centering}), and the quality
of the restorations obtained with DPIR~\citep{zhang2021plug}, which measures
what an estimate is worth after inversion; restoration with the true operator
serves as \rev{a reference}. \rev{On FFHQ, this evaluation uses a diffusion denoiser trained on
FFHQ, whereas operator learning uses DRUNet.} All results use the last checkpoint. All experiments run on a single RTX 6000 Ada GPU or an RTX 6000 PRO.
Details are given in Appendix~\ref{sec:appendix:exp}. Our code and the scripts to reproduce the results are provided
as supplementary material.

\vsp
\paragraph{Summary of results.}
In all three quantitative settings, reconstruction-aware matching yields more
accurate operators and better restorations than measurement-only matching,
with the largest gain for the most complex operator: it closes $74\%$ of the
PSNR gap to the true operator on uniform blur, and turns a failed estimate of
the chromatic PSF field ($23.9$~dB) into one within $3.1$~dB of the true field.
In single-image blind super-resolution, it outperforms state-of-the-art methods
specialized for this problem in PSNR, with a gap to the true operator about
$40\%$ smaller than that of the strongest one.

\subsection{Spatially uniform blur}
\label{sec:exp:uniform}

We consider nine degradations of FFHQ-256 images
(Appendix~\ref{sec:appendix:kernels}): three irregular motion blurs, four
analytic PSFs, and two channel-dependent operators. The clean and degraded training sets contain $1000$
images each, drawn from disjoint partitions, with noise level $\sigma=0.02$.
Each operator is estimated as a free $33\times33$ kernel, with one kernel per
channel for the channel-dependent operators.

Reconstruction-aware matching reduces the mean kernel error by $28\%$
($2.16\rightarrow1.56$), reduces the PSNR gap to the true operator from $0.61$
to $0.16$~dB, and improves LPIPS accordingly
(Table~\ref{tab:exp:uniform-summary}, left). The gain is systematic
rather than driven by a few kernels: restoration improves on eight of the nine
degradations (paired sign-flip test, $p=0.008$), and comes within $0.1$~dB of
the true operator on five of them.
The only exception is the broad $\sigma=3$ Gaussian, for which both estimates
lose about $0.3$~dB. Per-degradation results are given in
Appendix~\ref{sec:appendix:exp:uniform}. \rev{We also ran the single-image blind method
BlindDPS~\citep{chung2023blinddps} on $100$ test images, where our method is
better on every image, by $4.7$~dB (Gaussian) and $7.7$~dB (motion) on average
(Appendix~\ref{sec:appendix:exp:blinddps}).}

\begin{table}[t]
\caption{\textbf{Deblurring.} Uniform blur on FFHQ-256 (mean over the nine
degradations) and measured chromatic PSF field ($8\times8$ grid) on the $76$
DPDD test images. Kernel error: aligned kernel MAE ($\times10^{-4}$); PSNR and
LPIPS; SSIM is reported in
Appendix~\ref{sec:appendix:exp}. Restormer and OmniLens are supervised
single-image restoration networks: they output a restored image rather than an
operator (hence the dashes), are applied directly to the degraded images and are run on the PSF field only.}
\label{tab:exp:uniform-summary}
\label{tab:exp:nu}
\centering
\small
\setlength{\tabcolsep}{5pt}
\begin{tabular}{@{}lrrrcrrr@{}}
\toprule
& \multicolumn{3}{c}{Uniform blur (nine kernels)}
& & \multicolumn{3}{c}{Measured chromatic PSF field} \\
\cmidrule(lr){2-4}\cmidrule(l){6-8}
Method & Kernel err.\ $\downarrow$ & PSNR $\uparrow$ & LPIPS $\downarrow$
& & Kernel err.\ $\downarrow$ & PSNR $\uparrow$ & LPIPS $\downarrow$ \\
\midrule
Restormer~\citep{zamir2022restormer} & --- & --- & --- & & --- & 28.57 & 0.205 \\
OmniLens~\citep{jiang2024flexible} & --- & --- & --- & & --- & 29.06 & 0.203 \\
\midrule
Measurement-only & 2.16 & 31.57 & 0.110 & & 8.68 & 23.85 & 0.207 \\

Ours & \best{1.56} & \best{32.02} & \best{0.106}
     & & \best{5.52} & \best{32.41} & \best{0.132} \\
\midrule
True operator & --- & 32.18 & 0.105 & & --- & 35.52 & 0.111 \\
\bottomrule
\end{tabular}
\end{table}

\subsection{Measured PSF field with chromatic aberration}
\label{sec:exp:nu}

We next estimate the PSF field of a real wide-aperture lens, a Canon EF
24\,mm f/1.4L USM at f/1.4 measured by~\citet{bauer2018automatic}, modeled as
an $8\times8$ grid of $27\times27$ RGB kernels. Towards the corners, the PSFs become broader, elongated, with strong chromatic aberrations 
(Appendix~\ref{sec:appendix:kernels:nu}). The
two unpaired sets come from different corpora: measurements are built
from DIV2K~\citep{agustsson2017ntire}, and clean images come from DPDD~\citep{abuolaim2020defocus}.

Here, measurement-only matching performs poorly ($23.85$~dB, Table~\ref{tab:exp:nu}, right),
whereas our approach
reaches $32.41$~dB, within $3.1$~dB of the true field.
Figure~\ref{fig:psf-field-pair} in Appendix~\ref{sec:appendix:exp:nu} shows
why. The measurement-only
PSFs are over-smoothed, in particular along the borders of the field, where the PSF is broadest.
This is consistent with Section~\ref{sec:theory}: a broad PSF strongly
attenuates many components, which the degraded data barely constrain and on
which the reconstruction-side prior acts.

Table~\ref{tab:exp:nu} also reports two supervised networks that
restore the image directly, without estimating an operator:
Restormer~\citep{zamir2022restormer}, applied with its
defocus-deblurring model, which was trained on DPDD, and
OmniLens~\citep{jiang2024flexible}, trained on the PSF fields of a large
library of automatically designed lenses. They reach $28.57$ and $29.06$~dB, respectively.
Restormer targets a different blur, defocus rather than lens aberrations, but
OmniLens is trained precisely for lens aberrations; neither matches a
calibrated operator for this lens.

\subsection{Blind super-resolution}
\label{sec:exp:sr}

We evaluate blind $\times2$ super-resolution on the $100$ images of
DIV2KRK~\citep{bell2019kernelgan}, each degraded by a different anisotropic
kernel that must be estimated from the image alone. Following~\citet{bell2019kernelgan}, crops of the
low-resolution image serve as degraded samples and larger crops of the same
image as clean ones, under the assumption of cross-scale patch recurrence, so
that neither the true kernel nor any high-resolution image is used for training. We compare
with three state-of-the-art methods specialized for this problem:
KernelGAN~\citep{bell2019kernelgan}, the supervised DCLS~\citep{luo2022deep},
and DIP-DKP~\citep{yang2024dynamic}. Our method, by contrast, is not
specialized: it uses the same objective and restorer as for deblurring.

\vsp
\paragraph{Comparison setup with a common Gaussian prior.}
All three baselines favor compact, near-Gaussian kernels: DIP-DKP draws its
candidate kernels from an anisotropic Gaussian family, KernelGAN penalizes
non-compact and off-center kernels, and DCLS is trained on synthetic Gaussian
kernels. To factor out the effect of such priors in the evaluation, we apply the same Gaussian
prior to every method, ours included: each estimated kernel is centered, its
border is zeroed, and it is replaced by the anisotropic Gaussian with the same
first two moments, before restoration with the same DPIR solver. This
post-processing never
degrades a method on average. It also explains why our KernelGAN results are
higher than those reported by~\citet{bell2019kernelgan}. Results without post-processing, where our
method also yields the best PSNR, are given in
Appendix~\ref{sec:appendix:exp:sr:full}.

\vsp
\paragraph{Results.}
Although not designed for super-resolution, our method achieves the best mean
PSNR (Table~\ref{tab:exp:sr}), $1.04$~dB above measurement-only matching, with
a gap to the true operator of $0.6$~dB against $1.1$~dB for the strongest
baseline,
DIP-DKP. It outperforms KernelGAN and DCLS on $82$ and $97$ of the $100$
images, respectively. Against DIP-DKP, the per-image comparison is close to
even ($55/100$), but the mean gain is statistically significant ($+0.45$~dB,
$p=2\cdot10^{-4}$ two-sided
paired permutation test): our advantage comes from avoiding failures rather than from
winning most images. Consistently, the errors of the baselines grow with the
sharpness of the true kernel, whereas ours depends only weakly on its shape.

\begin{table}[t]
\caption{\textbf{Blind super-resolution on DIV2KRK ($\times2$)}. All kernels receive the same centering and Gaussian
post-processing and are restored with DPIR. Kernel error:
aligned MAE ($\times10^{-4}$). Bold denotes
best entry and those not significantly different from it according to a paired permutation
   test.}
\label{tab:exp:sr}
\centering
\small
\begin{tabular}{lrrrrr}
\toprule
& Kernel err. $\downarrow$ & PSNR-Y $\uparrow$ & PSNR $\uparrow$ & SSIM $\uparrow$ & LPIPS $\downarrow$ \\
\midrule
KernelGAN                 & 3.42 & 34.22 & 32.46 & 0.8481 & 0.2714 \\
DCLS                      & 6.97 & 33.59 & 31.84 & 0.8160 & 0.3129 \\
DIP-DKP                   & 2.57 & 34.54 & 32.77 & \best{0.8575} & \best{0.2599} \\
\midrule
Measurement-only          & 2.96 & 33.95 & 32.20 & 0.8459 & 0.2677 \\
Ours                      & \best{2.15} & \best{34.99} & \best{33.21} & \best{0.8561} & \best{0.2601} \\
\midrule
True operator             & ---  & 35.62 & 33.81 & 0.8724 & 0.2435 \\
\bottomrule
\end{tabular}
\end{table}

\subsection{Proof of concept for real lens calibration}
\label{sec:exp:real}

Finally, we calibrate a real lens from fifteen photographs of a single textured
outdoor scene, taken with a consumer camera at f/16 (seven frames that are blurry due to diffraction, forming
$\mathcal Y$) and at f/5.6 (eight center crops, which are sharp with no chromatic aberrations, forming $\mathcal X$), without
paired images. The camera is reframed
and rotated between shots, including upside-down, so that the two sets share no
image correspondence and no orientation is tied to either aperture
(Appendix~\ref{sec:appendix:exp:real}).

\rev{With a noise level calibrated from raw data
(Appendix~\ref{sec:appendix:exp:noise-est}), we} estimate a $9\times12$ grid of ${13\times13}$ RGB kernels, all treated as free
parameters. This generic model ignores that \rev{nearby PSFs} are
strongly dependent, and that accurate models exist, \emph{e.g.}, for lateral
chromatic aberration~\citep{kang2007automatic}. Yet, without any
such model, our method recovers a PSF field \rev{whose blur is consistent
with diffraction theory (Appendix~\ref{sec:appendix:exp:real-mtf})}.
\rev{Without ground truth}, we compare
restorations with two single-image methods: the method
of~\citet{eboli2022fast}, which estimates local Gaussian PSFs and corrects
lateral chromatic aberration, and OmniLens~\citep{jiang2024flexible},
already used in Section~\ref{sec:exp:nu}. \rev{Results are shown} in Figure~\ref{fig:exp:real} and Appendix~\ref{sec:appendix:exp:real}.
\rev{Full images are in the} supplementary material.

\newcommand{\realcropw}{0.5}%
\newcommand{\realcroph}{0.5}%
\newlength{\realpanelwidth}
\newcommand{\realpanel}[2]{%
  \pgfmathsetmacro{\rpllx}{5.76+77.76*#1+36*(1-\realcropw)}%
  \pgfmathsetmacro{\rplly}{161.28-77.76*#2+36*(1-\realcroph)}%
  \pgfmathsetmacro{\rpurx}{\rpllx+72*\realcropw}%
  \pgfmathsetmacro{\rpury}{\rplly+72*\realcroph}%
  \edef\rpnext{\noexpand\includegraphics[width=\noexpand\realpanelwidth,%
    viewport=\rpllx\space\rplly\space\rpurx\space\rpury,clip]{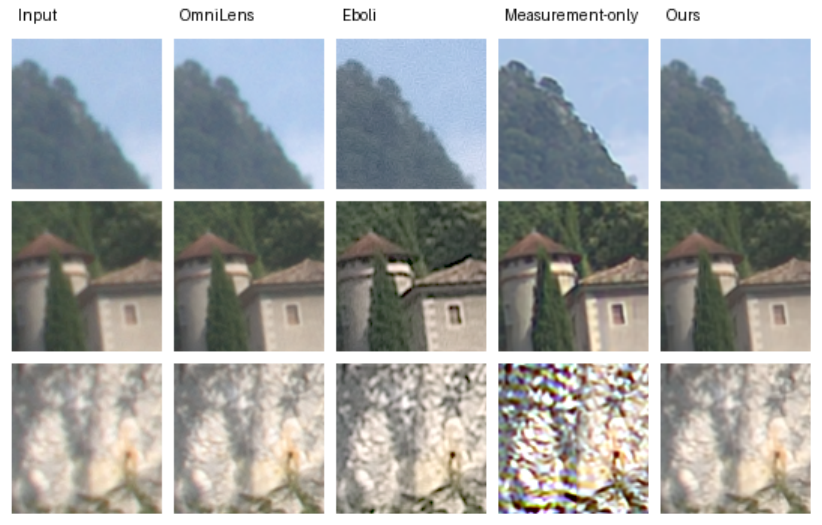}}%
  \rpnext}
\begin{figure}[t]
\centering
\setlength{\realpanelwidth}{0.195\textwidth}%
\setlength{\tabcolsep}{0pt}%
\renewcommand{\arraystretch}{0}%
   \scalebox{0.95}{
\begin{tabular*}{\textwidth}{@{\extracolsep{\fill}}ccccc@{}}
\footnotesize Input & \footnotesize OmniLens & \footnotesize Eboli et al. & \footnotesize Measurement-only & \footnotesize Ours \\[3pt]
\realpanel{0}{0} & \realpanel{1}{0} & \realpanel{2}{0} & \realpanel{3}{0} & \realpanel{4}{0} \\[2.5pt]
\realpanel{0}{1} & \realpanel{1}{1} & \realpanel{2}{1} & \realpanel{3}{1} & \realpanel{4}{1} \\[2.5pt]
\realpanel{0}{2} & \realpanel{1}{2} & \realpanel{2}{2} & \realpanel{3}{2} & \realpanel{4}{2}
\end{tabular*}
   }
   \caption{\textbf{Real lens calibration.} Chromatic aberrations are best removed by~\citet{eboli2022fast} but this method does not perform as well as others for deblurring, sometimes producing \rev{over-sharpening artefacts} (middle row) or residual noise artefacts (top row). Measurement-only sometimes fails (\rev{see colored artefacts, bottom row}). OmniLens is close to our approach, with slightly more residual chromatic aberrations (\rev{see green fringes}, top row). \rev{Best seen by zooming on a computer screen.}}
\label{fig:exp:real}
\end{figure}

\rev{%
\subsection{Ablations and robustness}
\label{sec:exp:ablations}

\vsp
\paragraph{Clean-side matching versus cycle consistency.}
To test whether cycle consistency~\citep{zhu2017cyclegan,lim2020cyclegan}
could provide the same gain, we replace $G$ by the cycle loss
$\|A_\theta R_\theta(Y)-Y\|^2$, whose gradient is rescaled to norm
$\lambda\|g_F\|$, and keep everything else unchanged. This variant
is worse than our method in every setting, with kernel errors $2$ to $11$ times
larger, and its PSNR is below that of measurement-only matching in eight of the
ten settings (Appendix~\ref{sec:appendix:exp:cycle}). On the PSF field, its PSNR
is close to ours ($32.30$ against $32.41$~dB), but it does not recover the
operator: its kernel error is larger than that of measurement-only matching
($10.99$ against $8.68$), and its LPIPS is much worse than ours ($0.186$
against $0.132$). Indeed, the cycle loss rewards restorations that reproduce the
measurement, including its noise, which is easier with operators that blur less
than the true one. With such operators, the restorations remain blurry, which
$G$ penalizes.

\vsp
\paragraph{Robustness.}
The gains are stable across seeds and initializations
(Appendix~\ref{sec:appendix:exp:sensitivity}): with three seeds on five
settings, our method improves the PSNR in $14$ of the $15$ runs, and the
standard deviation of its PSNR across seeds is at most $0.08$~dB.
With the learning rates and the balancing rule chosen for $K_{\rm bp}=1$,
differentiating through more solver iterations does not help: on the
$\sigma=1$ Gaussian, the kernel error increases from $0.87$ to $1.34$ and
$1.56$ for $K_{\rm bp}=2$ and $4$, and the PSNR decreases. Replacing the
balancing rule and its warm-up by a constant weight either loses part of the
gain (small weight) or diverges on \texttt{mbcomplex} (larger weights). Without retraining, the learned operators keep
their advantage with other solvers, e.g., DiffPIR on all nine kernels and on
the PSF field ($+7.4$~dB; Appendix~\ref{sec:appendix:exp:solvers}). Finally,
with a noise level wrong by a factor of two, training is
stable and our method still gives better restorations than measurement-only
matching, but a less accurate kernel when the level is underestimated
(Appendix~\ref{sec:appendix:exp:noise}); it can be estimated within $1\%$
(Appendix~\ref{sec:appendix:exp:noise-est}).%
}

\section{Concluding Remarks}\label{sec:ccl}
We have introduced a new approach for learning explicit imaging operators from
unpaired data. By evaluating the same candidate operator both before and after
inversion, the method complements degraded-side matching where, as our analysis
shows, the measurements weakly constrain the forward model. Across various operator learning tasks, this improves
estimation and typically reconstruction quality, with the largest
gains in the high-dimensional space-varying setting. The main limitations are
the calibration-time cost~(about $\times 6$, Appendix~\ref{sec:appendix:exp:analysis})\rev{,} the dependence on the
\rev{restorer defining} the clean-side prior\rev{, although the gains hold with other solvers at test time
(Appendix~\ref{sec:appendix:exp:solvers}), and a
Gaussian noise model, which sufficed on our real capture once
calibrated (Appendix~\ref{sec:appendix:exp:noise-est})}.

\subsubsection*{Acknowledgments}
This work was supported by ERC grant number 101087696 (APHELEIA project). This work was
granted access to the HPC resources of IDRIS under the allocation AD011015445 made by GENCI.

\subsection*{AI use statement}

We used generative AI tools, including ChatGPT and Claude, during the
development and preparation of this work. In accordance with the ICLR 2027 AI
Policy for Authors, we disclose their use in several parts of the research
workflow.

Generative AI was used to assist with the implementation of the method and the
experimental workflow. An initial version of the experimental codebase was
written without generative AI assistance; Claude Code was subsequently used to
help modify, debug, and improve the implementation and to assist with running
experiments. We also used LLMs to discuss and refine aspects of the methodology
and theoretical analysis, including checking derivations and improving the
presentation of mathematical arguments.

We additionally used ChatGPT and Claude for tasks for which disclosure is
recommended by the policy: drafting and editing portions of the manuscript,
improving clarity and organization, reviewing the paper for missing
reproducibility details, and searching for potentially relevant prior work.
LLM-assisted literature searches were used as a discovery tool only; relevant
references and claims were subsequently checked by the authors against the
original papers.

All AI-assisted code was reviewed and tested by the authors, and experimental
results were obtained from the resulting implementation rather than generated
by an LLM. Mathematical statements and derivations suggested or revised with
LLM assistance were checked by the authors. The authors reviewed all
AI-assisted text and remain responsible for the final content of the paper,
including its code, experiments, mathematical claims, citations, and
conclusions.

\subsection*{Ethics statement}

This work develops a method for estimating imaging operators and solving image
inverse problems from unpaired data. It does not involve human-subject
experiments or the collection or release of new personal data. Our experiments
use existing image datasets and pretrained models. \rev{The face images of
FFHQ are used only as a standard benchmark, as in prior work and under the
license of the dataset, and no face image is shown in the paper. The real
capture shows a landscape and contains no identifiable person.}

As with other image restoration methods, the proposed approach may inherit
biases or limitations present in the image distributions and pretrained
denoisers on which it relies. This consideration is particularly important in
applications in which reconstructed images may be used as evidence or support
important decisions. When deployed in the real world, image restoration methods
such as ours need to be thoroughly evaluated and tested for each practical 
use case.

\subsection*{Reproducibility statement}

The released code includes the configurations and scripts used to produce the
reported experiments and figures.

\bibliography{iclr2027_conference}

\begin{thebibliography}{68}
\providecommand{\natexlab}[1]{#1}
\providecommand{\url}[1]{\texttt{#1}}
\expandafter\ifx\csname urlstyle\endcsname\relax
  \providecommand{\doi}[1]{doi: #1}\else
  \providecommand{\doi}{doi: \begingroup \urlstyle{rm}\Url}\fi

\bibitem[Abdelhamed et~al.(2018)Abdelhamed, Lin, and Brown]{abdelhamed2018sidd}
Abdelrahman Abdelhamed, Stephen Lin, and Michael~S. Brown.
\newblock A high-quality denoising dataset for smartphone cameras.
\newblock In \emph{Conference on Computer Vision and Pattern Recognition
  (CVPR)}, pp.\  1692--1700, 2018.

\bibitem[Abuolaim \& Brown(2020)Abuolaim and Brown]{abuolaim2020defocus}
Abdullah Abuolaim and Michael~S. Brown.
\newblock Defocus deblurring using dual-pixel data.
\newblock In \emph{European Conference on Computer Vision (ECCV)}, 2020.

\bibitem[{Adobe}(2023)]{dng2023}
{Adobe}.
\newblock Digital negative ({DNG}) specification, version 1.7.1.0, 2023.

\bibitem[Agustsson \& Timofte(2017)Agustsson and Timofte]{agustsson2017ntire}
Eirikur Agustsson and Radu Timofte.
\newblock {NTIRE} 2017 challenge on single image super-resolution: Dataset and
  study.
\newblock In \emph{Conference on Computer Vision and Pattern Recognition (CVPR)
  Workshop}, 2017.

\bibitem[Bauer et~al.(2018)Bauer, Volchkov, Hirsch, and
  Sch{\"o}lkopf]{bauer2018automatic}
Matthias Bauer, Valentin Volchkov, Michael Hirsch, and Bernhard Sch{\"o}lkopf.
\newblock Automatic estimation of modulation transfer functions.
\newblock In \emph{IEEE International Conference on Computational Photography},
  2018.

\bibitem[Bell-Kligler et~al.(2019)Bell-Kligler, Shocher, and
  Irani]{bell2019kernelgan}
Sefi Bell-Kligler, Assaf Shocher, and Michal Irani.
\newblock Blind super-resolution kernel estimation using an internal-{GAN}.
\newblock In \emph{Adv. Neural Information Processing Systems (NeurIPS)}, 2019.

\bibitem[Bolte et~al.(2023)Bolte, Pauwels, and Vaiter]{bolte2023onestep}
J{\'e}r{\^o}me Bolte, Edouard Pauwels, and Samuel Vaiter.
\newblock One-step differentiation of iterative algorithms.
\newblock In \emph{Adv. Neural Information Processing Systems (NeurIPS)}, 2023.

\bibitem[Bora et~al.(2018)Bora, Price, and Dimakis]{bora2018ambientgan}
Ashish Bora, Eric Price, and Alexandros~G. Dimakis.
\newblock {AmbientGAN}: Generative models from lossy measurements.
\newblock In \emph{International Conference on Learning Representations
  (ICLR)}, 2018.

\bibitem[Chen et~al.(2015)Chen, Zhu, and Heng]{chen2015efficient}
Guangyong Chen, Fengyuan Zhu, and Pheng~Ann Heng.
\newblock An efficient statistical method for image noise level estimation.
\newblock In \emph{International Conference on Computer Vision (ICCV)}, pp.\
  477--485, 2015.

\bibitem[Chen et~al.(2025)Chen, Li, Dai, Gu, and Xue]{chen2025physics}
Liqun Chen, Yuxuan Li, Jun Dai, Jinwei Gu, and Tianfan Xue.
\newblock A physics-informed blur learning framework for imaging systems.
\newblock In \emph{Conference on Computer Vision and Pattern Recognition
  (CVPR)}, 2025.

\bibitem[Chen et~al.(2020)Chen, Han, Dai, Jia, Liu, Liu, Zou, Xu, Liu, and
  Tian]{chen20unsup}
Shuaijun Chen, Zhen Han, Enyan Dai, Xu~Jia, Ziluan Liu, Xing Liu, Xueyi Zou,
  Chunjing Xu, Jianzhuang Liu, and Qi~Tian.
\newblock Unsupervised image super-resolution with an indirect supervised path.
\newblock In \emph{Conference on Computer Vision and Pattern Recognition (CVPR)
  Workshop}, 2020.

\bibitem[Chen et~al.(2018)Chen, Badrinarayanan, Lee, and
  Rabinovich]{chen2018gradnorm}
Zhao Chen, Vijay Badrinarayanan, Chen-Yu Lee, and Andrew Rabinovich.
\newblock {GradNorm}: Gradient normalization for adaptive loss balancing in
  deep multitask networks.
\newblock In \emph{International Conference on Machine Learning (ICML)}, 2018.

\bibitem[Chung et~al.(2023)Chung, Kim, Kim, and Ye]{chung2023blinddps}
Hyungjin Chung, Jeongsol Kim, Sehui Kim, and Jong~Chul Ye.
\newblock Parallel diffusion models of operator and image for blind inverse
  problems.
\newblock In \emph{Conference on Computer Vision and Pattern Recognition
  (CVPR)}, 2023.

\bibitem[Cover \& Thomas(2006)Cover and Thomas]{cover2006elements}
Thomas~M. Cover and Joy~A. Thomas.
\newblock \emph{Elements of Information Theory}.
\newblock Wiley-Interscience, Hoboken, NJ, 2 edition, 2006.
\newblock ISBN 978-0-471-24195-9.

\bibitem[Donoho \& Johnstone(1994)Donoho and Johnstone]{donoho1994ideal}
David~L. Donoho and Iain~M. Johnstone.
\newblock Ideal spatial adaptation by wavelet shrinkage.
\newblock \emph{Biometrika}, 81\penalty0 (3):\penalty0 425--455, 1994.

\bibitem[Dytso et~al.(2023)Dytso, Poor, and Shamai]{dytso2023conditional}
Alex Dytso, H.~Vincent Poor, and Shlomo Shamai.
\newblock Conditional mean estimation in gaussian noise: A meta derivative
  identity with applications.
\newblock \emph{IEEE Transactions on Information Theory}, 69\penalty0
  (3):\penalty0 1883--1898, 2023.

\bibitem[Eboli et~al.(2022)Eboli, Morel, and Facciolo]{eboli2022fast}
Thomas Eboli, Jean-Michel Morel, and Gabriele Facciolo.
\newblock Fast two-step blind optical aberration correction.
\newblock In \emph{European Conference on Computer Vision (ECCV)}, 2022.

\bibitem[Efrat et~al.(2013)Efrat, Glasner, Apartsin, Nadler, and
  Levin]{efrat2013accurate}
Netalee Efrat, Daniel Glasner, Alexander Apartsin, Boaz Nadler, and Anat Levin.
\newblock Accurate blur models vs. image priors in single image
  super-resolution.
\newblock In \emph{International Conference on Computer Vision (ICCV)}, 2013.

\bibitem[{European Machine Vision Association}(2021)]{emva1288}
{European Machine Vision Association}.
\newblock {EMVA Standard 1288}: Standard for characterization of image sensors
  and cameras, {Release 4.0 Linear}, 2021.

\bibitem[Foi et~al.(2008)Foi, Trimeche, Katkovnik, and
  Egiazarian]{foi2008practical}
Alessandro Foi, Mejdi Trimeche, Vladimir Katkovnik, and Karen Egiazarian.
\newblock Practical {P}oissonian-{G}aussian noise modeling and fitting for
  single-image raw-data.
\newblock \emph{IEEE Trans. Image Process.}, 17\penalty0 (10):\penalty0
  1737--1754, 2008.

\bibitem[Guan et~al.(2024)Guan, Iqbal, Davenport, and
  Masood]{guan2024modelmismatch}
Peimeng Guan, Naveed Iqbal, Mark~A. Davenport, and Mudassir Masood.
\newblock Solving inverse problems with model mismatch using untrained neural
  networks within model-based architectures.
\newblock \emph{Transactions on Machine Learning Research (TMLR)}, 2024.

\bibitem[Guo et~al.(2005)Guo, Shamai, and Verd{\'u}]{guo2005}
Dongning Guo, Shlomo Shamai, and Sergio Verd{\'u}.
\newblock Mutual information and minimum mean-square error in {G}aussian
  channels.
\newblock \emph{IEEE Transactions on Information Theory}, 51\penalty0
  (4):\penalty0 1261--1282, 2005.

\bibitem[Guo et~al.(2011)Guo, Wu, Shamai, and Verd{\'u}]{guo2011}
Dongning Guo, Yihong Wu, Shlomo Shamai, and Sergio Verd{\'u}.
\newblock Estimation in {G}aussian noise: Properties of the minimum mean-square
  error.
\newblock \emph{IEEE Transactions on Information Theory}, 57\penalty0
  (4):\penalty0 2371--2385, 2011.

\bibitem[Immerk{\ae}r(1996)]{immerkaer1996fast}
John Immerk{\ae}r.
\newblock Fast noise variance estimation.
\newblock \emph{Computer Vision and Image Understanding}, 64\penalty0
  (2):\penalty0 300--302, 1996.

\bibitem[Jiang et~al.(2026)Jiang, Gao, Gao, Yi, Qian, Shi, Yang, Sun, Wang, and
  Bai]{jiang2024flexible}
Qi~Jiang, Yao Gao, Shaohua Gao, Zhonghua Yi, Xiaolong Qian, Hao Shi, Kailun
  Yang, Lei Sun, Kaiwei Wang, and Jian Bai.
\newblock {OmniLens}: Towards universal lens aberration correction via
  {LensLib}-to-specific domain adaptation.
\newblock \emph{Optics \& Laser Technology}, 198:\penalty0 114858, 2026.

\bibitem[Kang(2007)]{kang2007automatic}
Sing~Bing Kang.
\newblock Automatic removal of chromatic aberration from a single image.
\newblock In \emph{Conference on Computer Vision and Pattern Recognition
  (CVPR)}, 2007.

\bibitem[Kee et~al.(2011)Kee, Paris, Chen, and Wang]{kee2011modeling}
Eric Kee, Sylvain Paris, Simon Chen, and Jue Wang.
\newblock Modeling and removing spatially-varying optical blur.
\newblock In \emph{IEEE International Conference on Computational Photography},
  2011.

\bibitem[Laroche et~al.(2024)Laroche, Almansa, and Coupete]{laroche2024fast}
Charles Laroche, Andr{\'e}s Almansa, and Eva Coupete.
\newblock Fast diffusion {EM}: A diffusion model for blind inverse problems
  with application to deconvolution.
\newblock In \emph{IEEE/CVF Winter Conference on Applications of Computer
  Vision (WACV)}, 2024.

\bibitem[Lee et~al.(2026)Lee, Lee, and Choi]{lee2026uotip}
Donggyu Lee, Taekyung Lee, and Jaewoong Choi.
\newblock {UOTIP}: Unbalanced optimal transport map for unpaired inverse
  problems.
\newblock In \emph{International Conference on Machine Learning (ICML)}, 2026.

\bibitem[Levac et~al.(2025)Levac, Jalal, Ramchandran, and
  Tamir]{levac2025double}
Brett Levac, Ajil Jalal, Kannan Ramchandran, and Jonathan~I. Tamir.
\newblock Double blind imaging with generative modeling.
\newblock \emph{arXiv preprint arXiv:2503.21501}, 2025.

\bibitem[Li et~al.(2025)Li, Gan, and Kamilov]{li2025blindpnp}
Anqi Li, Weijie Gan, and Ulugbek~S. Kamilov.
\newblock Plug-and-play posterior sampling for blind inverse problems.
\newblock In \emph{2025 IEEE Statistical Signal Processing Workshop (SSP)},
  2025.

\bibitem[Lim et~al.(2020)Lim, Park, Lee, Chang, Sim, and Ye]{lim2020cyclegan}
Sungjun Lim, Hyoungjun Park, Sang-Eun Lee, Sunghoe Chang, Byeongsu Sim, and
  Jong~Chul Ye.
\newblock {CycleGAN} with a blur kernel for deconvolution microscopy: Optimal
  transport geometry.
\newblock \emph{IEEE Transactions on Computational Imaging}, 6:\penalty0
  1127--1138, 2020.

\bibitem[Lipman et~al.(2023)Lipman, Chen, Ben-Hamu, Nickel, and
  Le]{lipman2023flow}
Yaron Lipman, Ricky T.~Q. Chen, Heli Ben-Hamu, Maximilian Nickel, and Matt Le.
\newblock Flow matching for generative modeling.
\newblock In \emph{International Conference on Learning Representations
  (ICLR)}, 2023.

\bibitem[Liu et~al.(2023)Liu, Gong, and Liu]{liu2023flow}
Xingchao Liu, Chengyue Gong, and Qiang Liu.
\newblock Flow straight and fast: Learning to generate and transfer data with
  rectified flow.
\newblock In \emph{International Conference on Learning Representations
  (ICLR)}, 2023.

\bibitem[Liu et~al.(2013)Liu, Tanaka, and Okutomi]{liu2013single}
Xinhao Liu, Masayuki Tanaka, and Masatoshi Okutomi.
\newblock Single-image noise level estimation for blind denoising.
\newblock \emph{IEEE Trans. Image Process.}, 22\penalty0 (12):\penalty0
  5226--5237, 2013.

\bibitem[Lunz et~al.(2021)Lunz, Hauptmann, Tarvainen, Sch{\"o}nlieb, and
  Arridge]{lunz2021learned}
Sebastian Lunz, Andreas Hauptmann, Tanja Tarvainen, Carola-Bibiane
  Sch{\"o}nlieb, and Simon Arridge.
\newblock On learned operator correction in inverse problems.
\newblock \emph{SIAM Journal on Imaging Sciences}, 14\penalty0 (1):\penalty0
  92--127, 2021.

\bibitem[Luo et~al.(2023)Luo, Hu, Zhang, Sun, Li, and
  Zhang]{luo2023diffinstruct}
Weijian Luo, Tianyang Hu, Shifeng Zhang, Jiacheng Sun, Zhenguo Li, and Zhihua
  Zhang.
\newblock Diff-instruct: A universal approach for transferring knowledge from
  pre-trained diffusion models.
\newblock In \emph{Adv. Neural Information Processing Systems (NeurIPS)}, 2023.

\bibitem[Luo et~al.(2022{\natexlab{a}})Luo, Huang, Li, Wang, and
  Tan]{luo2022learning}
Zhengxiong Luo, Yan Huang, Shang Li, Liang Wang, and Tieniu Tan.
\newblock Learning the degradation distribution for blind image
  super-resolution.
\newblock In \emph{Conference on Computer Vision and Pattern Recognition
  (CVPR)}, 2022{\natexlab{a}}.

\bibitem[Luo et~al.(2022{\natexlab{b}})Luo, Huang, Yu, Li, Fan, and
  Liu]{luo2022deep}
Ziwei Luo, Haibin Huang, Lei Yu, Youwei Li, Haoqiang Fan, and Shuaicheng Liu.
\newblock Deep constrained least squares for blind image super-resolution.
\newblock In \emph{Conference on Computer Vision and Pattern Recognition
  (CVPR)}, 2022{\natexlab{b}}.

\bibitem[Maeda(2020)]{maeda2020unpaired}
Shunta Maeda.
\newblock Unpaired image super-resolution using pseudo-supervision.
\newblock In \emph{Conference on Computer Vision and Pattern Recognition
  (CVPR)}, 2020.

\bibitem[Meanti et~al.(2025)Meanti, Ryckeboer, Arbel, and
  Mairal]{meanti2025unsupervised}
Giacomo Meanti, Thomas Ryckeboer, Michael Arbel, and Julien Mairal.
\newblock Unsupervised imaging inverse problems with diffusion distribution
  matching.
\newblock In \emph{International Conference on Computer Vision (ICCV)}, 2025.

\bibitem[Monga et~al.(2021)Monga, Li, and Eldar]{monga2021algorithm}
Vishal Monga, Yuelong Li, and Yonina~C. Eldar.
\newblock Algorithm unrolling: Interpretable, efficient deep learning for
  signal and image processing.
\newblock \emph{IEEE Signal Processing Magazine}, 38\penalty0 (2):\penalty0
  18--44, 2021.
\newblock \doi{10.1109/MSP.2020.3016905}.

\bibitem[Murata et~al.(2023)Murata, Saito, Lai, Takida, Uesaka, Mitsufuji, and
  Ermon]{murata2023gibbsddrm}
Naoki Murata, Koichi Saito, Chieh-Hsin Lai, Yuhta Takida, Toshimitsu Uesaka,
  Yuki Mitsufuji, and Stefano Ermon.
\newblock {GibbsDDRM}: A partially collapsed gibbs sampler for solving blind
  inverse problems with denoising diffusion restoration.
\newblock In \emph{International Conference on Machine Learning (ICML)}, 2023.

\bibitem[Peyr{\'e} \& Cuturi(2019)Peyr{\'e} and Cuturi]{peyre2019computational}
Gabriel Peyr{\'e} and Marco Cuturi.
\newblock Computational optimal transport.
\newblock \emph{Foundations and Trends in Machine Learning}, 11\penalty0
  (5--6), 2019.

\bibitem[Pl{\"o}tz \& Roth(2017)Pl{\"o}tz and Roth]{plotz2017benchmarking}
Tobias Pl{\"o}tz and Stefan Roth.
\newblock Benchmarking denoising algorithms with real photographs.
\newblock In \emph{Conference on Computer Vision and Pattern Recognition
  (CVPR)}, pp.\  2750--2759, 2017.

\bibitem[Poole et~al.(2023)Poole, Jain, Barron, and
  Mildenhall]{poole2022dreamfusion}
Ben Poole, Ajay Jain, Jonathan~T. Barron, and Ben Mildenhall.
\newblock {DreamFusion}: Text-to-3d using 2d diffusion.
\newblock In \emph{International Conference on Learning Representations
  (ICLR)}, 2023.

\bibitem[Prelov \& Verd{\'u}(2004)Prelov and Verd{\'u}]{prelov2004}
Vyacheslav~V. Prelov and Sergio Verd{\'u}.
\newblock Second-order asymptotics of mutual information.
\newblock \emph{IEEE Transactions on Information Theory}, 50\penalty0
  (8):\penalty0 1567--1580, 2004.
\newblock \doi{10.1109/TIT.2004.831784}.

\bibitem[Pyatykh et~al.(2013)Pyatykh, Hesser, and Zheng]{pyatykh2013image}
Stanislav Pyatykh, J{\"u}rgen Hesser, and Lei Zheng.
\newblock Image noise level estimation by principal component analysis.
\newblock \emph{IEEE Trans. Image Process.}, 22\penalty0 (2):\penalty0
  687--699, 2013.

\bibitem[Romano et~al.(2017)Romano, Elad, and Milanfar]{romano2017little}
Yaniv Romano, Michael Elad, and Peyman Milanfar.
\newblock The little engine that could: Regularization by denoising ({RED}).
\newblock \emph{SIAM Journal on Imaging Sciences}, 10\penalty0 (4):\penalty0
  1804--1844, 2017.

\bibitem[Sanghvi et~al.(2022)Sanghvi, Gnanasambandam, Mao, and
  Chan]{sanghvi2022photon}
Yash Sanghvi, Abhiram Gnanasambandam, Zhiyuan Mao, and Stanley~H. Chan.
\newblock Photon-limited blind deconvolution using unsupervised iterative
  kernel estimation.
\newblock \emph{IEEE Transactions on Computational Imaging}, 8:\penalty0
  1051--1062, 2022.

\bibitem[Sanghvi et~al.(2024)Sanghvi, Chi, and Chan]{sanghvi2024kernel}
Yash Sanghvi, Yiheng Chi, and Stanley~H. Chan.
\newblock Kernel diffusion: An alternate approach to blind deconvolution.
\newblock In \emph{European Conference on Computer Vision (ECCV)}, 2024.

\bibitem[Schuler et~al.(2012)Schuler, Hirsch, Harmeling, and
  Sch{\"o}lkopf]{schuler2012blind}
Christian~J. Schuler, Michael Hirsch, Stefan Harmeling, and Bernhard
  Sch{\"o}lkopf.
\newblock Blind correction of optical aberrations.
\newblock In \emph{European Conference on Computer Vision (ECCV)}, 2012.

\bibitem[Shaban et~al.(2019)Shaban, Cheng, Hatch, and
  Boots]{shaban2019truncated}
Amirreza Shaban, Ching-An Cheng, Nathan Hatch, and Byron Boots.
\newblock Truncated back-propagation for bilevel optimization.
\newblock In \emph{International Conference on Artificial Intelligence and
  Statistics (AISTATS)}, 2019.

\bibitem[Sim et~al.(2020)Sim, Oh, Kim, Jung, and Ye]{sim2020optimal}
Byeongsu Sim, Gyutaek Oh, Jeongsol Kim, Chanyong Jung, and Jong~Chul Ye.
\newblock Optimal transport driven {CycleGAN} for unsupervised learning in
  inverse problems.
\newblock \emph{SIAM Journal on Imaging Sciences}, 13\penalty0 (4):\penalty0
  2281--2306, 2020.

\bibitem[Speiser et~al.(2021)Speiser, M{\"u}ller, Hoess, Matti, Obara, Legant,
  Kreshuk, Macke, Ries, and Turaga]{speiser2021deep}
Artur Speiser, Lucas-Raphael M{\"u}ller, Philipp Hoess, Ulf Matti,
  Christopher~J Obara, Wesley~R Legant, Anna Kreshuk, Jakob~H Macke, Jonas
  Ries, and Srinivas~C Turaga.
\newblock Deep learning enables fast and dense single-molecule localization
  with high accuracy.
\newblock \emph{Nature methods}, 18\penalty0 (9):\penalty0 1082--1090, 2021.

\bibitem[Tachella et~al.(2025)Tachella, Terris, Hurault, Wang, Davy, Scanvic,
  Sechaud, Vo, Moreau, Davies, Chen, Laurent, Monroy, Dong, Hu, Nguyen, Sarron,
  Weiss, Escande, Massias, Modrzyk, Levac, Liaudat, Song, Hertrich, Neumayer,
  and Schramm]{tachella2025deepinverse}
Juli{\'a}n Tachella, Matthieu Terris, Samuel Hurault, Andrew Wang, Leo Davy,
  J{\'e}r{\'e}my Scanvic, Victor Sechaud, Romain Vo, Thomas Moreau, Thomas
  Davies, Dongdong Chen, Nils Laurent, Brayan Monroy, Jonathan Dong, Zhiyuan
  Hu, Minh-Hai Nguyen, Florian Sarron, Pierre Weiss, Paul Escande, Mathurin
  Massias, Thibaut Modrzyk, Brett Levac, Tob{\'i}as~I. Liaudat, Maxime Song,
  Johannes Hertrich, Sebastian Neumayer, and Georg Schramm.
\newblock {DeepInverse}: A {Python} package for solving imaging inverse
  problems with deep learning.
\newblock \emph{Journal of Open Source Software}, 10\penalty0 (115):\penalty0
  8923, 2025.

\bibitem[Tang et~al.(2024)Tang, Hu, Gu, and Sun]{tang2024residual}
Xiaole Tang, Xin Hu, Xiang Gu, and Jian Sun.
\newblock Residual-conditioned optimal transport: towards structure-preserving
  unpaired and paired image restoration.
\newblock In \emph{International Conference on Machine Learning (ICML)}, 2024.

\bibitem[Venkatakrishnan et~al.(2013)Venkatakrishnan, Bouman, and
  Wohlberg]{venkatakrishnan2013plug}
Singanallur~V. Venkatakrishnan, Charles~A. Bouman, and Brendt Wohlberg.
\newblock Plug-and-play priors for model based reconstruction.
\newblock In \emph{Global Conference on Signal and Information Processing
  (GlobalSIP)}, 2013.

\bibitem[Wang et~al.(2021)Wang, Xie, Dong, and Shan]{wang2021realesrgan}
Xintao Wang, Liangbin Xie, Chao Dong, and Ying Shan.
\newblock Real-{ESRGAN}: Training real-world blind super-resolution with pure
  synthetic data.
\newblock In \emph{International Conference on Computer Vision (ICCV)
  Workshop}, 2021.

\bibitem[Wang et~al.(2023)Wang, Lu, Wang, Bao, Li, Su, and
  Zhu]{wang2023prolificdreamer}
Zhengyi Wang, Cheng Lu, Yikai Wang, Fan Bao, Chongxuan Li, Hang Su, and Jun
  Zhu.
\newblock {ProlificDreamer}: High-fidelity and diverse text-to-{3D} generation
  with variational score distillation.
\newblock In \emph{Adv. Neural Information Processing Systems (NeurIPS)}, 2023.

\bibitem[Wolf et~al.(2021)Wolf, Lugmayr, Danelljan, Gool, and
  Timofte]{wolf2021deflow}
Valentin Wolf, Andreas Lugmayr, Martin Danelljan, Luc~Van Gool, and Radu
  Timofte.
\newblock {DeFlow}: Learning complex image degradations from unpaired data with
  conditional flows.
\newblock In \emph{Conference on Computer Vision and Pattern Recognition
  (CVPR)}, 2021.

\bibitem[Yang et~al.(2024)Yang, Xia, Li, Huang, Zhang, Liu, Fu, and
  Liu]{yang2024dynamic}
Zhixiong Yang, Jingyuan Xia, Shengxi Li, Xinghua Huang, Shuanghui Zhang, Zhen
  Liu, Yaowen Fu, and Yongxiang Liu.
\newblock A dynamic kernel prior model for unsupervised blind image
  super-resolution.
\newblock In \emph{Conference on Computer Vision and Pattern Recognition
  (CVPR)}, 2024.

\bibitem[Yuan et~al.(2018)Yuan, Liu, Zhang, Zhang, Dong, and
  Lin]{yuan2018unsupervised}
Yuan Yuan, Siyuan Liu, Jiawei Zhang, Yongbing Zhang, Chao Dong, and Liang Lin.
\newblock Unsupervised image super-resolution using cycle-in-cycle generative
  adversarial networks.
\newblock In \emph{Conference on Computer Vision and Pattern Recognition (CVPR)
  Workshop}, 2018.

\bibitem[Zamir et~al.(2020)Zamir, Arora, Khan, Hayat, Khan, Yang, and
  Shao]{zamir2020cycleisp}
Syed~Waqas Zamir, Aditya Arora, Salman Khan, Munawar Hayat, Fahad~Shahbaz Khan,
  Ming-Hsuan Yang, and Ling Shao.
\newblock {CycleISP}: Real image restoration via improved data synthesis.
\newblock In \emph{Conference on Computer Vision and Pattern Recognition
  (CVPR)}, 2020.

\bibitem[Zamir et~al.(2022)Zamir, Arora, Khan, Hayat, Khan, and
  Yang]{zamir2022restormer}
Syed~Waqas Zamir, Aditya Arora, Salman Khan, Munawar Hayat, Fahad~Shahbaz Khan,
  and Ming-Hsuan Yang.
\newblock Restormer: Efficient transformer for high-resolution image
  restoration.
\newblock In \emph{Conference on Computer Vision and Pattern Recognition
  (CVPR)}, 2022.

\bibitem[Zhang et~al.(2021)Zhang, Li, Zuo, Zhang, Gool, and
  Timofte]{zhang2021plug}
Kai Zhang, Yawei Li, Wangmeng Zuo, Lei Zhang, Luc~Van Gool, and Radu Timofte.
\newblock Plug-and-play image restoration with deep denoiser prior.
\newblock \emph{IEEE Transactions on Pattern Analysis and Machine Intelligence
  (PAMI)}, 2021.

\bibitem[Zhu et~al.(2017)Zhu, Park, Isola, and Efros]{zhu2017cyclegan}
Jun-Yan Zhu, Taesung Park, Phillip Isola, and Alexei~A. Efros.
\newblock Unpaired image-to-image translation using cycle-consistent
  adversarial networks.
\newblock In \emph{International Conference on Computer Vision (ICCV)}, 2017.

\bibitem[Zhu et~al.(2023)Zhu, Zhang, Liang, Cao, Wen, Timofte, and
  Gool]{zhu2023diffpir}
Yuanzhi Zhu, Kai Zhang, Jingyun Liang, Jiezhang Cao, Bihan Wen, Radu Timofte,
  and Luc~Van Gool.
\newblock Denoising diffusion models for plug-and-play image restoration.
\newblock In \emph{IEEE/CVF Conference on Computer Vision and Pattern
  Recognition Workshops (NTIRE)}, 2023.

\end{thebibliography}
\bibliographystyle{iclr2027_conference}

\newpage
\appendix

\section{Operator constraints and kernel centering}
\label{sec:appendix:centering}

Except for the three motion kernels (Appendix~\ref{sec:appendix:kernels}),
all kernel-based operators use the centering procedure described in this
appendix.
Kernels are constrained to be non-negative and
unit-sum with a softplus parametrization, while their absolute position is fixed by a centering convention.

\subsection{Parameterization and translation ambiguity}
\label{sec:appendix:centering:set}
\label{sec:appendix:centering:gauge}

For a single kernel, we optimize an unconstrained tensor
$\theta\in\mathbb R^{h\times w}$ and define
\begin{equation}
    a_\theta=\operatorname{softplus}(\theta),
    \qquad
    \bar k_\theta
    =
    \frac{a_\theta}{\sum_u (a_\theta)_u}.
    \label{eq:app:center:param}
\end{equation}
Thus non-negativity and unit sum are enforced by construction. The same parameterization is applied to each kernel for
multi-channel and space-varying operators. The final forward kernel includes
the sub-pixel correction described below and is renormalized afterward.

If the patch distribution and measurement noise are translation-invariant,
shifting the kernel leaves the measurement distribution unchanged, which 
raises a classical translation ambiguity for kernel estimation in blind deblurring tasks.
Following \citet{meanti2025unsupervised}, we therefore fix the kernel position.
Rather than adding their center-of-mass penalty, we explicitly recenter the
kernel after each operator update, as explained below.

\subsection{Centering procedure}
\label{sec:appendix:centering:proc}

\paragraph{Robust center estimate.}
A global center of mass can be sensitive to small amounts of mass far from the
main kernel. We instead use a local mean-shift centroid. Starting from
$\bar k_\theta$, we subtract its median background level, retain the
non-negative residual $K$, initialize $c$ at its global centroid, and iterate
\begin{equation}
    c
    \leftarrow
    \frac{\sum_{u:\|u-c\|_\infty\leq w} K_u\,u}
         {\sum_{u:\|u-c\|_\infty\leq w} K_u}.
    \label{eq:app:center:meanshift}
\end{equation}
We use $w=10$ pixels and five iterations. Restricting the centroid to a local
window reduces the influence of distant residual mass while remaining suitable
for flat or elongated kernels.

\paragraph{Integer recentering.}
Let $\delta$ move $c$ to the geometric center of the support and let
$\delta_{\rm int}$ be its rounding to the nearest integer.
After each operator update,
we circularly shift $\theta$ by $\delta_{\rm int}$. This is a permutation of
its entries and therefore introduces no interpolation. For Adam, the first- and
second-moment buffers are shifted by the same amount so that the optimizer state
remains aligned with the parameters. 
The exponential moving average of the operator parameters is shifted by the
same amount before being updated, so that it never mixes differently
translated kernels.

\paragraph{Sub-pixel correction.}
The residual
$\delta_{\rm sub}=\delta-\delta_{\rm int}$ is at most half a pixel in each
coordinate. At each forward pass, we bilinearly shift the positive pre-kernel
and renormalize:
\begin{equation}
    k_\theta
    =
    \frac{T_{\delta_{\rm sub}}a_\theta}
         {\sum_u (T_{\delta_{\rm sub}}a_\theta)_u},
    \label{eq:app:center:forward-kernel}
\end{equation}
where $T_{\delta_{\rm sub}}$ denotes bilinear translation with zero padding.
The center estimate is detached from the computation graph, so gradients flow
through the shifted kernel values but not through the centering procedure.
Because this residual shift is recomputed at each forward pass and never written
back to $\theta$, 
interpolation errors do not accumulate, although a nonzero residual shift
slightly smooths the kernel.

\subsection{Variants and evaluation}
\label{sec:appendix:centering:variants}

\paragraph{Chromatic operators.}
Centering each color channel independently would remove the relative
displacements that encode chromatic aberration. We therefore estimate a single
shift from the green channel and apply it to all channels, removing only their
common translation.

\paragraph{Space-varying operators.}
For a grid of spatially varying PSFs, a common shift of all kernels corresponds
to the global ambiguity above. In practice, we apply the same centering
convention to each local kernel, using green as the anchor for color PSFs. This
keeps each PSF inside its finite support and fixes a consistent convention
across the grid. 
It also removes spatially varying shifts (geometric distortion), which are
only weakly identifiable from unpaired data; the ground-truth PSFs are
centered in the same way for evaluation.

\paragraph{Evaluation.}
\label{sec:appendix:centering:eval}
We report kernel error after translation alignment: the estimated kernel is
shifted at sub-pixel resolution so that its center of mass matches that of
the ground truth before computing the mean absolute error.
 This factors out the common translation ambiguity and measures
kernel shape more directly. For chromatic operators, alignment is anchored on a
single channel so that errors in relative channel displacement remain
penalized.

\section{IKL gradient estimator and implementation}
\label{sec:appendix:ikl}
For a distribution $P$, draw
$Z_1\sim P$, $Z_0\sim\mathcal N(0,I)$ independently, sample
$t\sim\operatorname{Unif}(0,1)$, and define 
\begin{equation}
    Z_t=(1-t)Z_0+tZ_1.
    \label{eq:appendix:interpolation}
\end{equation}
For fixed $t$, let $P^{(t)}$ denote the law of $Z_t$. The integrated KL
loss~\citep{luo2023diffinstruct} is then
\begin{equation}
    \mathcal D(P\Vert Q) = \IKL(P\Vert Q)
    = \int_0^1 w(t)\,\operatorname{KL}\big(P^{(t)}\Vert Q^{(t)}\big)\,\dd t,
    \label{eq:method:IKL}
\end{equation}
where $w$ is a weighting function specified below.

Conditional flow matching represents $P$ through a velocity field trained with
\begin{equation}
    \mathcal L_{\rm CFM}(v;P)
    =
    \E\!\left[
        \big\|v(Z_t,t)-(Z_1-Z_0)\big\|_2^2
    \right].
    \label{eq:appendix:cfm}
\end{equation}
Its minimizer is the conditional expectation
\begin{equation}
    v_P(z,t) = \E[Z_1-Z_0\mid Z_t=z],
    \label{eq:appendix:velocity}
\end{equation}
which appears in Eq.~(\ref{eq:method:ikl-gradient}).

We choose $w(t)=(1-t)/t^2$. Writing $s_P$ for the score of $P^{(t)}$, the
gradient of each KL term, for a target $Q$ independent of $\theta$, is
\begin{equation}
    \nabla_\theta \operatorname{KL}\big(P_\theta^{(t)}\Vert Q^{(t)}\big)
    = \E\big[(s_{P_\theta}-s_Q)(Z_t)^\top\,\partial_\theta Z_t\big],
\end{equation}
see \citet{luo2023diffinstruct}. Since $s_P-s_Q=\frac{t}{1-t}(v_P-v_Q)$ and
$\partial_\theta Z_t=t\,J_\theta Z_1$, this gradient equals
$\frac{t^2}{1-t}\,\E\big[(J_\theta Z_1)^\top(v_{P_\theta}-v_Q)(Z_t,t)\big]$;
the weight $w$ cancels the factor $t^2/(1-t)$ and yields
Eq.~(\ref{eq:method:ikl-gradient}).%

\paragraph{Teacher and auxiliary fields.}
For a fixed target distribution $Q$, we train a teacher field
$v^{\rm tc}$ with Eq.~(\ref{eq:appendix:cfm}) on samples from $Q$ and then
freeze it. The current model distribution $P_\theta$ changes as $\theta$ is
updated, so a second field $v^{\rm aux}$ is trained online with the same CFM
objective on samples $Z_1(\theta)\sim P_\theta$. Thus
$v^{\rm tc}\approx v_Q$ while $v^{\rm aux}\approx v_{P_\theta}$, following
the alternating distribution-matching construction of
\citet{luo2023diffinstruct}, and used by~\citet{meanti2025unsupervised}.

For one draw of $Z_1(\theta)$, $t$ and $Z_0$ (averaged over a mini-batch in
practice), Eq.~(\ref{eq:method:ikl-gradient}) is estimated
by
\begin{equation}
    \widehat g_\theta
    =
    (J_\theta Z_1)^\top
    \left[
        v^{\rm aux}(Z_t,t)-v^{\rm tc}(Z_t,t)
    \right].
    \label{eq:appendix:ikl-mc}
\end{equation}
During the operator update, both velocity predictions are treated as constants
with respect to $\theta$. In practice, the same gradient is obtained by
automatic differentiation of the stop-gradient surrogate
\begin{equation}
    \widehat{\mathcal J}(\theta)
    =
    \left\langle
        \operatorname{sg}\!\left[
            v^{\rm aux}(Z_t,t)-v^{\rm tc}(Z_t,t)
        \right],
        Z_1(\theta)
    \right\rangle,
    \label{eq:appendix:stopgrad}
\end{equation}
where $\operatorname{sg}$ is the stop-gradient operator. 
By the chain rule, $\nabla_\theta\widehat{\mathcal J}(\theta)=\widehat g_\theta$.

Although $Z_t$ depends on $Z_1(\theta)$, this dependence is used only to query
the two velocity fields; their outputs are detached when updating the operator.
This is exact, since the velocity fields enter
Eq.~(\ref{eq:method:ikl-gradient}) only through their values at $Z_t$.
Conversely, when updating $v^{\rm aux}$, the sample $Z_1(\theta)$ is detached
and $\theta$ is held fixed. Hence the alternating updates optimize either the
auxiliary field or the operator, never both through the same backward pass.

\paragraph{The two branches.}
The estimator is identical in the measurement and reconstruction branches. For
measurement matching,
\begin{equation}
    Z_1(\theta)=A_\theta X+\varepsilon,
\end{equation}
the teacher is trained on $P_Y$, and the auxiliary field tracks
$P_{Y,\theta}$. For reconstruction matching,
\begin{equation}
    Z_1(\theta)=R_\theta(Y),
\end{equation}
the teacher is trained on $P_X$, and the auxiliary field tracks $Q_\theta$.
The only difference is therefore the differentiable map through which
$Z_1(\theta)$ depends on the operator parameters; in the reconstruction branch,
this backward pass is truncated as described in
Section~\ref{sec:method:restoration}.

\section{Implementation and task-specific details}
\label{sec:method:details}

This section provides implementation details omitted from
Algorithm~\ref{alg:method}, in addition to 
the IKL estimator and auxiliary-flow training 
described in Appendix~\ref{sec:appendix:ikl}, and the kernel parameterization
and centering procedure in Appendix~\ref{sec:appendix:centering}.
\rev{Our implementation relies on the DeepInverse
library~\citep{tachella2025deepinverse} for the forward operators and noise
models, for the pretrained denoisers (DRUNet and the diffusion U-Net trained on
FFHQ), and for DiffPIR.}

\subsection{Optimization details}
\label{sec:appendix:impl}

\paragraph{Reconstruction warm-up.}
Early in training, $A_\theta$ is still close to its initialization and
$R_\theta(Y)$ is consequently unreliable. We therefore introduce the
reconstruction branch gradually by replacing $\lambda$ in
Eq.~(\ref{eq:method:balance}) with
\begin{equation}
    \lambda(k)
    =
    \min\!\left(1,\frac{k}{K_{\rm warm}}\right)\lambda .
    \label{eq:method:curriculum}
\end{equation}
After $K_{\rm warm}$ operator updates, the target weight $\lambda$ is kept
fixed. The clean-side auxiliary field continues to track the current
reconstruction distribution throughout the warm-up.

\paragraph{Operator averaging.}
We maintain an exponential moving average $\bar\theta$ of the operator
parameters and use $A_{\bar\theta}$ for evaluation. This reduces the
step-to-step variation induced by the stochastic IKL gradients. 
The average is shifted together with the operator at each recentering, so
that it does not mix differently translated kernel representatives; see
Appendix~\ref{sec:appendix:centering}.

\paragraph{Auxiliary-field initialization and reuse.}
The auxiliary fields are initialized from their corresponding frozen teachers,
\[
    v_X^{\rm aux}\leftarrow v_X^{\rm tc},
    \qquad
    v_Y^{\rm aux}\leftarrow v_Y^{\rm tc},
\]
rather than from random weights. This provides a useful initialization for the
online distributions they subsequently track.

Updating $v_Y^{\rm aux}$ only requires synthetic measurements
$A_\theta X+\varepsilon$. For $v_X^{\rm aux}$, we reuse the reconstruction
already computed for the clean-side gradient: after estimating $g_G$,
$R_\theta(Y)$ is detached and used as an auxiliary-flow training sample.
This avoids a second call to the reconstruction solver and results in one
reconstruction per operator-update cycle. All teacher and auxiliary fields use
the conditional flow-matching objective of
Eq.~(\ref{eq:appendix:cfm}), with $t\sim\operatorname{Unif}(0,1)$.

\subsection{Reconstruction and data handling}
\label{sec:appendix:impl:reconstruction}
\label{sec:appendix:impl:data}

\paragraph{DPIR and truncated backpropagation.}
We implement $R_\theta$ with DPIR~\citep{zhang2021plug}, keeping its pretrained
denoiser and iteration schedule fixed during operator learning. We run $K$
iterations and differentiate only through the final $K_{\rm bp}$ ($K=8$ and $K_{\rm bp}=1$ in all experiments; see Appendix~\ref{sec:appendix:exp:common}). The state
after iteration $K-K_{\rm bp}$ is detached and used to initialize the
differentiated part of the solver. The data-fidelity update uses a closed-form
proximal operator when available and conjugate gradients otherwise.

\paragraph{Noise levels.}
The measurement noise is not estimated jointly with the operator. The noise
standard deviation used to synthesize measurements in the $F$ branch and the
noise level supplied to DPIR are fixed to the known value for synthetic data.
\rev{On the real capture, the former is calibrated from the raw files of the
camera (Appendix~\ref{sec:appendix:exp:noise-est}) and the latter is given in
Appendix~\ref{sec:appendix:exp:common}.} Thus $\theta$ is the only
quantity inferred from the unpaired datasets.

\paragraph{Patch sampling and boundary handling.}
Training uses patches sampled at random spatial locations. The measurement
branch degrades clean patches, while the reconstruction branch restores
degraded patches at the same training resolution. We use replicate padding
unless the task uses a naturally periodic construction, in which case circular
padding is used; blind super-resolution instead applies the operator as a
normalized convolution, since its observed set consists of real crops for which
neither convention is appropriate. The same boundary convention is used by the
forward model and the reconstruction solver. Per-experiment choices are reported in
Appendix~\ref{sec:appendix:exp:common}.

\paragraph{Spatial conditioning.}
Except in blind super-resolution, all teacher and auxiliary velocity fields
are conditioned on the patch location.
They therefore take the form
\begin{equation}
    v(z_t,t,\ell),
\end{equation}
where $\ell$ encodes the normalized spatial coordinates of the patch. This is
important for spatially varying degradation, for which both image and
measurement statistics can vary with position. We use the same conditioning in
the spatially uniform experiments for consistency.

\subsection{Task-specific forward models}
\label{sec:appendix:impl:tasks}

\paragraph{Spatially varying blur.}
Following \citet{meanti2025unsupervised}, we represent the spatially varying PSF
by an $n_h\times n_w$ grid of local kernels. The kernel associated with a patch
is obtained by bilinear interpolation of this grid at the patch center.
Non-negativity, normalization, chromatic anchoring, and centering are described
in Appendix~\ref{sec:appendix:centering}. Our method and the measurement-only
baseline use exactly the same operator parameterization and constraints.

\paragraph{Blind super-resolution.}
Instead of generating synthetic measurements from a clean dataset as in
Eq.~(\ref{eq:method:synthetic}), $A_\theta$ is applied to crops of the observed
low-resolution image, and the resulting twice-degraded distribution is matched
to that of smaller crops of the same image, following the cross-scale
protocol of~\citet{bell2019kernelgan}. Likewise, the reconstruction
branch restores these smaller crops and matches the restorations to crops of
twice their size, which play the role of clean images.

\subsection{Initialization}
\label{sec:appendix:impl:repro}

The unconstrained operator parameters are initialized randomly. Combined with
the softplus-and-normalize parameterization of
Appendix~\ref{sec:appendix:centering:set}, this yields a diffuse initial kernel
without imposing a preferred PSF shape.

\section{Technical details for Section~\ref{sec:theory}}
\label{app:theory}

The results derived here are inspired by the analysis
of minimum mean square error estimators for Gaussian denoising
in the information theory literature~\citep[see][]{guo2011}.

Theorem~\ref{thm:theory:weak} concerns the scalar channel
\begin{equation}
    Y_\rho=\sqrt{\rho}\,X+N,
    \qquad
    N\sim\mathcal N(0,1),
    \qquad
    \E[X]=0,\quad
    \E[X^2]=1.
    \label{eq:app:channel}
\end{equation}
The scalar variable
$X$ represents one component of the clean signal, while $\rho$ measures how
strongly that component remains observable after application of the forward
operator and addition of noise.

We compare the true channel parameter $\rho$ with the candidate
\begin{equation}
    \rho_\delta=\rho(1+\delta).
\end{equation}

For later use, define the minimum-mean square estimator
\begin{equation}
    \eta_\gamma(y)
    =
    \E[X\mid \sqrt{\gamma}X+N=y],
    \label{eq:app:eta}
\end{equation}
the reconstruction distribution
\begin{equation}
    Q_{\rho,\delta}
    =
    \operatorname{Law}
    \!\left(\eta_{\rho(1+\delta)}(Y_\rho)\right),
\end{equation}
and its reconstruction risk
\begin{equation}
    \mathcal R_\rho(\delta)
    =
    \E\!\left[
      \bigl(X-\eta_{\rho(1+\delta)}(Y_\rho)\bigr)^2
    \right].
\end{equation}

\subsection{Measurement-side sensitivity and Fisher information}
\label{app:theory:measurement}

Let $p_\rho$ denote the density of
$P_\rho=\operatorname{Law}(\sqrt{\rho}X+N)$. A local expansion of the KL
divergence in the smooth family $\{P_\rho\}$ gives
\begin{equation}
    \operatorname{KL}(P_\rho\Vert P_{\rho+\Delta})
    =
    \frac{1}{2}J(\rho)\Delta^2+o(\Delta^2),
    \label{eq:app:local-kl}
\end{equation}
where
\begin{equation}
    J(\rho)
    =
    \E_{Y_\rho}\!\left[
      \left(\partial_\rho\log p_\rho(Y_\rho)\right)^2
    \right]
\end{equation}
is the Fisher information for the channel parameter $\rho$~\citep[see][Chapter 11]{cover2006elements}. Since the relative
perturbation corresponds to $\Delta=\rho\delta$,
\begin{equation}
    F_\rho''(0)=\rho^2J(\rho).
    \label{eq:app:F-J}
\end{equation}

Low-SNR expansions for Gaussian channels are classical
\citep{prelov2004}. Here the required leading term follows directly by
expanding the Gaussian kernel:
\begin{equation}
    p_\rho(y)
    =
    \E[\varphi(y-\sqrt{\rho}X)]
    =
    \varphi(y)
    \left[
      1+\frac{\rho}{2}(y^2-1)+O(\rho^{3/2})
    \right],
    \label{eq:app:density-expansion}
\end{equation}
where $\varphi$ is the standard-normal density. Therefore
\begin{equation}
    \left.
    \partial_\rho\log p_\rho(y)
    \right|_{\rho=0}
    =
    \frac{1}{2}(y^2-1),
\end{equation}
and hence
\begin{equation}
    J(0)
    =
    \frac{1}{4}\E[(N^2-1)^2]
    =
    \frac{1}{2}.
\end{equation}
The same expansion gives $J(\rho)\to1/2$, so
\begin{equation}
    F_\rho''(0)
    =
    \frac{\rho^2}{2}+o(\rho^2).
    \label{eq:app:F-curvature}
\end{equation}

\subsection{Bounds on low-SNR MMSE estimators}

\begin{lemma}
\label{lem:app:posterior-expansion}
Let \(X\) satisfy
\[
    \mathbb E[X]=0,
    \qquad
    \mathbb E[X^2]=1,
    \qquad
    |X|\le M \quad \text{a.s.},
\]
and let \(N\sim\mathcal N(0,1)\) be independent of \(X\). For
\[
    Y_\beta=\beta X+N,
    \qquad
    m_\beta(y)=\mathbb E[X\mid Y_\beta=y],
\]
there exists a constant \(C=C(M)\) such that, for every
\(y\in\mathbb R\) and every \(|\beta|\le1\),
\begin{align}
    |m_\beta(y)-\beta y|
        &\le C\beta^2(1+y^2), \label{eq:app:m-expansion}\\
    |\partial_\beta m_\beta(y)-y|
        &\le C|\beta|(1+y^2), \label{eq:app:dm-expansion}\\
    |\partial_\beta^2 m_\beta(y)|
        &\le C(1+y^2). \label{eq:app:d2m-bound}
\end{align}
\end{lemma}

\begin{proof}
Let us introduce the standard Gaussian density
\[
    \varphi(t)=\frac{1}{\sqrt{2\pi}}e^{-t^2/2}.
\]
Conditioned on \(X=x\),
\(Y_\beta\) has the density $p_{Y_\beta\mid X=x}(y)=\varphi(y-\beta x)$ and 
   for every bounded measurable~\(f\), we introduce the quantity
\begin{equation}
   E_{\beta,y}(f) := \mathbb E[f(X)\mid Y_\beta=y]
=
\frac{
    \int f(x)\varphi(y-\beta x)\,P_X(dx)
}{
    \int \varphi(y-\beta x)\,P_X(dx)
}.
\label{eq:app:bayes-posterior}
\end{equation}
   Thus, if
\[
    Z_\beta(y)
    :=
    \int
    \exp\!\left(
        \beta xy-\frac{\beta^2x^2}{2}
    \right)P_X(dx),
\]
then
\begin{equation}
   E_{\beta,y}(f)
=
\frac{1}{Z_\beta(y)}
\int f(x)
\exp\!\left(
    \beta xy-\frac{\beta^2x^2}{2}
\right)P_X(dx).
\label{eq:app:posterior-expectation}
\end{equation}

   We next differentiate (\ref{eq:app:posterior-expectation}), noting that if
\[
    h_{\beta,y}(x):=xy-\beta x^2,
\]
then
\[
    \partial_\beta
    \exp\!\left(
        \beta xy-\frac{\beta^2x^2}{2}
    \right)
    =
    h_{\beta,y}(x)
    \exp\!\left(
        \beta xy-\frac{\beta^2x^2}{2}
    \right).
\]
Then, a short calculation gives
\begin{align}
   \partial_\beta E_{\beta,y}(f) 
    &=
    E_{\beta,y}( f h_{\beta,y})
    -
    E_{\beta,y}(f)
    E_{\beta,y}(h_{\beta,y})
    \notag\\
    &=
    \operatorname{Cov}
    \!\left(
        f(X),\,h_{\beta,y}(X)
        \mid Y_\beta=y
    \right).
\label{eq:app:posterior-beta-derivative}
\end{align}
The differentiation under the integral is justified by
\(|X|\le M\); for fixed \(y\) and \(|\beta|\le1\), both the
integrand and its \(\beta\)-derivatives are bounded by
integrable functions depending only on \(M\) and \(y\).

We now apply (\ref{eq:app:posterior-beta-derivative}) with
\(f(X)=X\). Define
\[
    V_\beta(y)
    :=
    \operatorname{Var}(X\mid Y_\beta=y),
    \qquad
    K_\beta(y)
    :=
    \operatorname{Cov}(X,X^2\mid Y_\beta=y).
\]
Since \(h_{\beta,y}(X)=yX-\beta X^2\), we obtain the exact
identity
\begin{equation}
    \partial_\beta m_\beta(y)
    =
    yV_\beta(y)-\beta K_\beta(y).
\label{eq:app:m-beta-derivative}
\end{equation}
Then,
differentiating once more gives
\[
    \partial_\beta^2m_\beta(y)
    =
    y\,\partial_\beta V_\beta(y)
    -K_\beta(y)
    -\beta\,\partial_\beta K_\beta(y).
\]
We will now proceed by bounding each term, starting with an auxiliary  result.
For $j=1,2,3$, let
\[
    a_{j,\beta}(y)
    :=
    \mathbb E[X^j\mid Y_\beta=y].
\]
Since $|X|\le M$, we have $|a_{j,\beta}(y)|\le M^j$.
Moreover, Eq.~(\ref{eq:app:posterior-beta-derivative}) gives
\[
    \partial_\beta a_{j,\beta}(y)
    =
    \operatorname{Cov}
    \left(
        X^j,\,yX-\beta X^2
        \mid Y_\beta=y
    \right).
\]
For $|\beta|\le1$,
\[
    |yX-\beta X^2|
    \le M|y|+M^2,
\]
and therefore
\begin{equation}
    |\partial_\beta a_{j,\beta}(y)|
    \le C(1+|y|),
    \qquad j=1,2,3,
    \label{eq:app:moment-derivative-bound}
\end{equation}
where $C$ depends only on $M$ (note that we use the relation $\operatorname{Cov}( Z1, Z2 ) \leq 2 \|Z_1\|_\infty \|Z_2\|_\infty$ and the fact that necessarily $M \geq 1$ given the normalization assumption on $X$).

We now remark that
\[
    V_\beta
    =a_{2,\beta}-a_{1,\beta}^2,
    \qquad
    K_\beta
    =a_{3,\beta}-a_{1,\beta}a_{2,\beta},
\]
and we can now obtain the following bounds (by adjusting the value of $C$)
\begin{equation}
    |K_\beta(y)|\le C,
    \qquad
    |\partial_\beta V_\beta(y)|
    +
    |\partial_\beta K_\beta(y)|
    \le C(1+|y|).
    \label{eq:app:VK-bounds}
\end{equation}

Using Eq.~(\ref{eq:app:VK-bounds}), and adjusting again the value of $C$, for $|\beta|\le1$,
\[
    |\partial_\beta^2m_\beta(y)|
    \le C(1+y^2),
\]
which proves (\ref{eq:app:d2m-bound}). Note that we use here the fact that $|y| \leq (1 + y^2)/2$.

It remains to integrate this bound. At $\beta=0$,
$Y_0=N$ carries no information about $X$, so
\[
    m_0(y)=\mathbb E[X]=0,
    \qquad
    V_0(y)=\operatorname{Var}(X)=1.
\]
Evaluating (\ref{eq:app:m-beta-derivative}) at $\beta=0$
therefore gives
\[
    \left.\partial_\beta m_\beta(y)\right|_{\beta=0}
    =y.
\]
By the fundamental theorem of calculus,
\[
    \partial_\beta m_\beta(y)-y
    =
    \int_0^\beta \partial_t^2m_t(y)\,dt,
\]
and hence
\[
    |\partial_\beta m_\beta(y)-y|
    \le C|\beta|(1+y^2),
\]
which proves (\ref{eq:app:dm-expansion}). Integrating once
more,
\[
    m_\beta(y)-\beta y
    =
    \int_0^\beta
    \bigl(\partial_t m_t(y)-y\bigr)\,dt,
\]
so that
\[
    |m_\beta(y)-\beta y|
    \le C\beta^2(1+y^2),
\]
after adjusting once more the constant $C$. This proves
(\ref{eq:app:m-expansion}).

\end{proof}

\subsection{Effect of operator error on reconstruction}
\label{app:theory:risk}

The posterior mean in Eq.~(\ref{eq:app:eta}) is the minimum-MSE estimator for
the Gaussian channel; we use standard properties of Gaussian-channel
conditional-mean estimation throughout this section
\citep{guo2005,guo2011}. In particular,
$\eta_\rho(Y_\rho)=\E[X\mid Y_\rho]$ is the $L^2$ projection of $X$ onto
functions of~$Y_\rho$. The resulting orthogonality identity gives
\begin{equation}
    \mathcal R_\rho(\delta)-\mathcal R_\rho(0)
    =
    \E\left[
        \bigl(
            \eta_{\rho(1+\delta)}(Y_\rho)
            -
            \eta_\rho(Y_\rho)
        \bigr)^2
    \right].
    \label{eq:app:mmse-pythagoras}
\end{equation}
Thus the excess MSE is exactly the squared change in the reconstruction caused
by using the wrong operator parameter.

To study the low-SNR regime, write $b=\sqrt{\rho}$ and, using the same notation as in the previous sub-section,
\[
    m_b(y)=\E[X\mid bX+N=y],
\]
we obtain from Lemma~\ref{lem:app:posterior-expansion} that
\begin{equation}
    |m_b(y)-by|
    \leq Cb^2(1+y^2),
    \label{eq:app:m-low-b}
\end{equation}
The amplitude after perturbation is then
$b_\delta=b\sqrt{1+\delta}$, such that
\begin{equation}
    \left.
    \partial_\delta m_{b_\delta}(y)
    \right|_{\delta=0}
    =
    \frac{b}{2}\partial_b m_b(y).
\end{equation}

Differentiating Eq.~(\ref{eq:app:mmse-pythagoras}) twice under the
expectation, which is justified by Lemma~\ref{lem:app:posterior-expansion}
and the finite fourth moment of $Y_\rho$, and using that the integrand
vanishes at $\delta=0$, gives
\begin{equation}
    \left.
    \frac{d^2}{d\delta^2}
    [\mathcal R_\rho(\delta)-\mathcal R_\rho(0)]
    \right|_{\delta=0}
    =
    \frac{\rho}{2}
    \E\left[
        \bigl(\partial_b m_b(Y_\rho)\bigr)^2
    \right].
\end{equation}
By Lemma~\ref{lem:app:posterior-expansion},
\[
    |\partial_b m_b(y)-y|
    \leq Cb(1+y^2).
\]
Evaluating this bound at $y=Y_\rho=N+bX$ and taking the $L^2$ norm gives
\[
    \|\partial_b m_b(Y_\rho)-Y_\rho\|_2
    \leq
    Cb
    \left(\E[(1+Y_\rho^2)^2]\right)^{1/2}
    =
    O(b),
\]
since $X$ is bounded and $N$ is Gaussian, so
$\E[Y_\rho^4]$ is uniformly bounded for $0\leq b\leq1$. Moreover,
\[
    \|Y_\rho-N\|_2
    =
    b\|X\|_2
    =
    b.
\]
Hence, by the triangle inequality,
\[
    \|\partial_b m_b(Y_\rho)-N\|_2
    =
    O(b).
\]
It follows that
\[
    \E\left[
        \bigl(\partial_b m_b(Y_\rho)\bigr)^2
    \right]
    \longrightarrow
    \E[N^2]
    =
    1
\]
as $\rho\to0$. Therefore,
\begin{equation}
    \left.
    \frac{d^2}{d\delta^2}
    [\mathcal R_\rho(\delta)-\mathcal R_\rho(0)]
    \right|_{\delta=0}
    =
    \frac{\rho}{2}+o(\rho).
    \label{eq:app:R-curvature}
\end{equation}

\subsection{Sensitivity of the observable clean-target loss}
\label{app:theory:clean}

We now study the reconstruction-side objective
\begin{equation}
    G_\rho(\delta)
    =
    W_2^2(Q_{\rho,\delta},P_X),
\end{equation}
where $P_X=\operatorname{Law}(X)$ is the clean law. The goal is
to understand how this loss changes when the candidate operator
is perturbed by $\delta$, particularly when $\rho$ is small.

\paragraph{Writing the loss in terms of quantiles.}
Our first observation is that the MMSE estimator is non-decreasing in the observation.
Indeed, from Prop. 1 of \citet{dytso2023conditional},
\begin{equation}
    \eta_\gamma'(y)
    =
    \sqrt{\gamma}\,
    \operatorname{Var}
    (X\mid \sqrt{\gamma}X+N=y)
    \geq 0.
    \label{eq:app:posterior-monotone}
\end{equation}

Let us now write $b=\sqrt{\rho}$, as in the previous subsection,
and recall that $m_\beta(y)=\eta_{\beta^2}(y)$.
Let $F_X$ and $F_{Y_\rho}$ denote the cdfs of the clean
signal and observation, respectively, and define their
generalized quantile functions by
\begin{equation}
    x(u)=F_X^{-1}(u),
    \qquad
    y_b(u)=F_{Y_\rho}^{-1}(u).
\end{equation}
Because $m_{b\sqrt{1+\delta}}$ is non-decreasing, as shown earlier, it preserves
the ordering of observations: applying the reconstruction map
to the $u$-quantile of $Y_\rho$ gives the $u$-quantile of the
reconstruction distribution. Hence,
\begin{equation}
    F_{Q_{\rho,\delta}}^{-1}(u)
    =
    m_{b\sqrt{1+\delta}}(y_b(u))
    \qquad\text{for almost every }u\in(0,1).
\end{equation}

This observation allows us to exploit a classical 
representation of the Wasserstein metric in one dimension~\citep[Sec.~2.6]
{peyre2019computational}:
\begin{equation}
    G_\rho(\delta)
    =
    \int_0^1
    \left[
        x(u)
        -
        m_{b\sqrt{1+\delta}}(y_b(u))
    \right]^2
    \dd u.
    \label{eq:app:G-quantile}
\end{equation}
This expression makes the comparison explicit: the loss
matches each clean-signal quantile to the corresponding
quantile of the reconstructed signal.

\paragraph{Low SNR analysis.}
As above, set $b=\sqrt\rho$ and write
$z_\delta(u):=m_{b\sqrt{1+\delta}}(y_b(u))$.
Equation~(\ref{eq:app:G-quantile}) can be equivalently rewritten as
$G_\rho(\delta)=\|x-z_\delta\|_2^2$, where
$\|\cdot\|_2$ is the norm on $L^2(0,1)$.
Note that the quantiles $y_b$ depend on the true $b$ but not on
$\delta$. Below, primes on $z_\delta$ denote derivatives with respect
to $\delta$.

By the chain rule, with derivatives of $m_\beta$ taken with respect
to $\beta$,
\[
\begin{aligned}
    z_0'(u)
    &=\frac b2\left.\partial_\beta m_\beta(y_b(u))\right|_{\beta=b},\\
    z_0''(u)
    &=\frac{b^2}{4}\left.\partial_\beta^2 m_\beta(y_b(u))\right|_{\beta=b}
      -\frac b4\left.\partial_\beta m_\beta(y_b(u))\right|_{\beta=b}.
\end{aligned}
\]
Lemma~\ref{lem:app:posterior-expansion} bounds the deviations from the
linear terms and the second derivative by multiples of $1+y^2$.
Since $y_b$ is the quantile function of $Y_\rho=N+bX$,
\[
    \int_0^1 (1+y_b(u)^2)^2\,\dd u
    =\E[(1+Y_\rho^2)^2]
    =O(1)
\]
as $b\to 0$, because $X$ is bounded and $N$ is Gaussian.
Thus the three bounds in Lemma~\ref{lem:app:posterior-expansion}
give directly
\begin{equation}
\begin{aligned}
    \|z_0-b y_b\|_2 &= O(b^2),\\
    \left\|z_0'-\frac b2 y_b\right\|_2 &= O(b^2),\\
    \left\|z_0''+\frac b4 y_b\right\|_2 &= O(b^2).
\end{aligned}
\label{eq:app:z-derivatives}
\end{equation}

Differentiating the exact quantile identity in
Eq.~(\ref{eq:app:G-quantile}) yields
\begin{align}
    G_\rho'(0)
    &=-2\int_0^1(x-z_0)z_0'\,\dd u,\\
    G_\rho''(0)
    &=2\int_0^1(z_0')^2\,\dd u
      -2\int_0^1(x-z_0)z_0''\,\dd u.
\end{align}
Differentiation under the integral is justified by the bounds in
Lemma~\ref{lem:app:posterior-expansion} and the finite fourth moment of
$Y_\rho$.
Since $\|x\|_2=1$ and
$\|y_b\|_2^2=\E[Y_\rho^2]=1+b^2$,
Eq.~(\ref{eq:app:z-derivatives}) and Cauchy--Schwarz give
\begin{equation}
\begin{aligned}
    G_\rho'(0)&=-b C_b+O(b^2),\\
    G_\rho''(0)&=\frac b2 C_b+O(b^2),
\end{aligned}
\qquad
C_b:=\int_0^1x(u)y_b(u)\,\dd u.
\label{eq:app:G-cross-moment}
\end{equation}

The remaining issue is that $C_b$ still depends on the
observation SNR. We now identify a constant depending only on
the clean law that approximates $C_b$ with error $O(b)$.
Let $\Phi$ denote the standard-normal cdf. Since
$Y_\rho=N+bX$, the particular coupling $(Y_\rho,N)$
and the one-dimensional Wasserstein quantile formula
\citep[Sec.~2.6]{peyre2019computational} imply
\begin{equation}
\begin{aligned}
    \|y_b-\Phi^{-1}\|_2
    &=W_2(P_\rho,\mathcal N(0,1))\\
    &\leq \bigl(\E|bX|^2\bigr)^{1/2}=b.
\end{aligned}
\label{eq:app:observation-quantile-bound}
\end{equation}
Define the fixed cross-moment
\begin{equation}
    \kappa_X
    :=\int_0^1x(u)\Phi^{-1}(u)\,\dd u
    =\int_0^1F_X^{-1}(u)\Phi^{-1}(u)\,\dd u.
    \label{eq:app:kappa}
\end{equation}
Because $\|x\|_2=1$, Cauchy--Schwarz now gives
\begin{equation}
    |C_b-\kappa_X|
    \leq\|x\|_2\|y_b-\Phi^{-1}\|_2\leq b.
    \label{eq:app:cross-moment-bound}
\end{equation}
Thus $C_b=\kappa_X+O(b)$; its contribution to
Eq.~(\ref{eq:app:G-cross-moment}) has error $O(b^2)$.
The two quantile functions in Eq.~(\ref{eq:app:kappa}) are
non-decreasing, centered and non-constant, so $\kappa_X>0$;
Cauchy--Schwarz gives $\kappa_X\leq1$, with equality if and
only if $X$ is standard Gaussian.

\begin{proposition}[Clean-target sensitivity]
\label{prop:app:clean}
Under the assumptions of Theorem~\ref{thm:theory:weak},
as $\rho\to 0$,
\begin{equation}
    G_\rho'(0)
    =-\kappa_X\sqrt{\rho}+O(\rho),
    \qquad
    G_\rho''(0)
    =\frac{\kappa_X}{2}\sqrt{\rho}+O(\rho).
    \label{eq:app:G-derivatives}
\end{equation}
\end{proposition}

\begin{proof}
Substitute Eq.~(\ref{eq:app:cross-moment-bound})
into Eq.~(\ref{eq:app:G-cross-moment}) and use
$b=\sqrt\rho$.
\end{proof}

\rev{%
\subsection{From the scalar model to IKL and DPIR}
\label{app:theory:practice}

Theorem~\ref{thm:theory:weak} is stated for the KL divergence and for the
posterior mean, whereas our method uses the IKL divergence of
Appendix~\ref{sec:appendix:ikl} and the DPIR algorithm. We now explain how the
analysis relates to these two choices. We first show that the measurement-side
result of Theorem~\ref{thm:theory:weak} also holds for IKL, and we then
describe how the operator affects DPIR.

\paragraph{Smoothing along the flow path preserves the channel.}
The IKL divergence is a weighted integral of KL divergences between smoothed
distributions,
\begin{equation}
    \IKL(P\Vert Q)
    =\int_0^1 w(t)\,\operatorname{KL}\big(P^{(t)}\Vert Q^{(t)}\big)\,\dd t,
    \qquad
    w(t)=\frac{1-t}{t^2},
    \label{eq:app:ikl-def}
\end{equation}
where $P^{(t)}$ is the law of $Z_t=(1-t)Z_0+tZ_1$, with $Z_1\sim P$ and $Z_0$ a
Gaussian base variable independent of $Z_1$. Our first observation is that this
smoothing does not change the nature of the
channel~(\ref{eq:app:channel}). Indeed, if $Z_1=\sqrt{\gamma}X+N$ is an
observation of this channel at SNR $\gamma$, then
\begin{equation}
    Z_t=t\sqrt{\gamma}\,X+\big(tN+(1-t)Z_0\big).
    \label{eq:app:ikl-zt}
\end{equation}
The term in parentheses is a centered Gaussian variable, independent of $X$.
Since the measurement noise $N$ has unit variance in the
channel~(\ref{eq:app:channel}), we denote by $\tau$ the standard deviation of
the base variable $Z_0$ in the same units. In particular, $\tau=1$ when the
base noise and the measurement noise have the same variance. The Gaussian term
then has
variance $s_t^2=t^2+\tau^2(1-t)^2$, and dividing Eq.~(\ref{eq:app:ikl-zt}) by
$s_t$ gives
\begin{equation}
    \frac{Z_t}{s_t}=\sqrt{c_t\gamma}\,X+N',
    \qquad
    N'\sim\mathcal N(0,1),
    \qquad
    c_t=\frac{t^2}{t^2+\tau^2(1-t)^2}\in[0,1].
    \label{eq:app:ikl-channel}
\end{equation}
This is again an observation of the channel~(\ref{eq:app:channel}), at the
smaller SNR $c_t\gamma$. In other words, smoothing reduces the SNR of the
component without changing the form of the channel.
Since the KL divergence does not change when the same rescaling is applied to
both distributions, the smoothed divergences are those of the original channel
at a lower SNR,
\begin{equation}
    \operatorname{KL}\big(P_\gamma^{(t)}\Vert P_{\gamma'}^{(t)}\big)
    =\operatorname{KL}\big(P_{c_t\gamma}\Vert P_{c_t\gamma'}\big).
    \label{eq:app:ikl-kl}
\end{equation}

\paragraph{Curvature of the IKL loss.}
We now consider the IKL counterpart of the measurement loss of
Theorem~\ref{thm:theory:weak},
\begin{equation}
    F^{\rm IKL}_\rho(\delta)=\IKL\big(P_\rho\Vert P_{\rho(1+\delta)}\big).
\end{equation}
Exchanging the two arguments of the divergence, as in
Eq.~(\ref{eq:method:F}), does not change its second derivative at $\delta=0$.
Using Eq.~(\ref{eq:app:ikl-kl}) in the definition~(\ref{eq:app:ikl-def}), we
obtain
\begin{equation}
    F^{\rm IKL}_\rho(\delta)=\int_0^1 w(t)\,F_{c_t\rho}(\delta)\,\dd t,
    \label{eq:app:ikl-F}
\end{equation}
where $F_r$ denotes the measurement loss~(\ref{eq:theory:F}) at SNR $r$. The
IKL loss is thus a weighted integral of the losses studied in
Appendix~\ref{app:theory:measurement}, at SNRs that are all smaller
than~$\rho$. We have shown there that $F_r''(0)=r^2J(r)$, see
Eq.~(\ref{eq:app:F-J}). Differentiating Eq.~(\ref{eq:app:ikl-F}) twice under
the integral sign, which is justified by the boundedness of $X$, yields
\begin{equation}
    \left(F^{\rm IKL}_\rho\right)''(0)
    =\rho^2\int_0^1 w(t)\,c_t^2\,J(c_t\rho)\,\dd t .
\end{equation}
Since $c_t\leq1$, the SNR $c_t\rho$ tends to zero uniformly in $t$ when
$\rho\to0$, so that $J(c_t\rho)\to1/2$. Therefore,
\begin{equation}
    \left(F^{\rm IKL}_\rho\right)''(0)
    =C_\tau\,\rho^2+o(\rho^2),
    \qquad
    C_\tau:=\frac12\int_0^1\frac{t^2(1-t)}{\big(t^2+\tau^2(1-t)^2\big)^2}\,\dd t,
    \label{eq:app:ikl-curv}
\end{equation}
where the constant $C_\tau$ is positive and finite for every $\tau>0$.

When the base noise has the variance of the measurement noise ($\tau=1$), this
constant can be computed in closed form:
\begin{align}
    \int_0^1\frac{t^2(1-t)}{\big(t^2+(1-t)^2\big)^2}\,\dd t
    &=\frac12\int_0^1\frac{t(1-t)}{\big(t^2+(1-t)^2\big)^2}\,\dd t
    \notag\\
    &=\frac14\int_{-1}^{1}\frac{1-u^2}{(1+u^2)^2}\,\dd u
    =\frac14\left[\frac{u}{1+u^2}\right]_{-1}^{1}
    =\frac14 .
\end{align}
The first equality follows from the symmetry $t\mapsto1-t$, and the second one
from the change of variable $t=(1+u)/2$. Hence $C_1=1/8$.

\begin{proposition}[Measurement-side sensitivity with IKL]
\label{prop:app:ikl}
Under the assumptions of Theorem~\ref{thm:theory:weak}, as $\rho\to0$,
\begin{equation}
    \left(F^{\rm IKL}_\rho\right)''(0)=C_\tau\,\rho^2+o(\rho^2),
\end{equation}
where $C_\tau$ is the constant of Eq.~(\ref{eq:app:ikl-curv}), with $C_1=1/8$.
\end{proposition}

\begin{proof}
This is Eq.~(\ref{eq:app:ikl-curv}).
\end{proof}

The curvature of the measurement loss is thus of order $\rho^2$ with IKL, as
with KL in Theorem~\ref{thm:theory:weak}. The scale of the base noise only
changes the constant.

\paragraph{How the operator enters DPIR.}
Theorem~\ref{thm:theory:weak} uses the posterior mean as a restorer. DPIR is a
different algorithm, but the operator enters it in a similar way. DPIR
alternates a data-fidelity step and a denoising step,
\begin{equation}
    x_k=\argmin_x \tfrac12\|y-A_\theta x\|^2+\tfrac{\alpha_k}{2}\|x-z_{k-1}\|^2,
    \qquad
    z_k=D_{\sigma_k}(x_k),
    \label{eq:app:dpir}
\end{equation}
where $D_{\sigma_k}$ is the pretrained denoiser at level $\sigma_k$ and
$\alpha_k>0$ is set by the denoising schedule. Only the first step involves
the operator: the denoiser does not depend on $\theta$ and plays the role of
the prior. For a circular convolution with kernel $a$, the data-fidelity step
is diagonal in the Fourier domain. At each frequency $\omega$,
\begin{equation}
    \hat x_k(\omega)
    =\frac{\overline{\hat a(\omega)}\,\hat y(\omega)+\alpha_k\,\hat z_{k-1}(\omega)}
           {|\hat a(\omega)|^2+\alpha_k}.
    \label{eq:app:dpir-fourier}
\end{equation}
For a frequency that the operator transmits weakly, in the sense that
$|\hat a(\omega)|^2\ll\alpha_k$, this becomes
\begin{equation}
    \hat x_k(\omega)
    =\hat z_{k-1}(\omega)+\frac{\overline{\hat a(\omega)}}{\alpha_k}\,\hat y(\omega)
     +O\big(|\hat a(\omega)|^2\big).
\end{equation}
The measurement thus enters the restoration with a gain proportional to the
amplitude of the operator at this frequency. This is the behavior of the
posterior mean at low SNR, $m_\beta(y)=\beta y+O(\beta^2)$, on which the proof
of Theorem~\ref{thm:theory:weak} relies
(Lemma~\ref{lem:app:posterior-expansion}). 
}

\section{Experimental settings and additional results}
\label{sec:appendix:exp}

This appendix details the experimental settings
(Appendix~\ref{sec:appendix:exp:common}) and the operators
(Appendix~\ref{sec:appendix:kernels}), and gives additional results for each
experiment, for the unbiased variant of the reconstruction loss
(Appendix~\ref{subsec:unbiased_exp}), as well as the runtime
(Appendix~\ref{sec:appendix:exp:analysis})\rev{, the ablations of
Section~\ref{sec:exp:ablations}
(Appendices~\ref{sec:appendix:exp:sensitivity}--\ref{sec:appendix:exp:noise}),
a comparison with a single-image blind method
(Appendix~\ref{sec:appendix:exp:blinddps}), the estimation and calibration of
the noise level (Appendix~\ref{sec:appendix:exp:noise-est}), and a comparison
with diffraction theory on the real capture
(Appendix~\ref{sec:appendix:exp:real-mtf})}.

\subsection{Experimental settings}
\label{sec:appendix:exp:common}

For each task, the measurement-only configuration is fixed first and reused
for our method, with the reconstruction branch added at $\lambda=1$. Teachers
are conditional flow-matching velocity fields with $32$ base channels,
multipliers $(1,2,3,4)$, one block per resolution, attention at $32\times32$;
they are trained once per task on $8\cdot10^{6}$ images and then frozen.
Auxiliary fields share this architecture and are initialized from
the corresponding teachers.

Operator learning uses Adam with learning rate $3\cdot10^{-3}$ for the operator
and $10^{-4}$ for the auxiliary fields, batch size $32$, $64\times64$ patches,
one auxiliary update per operator update, interpolation times drawn uniformly
in $(0,1)$, $K=8$ DPIR iterations with the final $K_{\rm bp}=1$
differentiated, $\epsilon=10^{-12}$ in Eq.~(\ref{eq:method:balance}), and an
operator EMA with $\beta=0.999$. Table~\ref{tab:appendix:settings} lists the
task-specific choices, and Appendix~\ref{sec:method:details} gives further
implementation details.

\begin{table}[h]
\caption{\textbf{Per-task settings.} $K_{\rm warm}$ is the linear warm-up of
   $\lambda$ in Eq.~(\ref{eq:method:curriculum}), expressed as a fraction of the
training budget. Centering is used unless otherwise stated; see
Appendix~\ref{sec:appendix:centering}.}
\label{tab:appendix:settings}
\centering
\small
\begin{tabular}{lllll}
\toprule
& Uniform blur & Space-varying & Blind SR & Real capture \\
\midrule
Clean set $\mathcal X$    & FFHQ-256, 1k & DPDD, sharp & LR image & f/5.6 crops \\
Degraded set $\mathcal Y$ & FFHQ-256, 1k & DIV2K + PSF & LR crops, $32^2$ & f/16 frames \\
Evaluation                & FFHQ-256, 1k & DPDD test, 76 & DIV2KRK, 100 & --- \\
\midrule
Operator                  & $33\times33$ & $8\times8$ grid, $27^2$ & $11\times11$, $\times2$ & $9\times12$ grid, $13^2$ \\
Channels                  & 1 (3 for color) & 3 & 1 & 3 \\
Boundary                  & circ.\ / repl. & replicate & norm.\ conv. & replicate \\
Centering                 & off / on & on, green & on & on, green \\
Location conditioning     & yes & yes & no & yes \\
\midrule
Forward $\sigma$          & $2\cdot10^{-2}$ & $10^{-2}$ & $0$ & \rev{$8\cdot10^{-3}$} \\
$\lambda$ / $K_{\rm warm}$ & $1$ / $1/4$ & $1$ / $1/4$ & $1$ / $1/2$ & $1$ / $1/4$ \\
Budget (images)           & $4\cdot10^{6}$ & $4\cdot10^{6}$ & $5\cdot10^{5}$ & $4\cdot10^{6}$ \\
Operator lr schedule      & cosine & cosine & cosine & cosine \\
\bottomrule
\end{tabular}
\end{table}

\paragraph{Task-specific choices.}
The three motion-blur experiments use circular padding and disable
recentering; all other experiments follow the centering convention of
Appendix~\ref{sec:appendix:centering}. Blind super-resolution keeps centering,
because the operator position selects a decimation phase, and applies the
forward operator as a normalized convolution: the observation is zero-padded
and then divided by the response of the same kernel to an all-ones mask, so
that border pixels are renormalized by the support they actually received.
Forward and restorer noise levels coincide for the \rev{synthetic} deblurring
experiments; blind super-resolution is noiseless, and DPIR uses $10^{-2}$ as a
regularization level. \rev{On the real capture, the level is calibrated from the raw files of the
camera (Appendix~\ref{sec:appendix:exp:noise-est}) and used for the forward
model and for the final restorations, while the restorer used during training
runs at $10^{-2}$, the lower bound of our DPIR schedule.} Since super-resolution estimates one operator per image,
it uses a shorter budget and a longer warm-up. In blind super-resolution, a
Gaussian prior is applied to every method, ours included, through the
post-processing of Section~\ref{sec:exp:sr}; no other kernel regularizer is
used in any task.

\subsection{Operators used in the experiments}
\label{sec:appendix:kernels}

This appendix visualizes the ground-truth operators used in
Sections~\ref{sec:exp:uniform} and~\ref{sec:exp:nu}. Kernels are displayed on a
square-root intensity scale and normalized to their own maximum to make
low-amplitude structure visible.

\paragraph{Spatially uniform kernels}
\label{sec:appendix:kernels:uniform}

Figure~\ref{fig:appendix:kernels-uniform} shows the nine uniform degradations of
Table~\ref{tab:exp:uniform}, represented on the $33\times33$ support used for
operator estimation.

\begin{figure}[t]
\centering
\includegraphics[width=\textwidth]{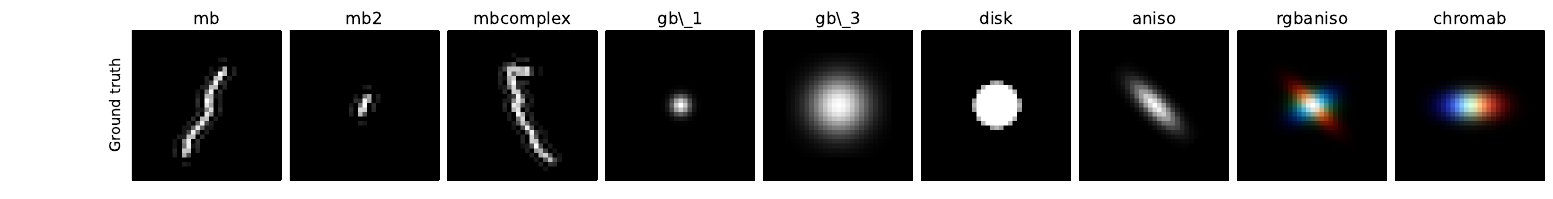}
\caption{\textbf{Uniform degradations used in our experiments.}
The set contains three motion kernels, four analytic PSFs, and two
channel-dependent operators. Each kernel is shown on a square-root intensity
scale and normalized to its own maximum. The final two are displayed as RGB
composites to reveal differences between color channels.}
\label{fig:appendix:kernels-uniform}
\end{figure}

The first three operators are irregular motion-blur kernels generated from
random trajectories of increasing complexity and spatial extent.
The analytic PSFs contain isotropic Gaussians
with $\sigma\in\{1,3\}$, a defocus disk of radius $5$, and an anisotropic
Gaussian with $\sigma=(3,1)$ rotated by $45^\circ$. Together they cover
smooth isotropic blur, directional blur, defocus, and irregular motion.
The motion kernels use circular padding and do not apply the centering
correction; all other kernel experiments follow the convention of
Appendix~\ref{sec:appendix:centering}.

The two color operators test channel-dependent degradation. The first uses a
different anisotropic Gaussian in each channel while keeping their centers
aligned. The second models lateral chromatic aberration by shifting the red and
blue kernels by $\pm2$ pixels relative to green. We therefore use the green
channel as the centering anchor so that the common translation is removed
without eliminating the relative chromatic displacement; see
Appendix~\ref{sec:appendix:centering}.

\paragraph{Measured space-varying PSF}
\label{sec:appendix:kernels:nu}

The operator of Section~\ref{sec:exp:nu} is derived from the measured PSF field
of a Canon EF 24\,mm f/1.4L USM lens at full aperture released by
\citet{bauer2018automatic}. We sample the measured field on an $8\times8$
regular grid, crop each PSF to $55\times55$, resample it to $27\times27$, and
normalize each channel to unit sum. The resulting grid is centered using the
same green-anchored convention as the estimated operator.

\begin{figure}[t]
\centering
\includegraphics[width=0.75\textwidth]{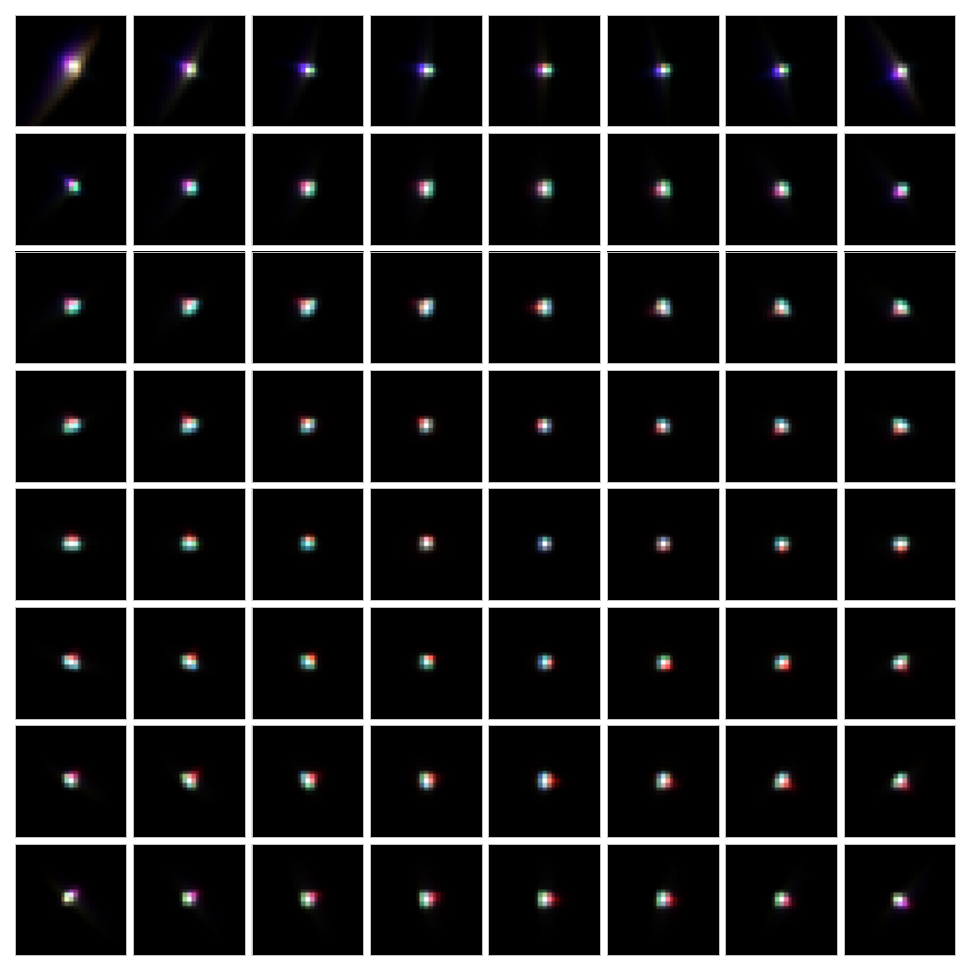}
\caption{\textbf{Measured $8\times8$ PSF field} from a Canon EF 24\,mm
f/1.4L USM lens at f/1.4~\citep{bauer2018automatic}. Each tile shows the
$27\times27$ RGB PSF at one field position and is normalized to its own maximum
for visualization.}
\label{fig:appendix:psf-field}
\end{figure}

\subsection{Spatially uniform blur}
\label{sec:appendix:exp:uniform}

The clean, degraded, and evaluation sets contain $1000$ FFHQ-256 images each,
drawn from disjoint partitions. Table~\ref{tab:exp:uniform} gives the
per-degradation results summarized in Table~\ref{tab:exp:uniform-summary}, and
Table~\ref{tab:appendix:uniform-full} the corresponding SSIM and LPIPS.
\rev{For these evaluations, the denoiser of DPIR is the diffusion U-Net trained
on FFHQ that is available in the DeepInverse
library~\citep{tachella2025deepinverse}, which we use as a
Gaussian denoiser at the noise level of each DPIR iteration. Operator learning
uses DRUNet, as in all our experiments.}

\begin{table}[h]
\caption{\textbf{Spatially uniform blur on FFHQ-256.} Kernel error:
shift-aligned kernel MAE ($\times10^{-4}$); restoration is RGB PSNR (dB) after
DPIR with the estimated operator. \emph{Input} is the degraded image and
\emph{True op.} uses the ground-truth operator.}
\label{tab:exp:uniform}
\centering
\small
\begin{tabular}{lrrcrrrr}
\toprule
& \multicolumn{2}{c}{Kernel error $\downarrow$}
& & \multicolumn{4}{c}{Restoration PSNR $\uparrow$} \\
\cmidrule(lr){2-3}\cmidrule(lr){5-8}
Degradation & Meas.-only & Ours & & Input & Meas.-only & Ours & True op. \\
\midrule
Motion \texttt{mb}          & 3.49 & \best{1.74} & & 22.12 & 31.23 & \best{31.63} & 31.71 \\
Motion \texttt{mb2}         & \best{2.02} & 2.94 & & 28.76 & 34.12 & \best{34.56} & 34.57 \\
Motion \texttt{mbcomplex}   & 3.67 & \best{1.71} & & 21.68 & 31.64 & \best{32.18} & 32.19 \\
Gaussian $\sigma{=}1$       & 1.74 & \best{0.87} & & 30.01 & 33.19 & \best{34.75} & 34.85 \\
Gaussian $\sigma{=}3$       & \best{1.20} & 1.25 & & 24.98 & \best{29.23} & 29.21 & 29.52 \\
Defocus disk                & 2.50 & \best{2.16} & & 25.12 & 30.39 & \best{30.76} & 31.17 \\
Anisotropic Gaussian        & 1.77 & \best{1.47} & & 26.41 & 30.88 & \best{31.05} & 31.36 \\
Per-channel anisotropic     & 1.71 & \best{0.95} & & 26.91 & 31.87 & \best{32.26} & 32.35 \\
Chromatic aberration        & 1.37 & \best{0.99} & & 25.16 & 31.54 & \best{31.77} & 31.88 \\
\midrule
Mean                        & 2.16 & \best{1.56} & & 25.68 & 31.57 & \best{32.02} & 32.18 \\
\bottomrule
\end{tabular}
\end{table}

\begin{table}[h]
\caption{\textbf{Spatially uniform blur: SSIM and LPIPS} at the last checkpoint.}
\label{tab:appendix:uniform-full}
\centering
\small
\begin{tabular}{lrrrcrrr}
\toprule
& \multicolumn{3}{c}{SSIM $\uparrow$} & & \multicolumn{3}{c}{LPIPS $\downarrow$} \\
\cmidrule(lr){2-4}\cmidrule(lr){6-8}
Degradation & Meas.-only & Ours & True op. & & Meas.-only & Ours & True op. \\
\midrule
Motion \texttt{mb}        & 0.8744 & \best{0.8791} & 0.8779 & & 0.0867 & \best{0.0771} & 0.0789 \\
Motion \texttt{mb2}       & \best{0.9274} & 0.9272 & 0.9267 & & 0.0397 & \best{0.0352} & 0.0405 \\
Motion \texttt{mbcomplex} & 0.8845 & \best{0.8915} & 0.8915 & & 0.0819 & \best{0.0699} & 0.0720 \\
Gaussian $\sigma{=}1$     & 0.9249 & \best{0.9292} & 0.9281 & & 0.0492 & \best{0.0467} & 0.0464 \\
Gaussian $\sigma{=}3$     & \best{0.8257} & 0.8192 & 0.8272 & & \best{0.2061} & 0.2197 & 0.2053 \\
Defocus disk              & \best{0.8551} & 0.8537 & 0.8671 & & 0.1435 & \best{0.1287} & 0.1281 \\
Anisotropic Gaussian      & \best{0.8737} & 0.8686 & 0.8732 & & 0.1339 & \best{0.1325} & 0.1332 \\
Per-channel anisotropic   & \best{0.8895} & 0.8885 & 0.8903 & & \best{0.1145} & 0.1172 & 0.1113 \\
Chromatic aberration      & \best{0.8815} & 0.8797 & 0.8812 & & \best{0.1301} & 0.1310 & 0.1284 \\
\bottomrule
\end{tabular}
\end{table}

\subsection{Measured PSF field with chromatic aberration}
\label{sec:appendix:exp:nu}

The degraded set consists of DIV2K images blurred with the measured field of
Appendix~\ref{sec:appendix:kernels:nu} and corrupted by Gaussian noise with
$\sigma=10^{-2}$; the clean set consists of the sharp DPDD training images.
Restoration is evaluated on the $76$ DPDD test images at full resolution,
degraded with the true field. Table~\ref{tab:appendix:nu-metrics} adds SSIM to
the PSNR and LPIPS of Table~\ref{tab:exp:nu}: all metrics rank the methods
identically.

\begin{table}[h]
\caption{\textbf{Measured PSF field: secondary restoration metrics} on the
$76$ DPDD test images.}
\label{tab:appendix:nu-metrics}
\centering
\small
\begin{tabular}{lrr}
\toprule
Method & SSIM $\uparrow$ & LPIPS $\downarrow$ \\
\midrule
Restormer~\citep{zamir2022restormer} & 0.854 & 0.205 \\
OmniLens~\citep{jiang2024flexible}   & 0.862 & 0.203 \\
\midrule
Measurement-only & 0.838 & 0.207 \\
Ours             & \best{0.925} & \best{0.132} \\
True operator    & 0.939 & 0.111 \\
\bottomrule
\end{tabular}
\end{table}

\begin{figure}[h]
\centering
\includegraphics[width=\textwidth]{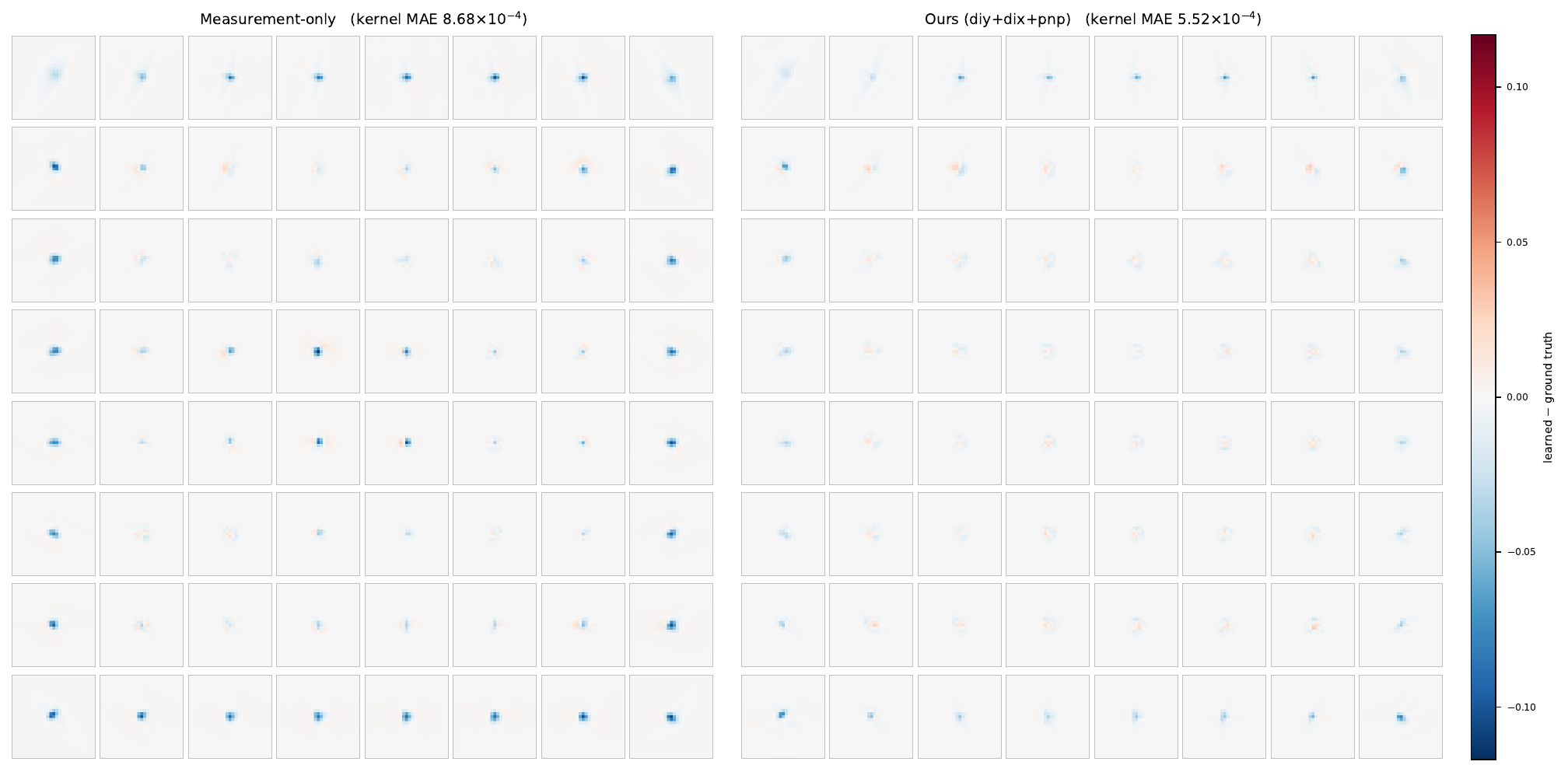}
\caption{\textbf{Error of the recovered PSF field} (estimate $-$ true PSF, each
normalized to unit sum) at the $8\times8$ field positions, for measurement-only
matching (left) and ours (right), on a shared color scale. }
\label{fig:psf-field-pair}
\end{figure}

\subsection{Blind super-resolution}
\label{sec:appendix:exp:sr}

\paragraph{Protocol.}
For each image, the degraded set consists of $32\times32$ crops of the
low-resolution image and the clean set of $64\times64$ crops of the same image,
so that the objective matches the law of $A\cdot x_{64}$ to that of $x_{32}$,
following the cross-scale recurrence assumption of~\citet{bell2019kernelgan}.
We estimate one $11\times11$ kernel per image, the native size of the DIV2KRK
ground truth, apply it with the normalized-convolution boundary condition of
Appendix~\ref{sec:appendix:exp:common}, and restore with DPIR, using the
estimated kernel as an explicit $\times2$ downsampling operator. We report PSNR
on the luminance channel (PSNR-Y) and in RGB, SSIM, LPIPS, and the
center-of-mass-aligned kernel MAE on a common $33\times33$ canvas, which is
needed to compare methods that predict kernels of different supports.
KernelGAN's kernels, stored on a $17\times17$ grid, are first cropped to
$11\times11$, so that the zeroed border covers the same fraction of the support
for every method. Paired comparisons use two-sided sign-flip permutation tests
on the per-image differences ($20{,}000$ permutations).

\paragraph{Post-processing.}
The common Gaussian prior of Section~\ref{sec:exp:sr} is applied identically to
every estimated kernel before restoration: the kernel is centered, its outer
one-pixel border is set to zero, it is renormalized, and it is replaced by the
anisotropic Gaussian matching its first two moments. The border is cropped
before the projection, since outer mass would otherwise bias the moments. The
projection needs no tuning: blending the kernel with its projection
($\alpha=0.75$) and replacing it entirely ($\alpha=1$) are statistically
indistinguishable ($+0.02$~dB, $p=0.25$), and we use $\alpha=1$.
For the baselines, centering is a half-pixel correction of the kernel
position. The position sets the decimation phase, so this correction matters
for restoration while leaving the kernel shape untouched
(Appendix~\ref{sec:appendix:centering}); applying it only to our own estimate
would be an unfair advantage.

\paragraph{Full results}
\label{sec:appendix:exp:sr:full}

Table~\ref{tab:exp:sr:full} presents the results of every method with its raw kernel, and after post-processing. The centering correction alone is worth $+1.31$~dB to
KernelGAN and $+0.43$~dB to DIP-DKP ($34.45$~dB when centered only), while
moving their aligned kernel error by at most $0.05$.
The Gaussian part of the post-processing is then strongly method-dependent and
inversely related to how much each method already constrains its estimate:
$+0.61$~dB for ours, $+0.31$ for KernelGAN, $+0.14$ for DCLS, and $+0.09$ for
DIP-DKP ($+1.73$ for measurement-only matching). The post-processing never harms a method on
average, which is why the main text reports post-processed numbers; without
it, ours still has the best PSNR. For KernelGAN, centering and projection
together add $1.62$~dB in PSNR-Y ($32.60\rightarrow34.22$); with restoration by
DPIR instead of ZSSR, they explain why our KernelGAN results are higher than
those reported by~\citet{bell2019kernelgan}.

\begin{table}[h]
\caption{\textbf{Blind SR on DIV2KRK ($\times2$), full results.} Each method
is shown with its raw kernel and after the post-processing of
Section~\ref{sec:exp:sr} (border zeroed, recentred, renormalized, then replaced
by the anisotropic Gaussian matching its first two moments; for the baselines,
this also includes the half-pixel centering correction). Kernel error is
COM-aligned MAE ($\times10^{-4}$) on a common $33\times33$ canvas.
\textbf{Bold} marks the best entry in a column together with every entry not
significantly different from it (two-sided paired permutation test,
$p\geq0.05$).}
\label{tab:exp:sr:full}
\centering
\small
\begin{tabular}{llrrrrr}
\toprule
Method & & Kernel err. $\downarrow$ & PSNR-Y $\uparrow$ & PSNR $\uparrow$ & SSIM $\uparrow$ & LPIPS $\downarrow$ \\
\midrule
KernelGAN  & raw                & 3.50 & 32.60 & 30.89 & 0.8172 & 0.2769 \\
           & centered only      & 3.50 & 33.91 & 32.17 & 0.8438 & 0.2715 \\
           & post-processed     & 3.42 & 34.22 & 32.46 & 0.8481 & 0.2714 \\
\midrule
DCLS       & raw                & 8.74 & 33.45 & 31.71 & 0.8111 & 0.3219 \\
           & post-processed     & 6.97 & 33.59 & 31.84 & 0.8160 & 0.3129 \\
\midrule
DIP-DKP    & raw                & 3.02 & 34.02 & 32.28 & 0.8497 & \best{0.2611} \\
           & post-processed     & 2.57 & 34.54 & 32.77 & \best{0.8575} & \best{0.2599} \\
\midrule
Measurement-only & raw          & 5.43 & 32.22 & 30.52 & 0.7958 & 0.2931 \\
           & post-processed     & 2.96 & 33.95 & 32.20 & 0.8459 & 0.2677 \\
\midrule
Ours       & raw                & 3.76 & 34.38 & 32.62 & 0.8486 & 0.2705 \\
           & post-processed     & \best{2.15} & \best{34.99} & \best{33.21} & \best{0.8561} & \best{0.2601} \\
\midrule
True operator &                 & ---  & 35.62 & 33.81 & 0.8724 & 0.2435 \\
\bottomrule
\end{tabular}
\end{table}

\subsection{Proof of concept for real lens calibration}
\label{sec:appendix:exp:real}

\paragraph{Acquisition.}
A single textured outdoor scene is photographed with a consumer camera
(Panasonic Lumix GX9 with 
Leica DG Summilux 25mm f/1.4 II
ASPH lens) at two apertures. The seven f/16 frames, which exhibit
stronger diffraction blur, are used at full sensor resolution
($5200\times3904$) and form $\mathcal Y$; the eight f/5.6 frames are
center-cropped to $768\times768$ and form $\mathcal X$
(Figure~\ref{fig:appendix:real-data}). The camera is reframed between
apertures, so that no image in one set corresponds to any image in the other,
and the content of the two sets differs: the f/16 frames cover the whole field,
while the f/5.6 crops are taken from its center. The camera is also rotated
between shots, including upside-down, so that no orientation is systematically
associated with either aperture.

\paragraph{Estimation and baselines.}
We estimate a $9\times12$ grid of $13\times13$ RGB kernels, matching the sensor
aspect ratio. The kernels are treated as free parameters, without any lens or
chromatic-aberration model, and are learned with the recipe of the other
space-varying experiment and the values of Table~\ref{tab:appendix:settings};
\rev{the noise level is calibrated from the raw files of the camera
(Appendix~\ref{sec:appendix:exp:noise-est})}. Both baselines restore
each f/16 frame on its own: the two-step method of~\citet{eboli2022fast}, which
estimates local Gaussian PSFs and corrects lateral chromatic aberration, is run
with its official code; OmniLens~\citep{jiang2024flexible}, a network
trained on the PSF fields of a large library of automatically designed lenses,
is run with the released pretrained weights.

\paragraph{Evaluation.}
No paired sharp image or calibrated PSF is available for this capture, so we
report neither PSNR nor kernel error. We also avoid reprojection consistency as
a primary metric: in our tests, it can be small for operators that still
visibly under-deblur the image. \rev{The amount of blur of the learned kernels
can nevertheless be compared with diffraction theory
(Appendix~\ref{sec:appendix:exp:real-mtf}).} Figure~\ref{fig:appendix:real}
shows the recovered PSF fields.
Figure~\ref{fig:exp:real2} presents additional qualitative results on several crops.

\begin{figure}[h]
\centering
\includegraphics[width=\textwidth]{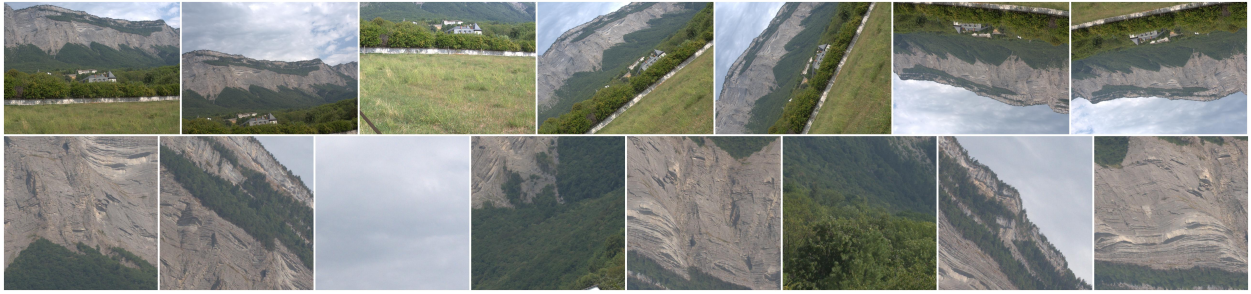}
\caption{\textbf{The fifteen images of the real capture.} Top: the seven f/16
frames forming $\mathcal Y$, at full sensor resolution ($5200\times3904$).
Bottom: the eight f/5.6 frames forming $\mathcal X$, center-cropped to
$768^2$. The camera is reframed and rotated between shots, so that the two sets
share no image correspondence and no common orientation; images are shown as
captured.}
\label{fig:appendix:real-data}
\end{figure}

\begin{figure}[h]
\centering
\includegraphics[width=\textwidth]{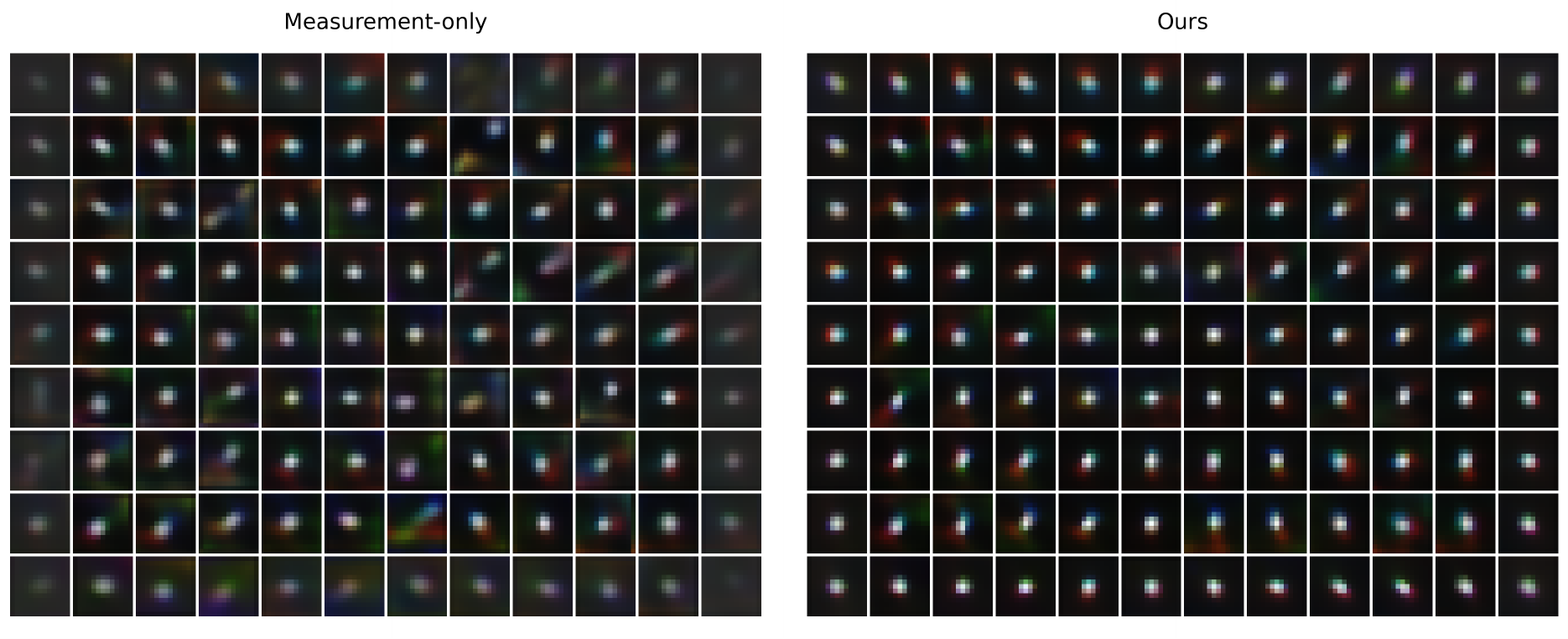}
\caption{\textbf{Recovered PSF fields on the real capture.} Learned $9\times12$
grid of $13\times13$ RGB kernels after $4\cdot10^{6}$ training images, for
measurement-only matching (left) and ours (right), laid out by field position.
Each PSF is normalized to unit sum and displayed as an RGB composite, so a
colored fringe indicates a channel-dependent shift.}
\label{fig:appendix:real}
\end{figure}

\begin{figure}[h]
\centering
   \includegraphics[width=\textwidth]{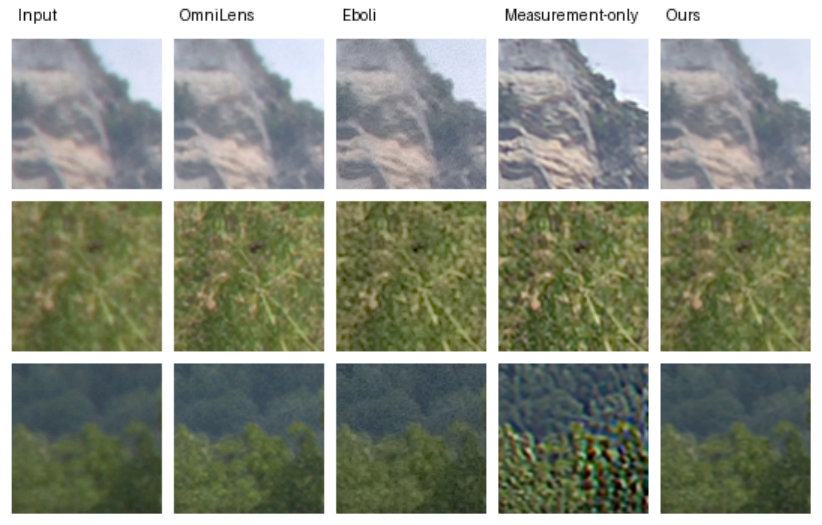}
   \caption{\textbf{Real lens calibration. Additional crops.}  First row: note the green fringe between vegetation and sky. OmniLens does not remove it whereas other approaches do. On the last row, measurement-only catastrophically fails whereas other approaches perform well. Residual noise artefacts are still present in the method of~\citet{eboli2022fast} despite trying many parameters.}
\label{fig:exp:real2}
\end{figure}

\subsection{Unbiased reconstruction loss}
\label{subsec:unbiased_exp}

\paragraph{Loss.}
As shown in Section~\ref{sec:theory}, the clean-side loss $G$ is biased: with
a point-estimate restorer, restorations of real measurements are not
distributed like clean images, even with the true operator. The bias can be
removed by comparing restorations of real measurements with restorations of
synthetic measurements generated by the candidate operator:
\begin{equation}
    \widetilde G(\theta) := \mathcal D\big(Q_\theta \,\Vert\, Q_\theta^{\rm syn}\big),
    \qquad
    Q_\theta^{\rm syn} := \operatorname{Law}\big(R_\theta(A_\theta X+\varepsilon)\big),
    \quad X\sim P_X,
    \label{eq:app:two-sample}
\end{equation}
where $Q_\theta = \operatorname{Law}(R_\theta(Y))$ and $\mathcal D$ is the
weighted IKL of Appendix~\ref{sec:appendix:ikl}. Under the model of
Eq.~(\ref{eq:method:model}), $A_{\theta^\star}X+\varepsilon$ has the law of
$Y$, so $Q_{\theta^\star}=Q^{\rm syn}_{\theta^\star}$ and
$\widetilde G(\theta^\star)=0$: since $\mathcal D\ge 0$, the true operator
minimizes $\widetilde G$, whatever the restorer. The two losses differ only in
their target: $G$ compares restorations of real measurements with clean images,
whereas $\widetilde G$ compares them with clean images as rendered by the
candidate degradation and the restorer. They coincide for a restorer that
preserves the clean distribution, such as posterior sampling, for which
$Q^{\rm syn}_\theta=P_X$ for every $\theta$.

\paragraph{Implementation.}
We estimate the gradient of $\widetilde G$ with the estimator of
Eq.~(\ref{eq:method:ikl-gradient}), taking $Z_1(\theta)=R_\theta(Y)$ as for
$G$. Since the target $Q^{\rm syn}_\theta$ depends on $\theta$, it cannot be
represented by a frozen teacher: the clean-image teacher $v_X^{\rm tc}$ is
replaced by a second auxiliary field, with the same architecture and training
as $v_X^{\rm aux}$, trained online on detached restorations
$R_\theta(A_\theta X+\varepsilon)$ and treated as fixed during the operator
update.
Each operator update therefore requires a second restoration and a second
auxiliary-field update, which makes $\widetilde G$ more expensive than $G$.
All other settings are those of Appendix~\ref{sec:appendix:exp:common}.

\begin{table}[h]
\caption{\textbf{Unbiased loss $\widetilde G$ versus the default $G$} on the
deblurring tasks, at the last checkpoint. The kernel-error change is relative
to $G$; $\Delta$PSNR is computed before rounding.}
\label{tab:app:two-sample}
\centering
\small
\begin{tabular}{lrrrr}
\toprule
Degradation & $\Delta$ kernel err. & PSNR, $G$ & PSNR, $\widetilde G$ & $\Delta$PSNR \\
\midrule
Motion \texttt{mb}              & $+19.5\%$ & 31.63 & 31.61 & $-0.02$ \\
Motion \texttt{mb2}             & $+14.3\%$ & 34.56 & 34.58 & $+0.02$ \\
Motion \texttt{mbcomplex}       & $+39.2\%$ & 32.18 & 32.25 & $+0.08$ \\
Gaussian $\sigma{=}1$           & $+32.2\%$ & 34.75 & 34.33 & $-0.42$ \\
Gaussian $\sigma{=}3$           & $-35.9\%$ & 29.21 & 29.43 & $+0.23$ \\
Defocus disk                    & $-2.0\%$  & 30.76 & 30.72 & $-0.04$ \\
Anisotropic Gaussian           & $-24.9\%$ & 31.05 & 31.24 & $+0.19$ \\
Per-channel anisotropic         & $+9.9\%$  & 32.26 & 32.17 & $-0.09$ \\
Chromatic aberration            & $-10.7\%$ & 31.77 & 31.75 & $-0.02$ \\
Measured PSF field              & $-5.2\%$  & 32.41 & 32.35 & $-0.06$ \\
\midrule
Mean                            & ---       & 32.06 & 32.04 & $-0.01$ \\
\bottomrule
\end{tabular}
\end{table}

\paragraph{Results.}
We rerun the nine uniform degradations and the measured PSF field with
$\widetilde G$ in place of $G$, as well as blind super-resolution on the same
$100$ DIV2KRK images, and report the last checkpoint.
On average, replacing $G$ by $\widetilde G$ leaves restoration unchanged
($-0.01$~dB over the ten deblurring settings, better on four;
Table~\ref{tab:app:two-sample}), which supports the statement of
Section~\ref{sec:theory} that the bias of $G$ does not degrade performance in
practice. Individual settings move in both directions, most visibly for the
Gaussian-type kernels ($-0.42$~dB for $\sigma=1$, $+0.23$~dB for $\sigma=3$,
$+0.19$~dB for the anisotropic Gaussian). Large changes in kernel error need
not reach the restored image: on the three motion kernels, kernel MAE
increases by $14$ to $39\%$ while PSNR changes by at most $0.08$~dB. On blind
super-resolution, paired over the $100$ images, $\widetilde G$ improves
restoration with raw kernels ($+0.10$~dB, median $+0.02$~dB, $56/100$ images,
$p=0.001$) but not after the post-processing of Section~\ref{sec:exp:sr}
($-0.04$~dB, median $-0.04$~dB, $31/100$ images, $p=0.06$), suggesting that
both correct the same defect of the raw estimates. Since $\widetilde G$ brings
no measurable benefit at a higher cost (roughly doubling the computational time), we keep $G$ as the default.

\subsection{Runtime Analysis}
\label{sec:appendix:exp:analysis}

\label{sec:appendix:exp:speed}
The reconstruction branch adds one solver call per operator update. On
identical hardware, it increases the time spent in operator updates by
$5.7\times$ to $6.6\times$ (Table~\ref{tab:appendix:speed}). This cost is
incurred only while estimating the operator; once learned, both methods produce
the same type of explicit operator and therefore have identical restoration
cost. \rev{Since operator learning uses batches of $32$ patches of size
$64\times64$, the cost of an update does not depend on the image resolution;
for instance, the frames of the real capture have $5200\times3904$ pixels. The
peak GPU memory is $2.2$~GB for uniform blur and $3.2$~GB for the PSF field.}

\begin{table}[h]
\caption{\textbf{Cost of the reconstruction branch.} Mean time spent in the
optimization steps per $131K$ training images ($4096$ operator updates at
batch size $32$), excluding data loading, validation and checkpointing, and
discarding the first $131K$ images; all runs on identical GPUs.
}
\label{tab:appendix:speed}
\centering
\small
\begin{tabular}{lrrr}
\toprule
Experiment & Meas.-only (s) & Ours (s) & Overhead \\
\midrule
Uniform, \texttt{mb}           & 212 & 1401 & $6.6\times$ \\
Uniform, Gaussian $\sigma{=}1$ & 235 & 1492 & $6.3\times$ \\
Space-varying PSF              & 309 & 1762 & $5.7\times$ \\
\bottomrule
\end{tabular}
\end{table}

\rev{%
\subsection{Sensitivity to training choices}
\label{sec:appendix:exp:sensitivity}

Each run below differs from the reference configuration by a single element.
The only exception is the experiment with fixed weights, which removes both the
balancing rule and its warm-up. All tables report the kernel error
($\times10^{-4}$) and the PSNR (dB) at the last checkpoint.

\paragraph{Seeds and initialization.}
We first repeat the experiments on five settings with two additional seeds
(Table~\ref{tab:app:seeds}). Across the three seeds, the standard deviation of
the PSNR of our method is at most $0.08$~dB. For a given setting and a given
seed, our method yields a better PSNR than measurement-only matching in $14$ of
the $15$ cases. The exception is the $\sigma=3$ Gaussian with the seed used in
the paper, where the difference is $0.02$~dB. We then replace the random
diffuse initialization of Appendix~\ref{sec:appendix:impl:repro} by a
near-delta kernel (Table~\ref{tab:app:init}). For uniform blur, this kernel is
an isotropic Gaussian with a standard deviation of $0.5$ pixel. For the PSF
field, whose parameterization does not support a Gaussian initialization, it is
a centered $4\times4$ square. The comparison between the two methods does not
change.

\begin{table}[h]
\revcolor
\caption{\textbf{Seeds.} Kernel error / PSNR for the seed used in the paper
(41) and for two additional seeds.}
\label{tab:app:seeds}
\centering
\footnotesize
\setlength{\tabcolsep}{3pt}
\begin{tabular}{lccccccc}
\toprule
& \multicolumn{3}{c}{Measurement-only} & & \multicolumn{3}{c}{Ours} \\
\cmidrule(lr){2-4}\cmidrule(lr){6-8}
Setting & seed 41 & seed 42 & seed 43 & & seed 41 & seed 42 & seed 43 \\
\midrule
Gaussian $\sigma{=}1$ & 1.74 / 33.19 & 2.66 / 32.94 & 2.28 / 33.07 & & \best{0.87} / \best{34.75} & \best{1.25} / \best{34.60} & \best{0.94} / \best{34.73} \\
Gaussian $\sigma{=}3$ & \best{1.20} / \best{29.23} & 1.28 / 29.18 & 1.35 / 29.19 & & 1.25 / 29.21 & \best{1.11} / \best{29.23} & \best{1.04} / \best{29.30} \\
Motion \texttt{mbcomplex} & 3.67 / 31.64 & 3.58 / 31.55 & 3.46 / 31.62 & & \best{1.71} / \best{32.18} & \best{1.75} / \best{32.18} & \best{1.69} / \best{32.19} \\
Per-channel anisotropic & 1.71 / 31.87 & 1.90 / 31.80 & 1.85 / 31.75 & & \best{0.95} / \best{32.26} & \best{0.93} / \best{32.26} & \best{1.15} / \best{32.22} \\
Measured PSF field & 8.68 / 23.85 & 8.57 / 24.21 & 8.80 / 23.93 & & \best{5.52} / \best{32.41} & \best{5.44} / \best{32.42} & \best{5.51} / \best{32.39} \\
\bottomrule
\end{tabular}
\end{table}

\begin{table}[h]
\revcolor
\caption{\textbf{Initialization.} Kernel error / PSNR with the random diffuse
initialization used in the paper and with a near-delta initialization.}
\label{tab:app:init}
\centering
\footnotesize
\setlength{\tabcolsep}{4pt}
\begin{tabular}{lccccc}
\toprule
& \multicolumn{2}{c}{Measurement-only} & & \multicolumn{2}{c}{Ours} \\
\cmidrule(lr){2-3}\cmidrule(lr){5-6}
Setting & random & near-delta & & random & near-delta \\
\midrule
Gaussian $\sigma{=}1$ & 1.74 / 33.19 & 2.27 / 33.13 & & \best{0.87} / \best{34.75} & \best{0.98} / \best{34.74} \\
Gaussian $\sigma{=}3$ & \best{1.20} / \best{29.23} & \best{1.16} / \best{29.24} & & 1.25 / 29.21 & 1.20 / 29.21 \\
Motion \texttt{mbcomplex} & 3.67 / 31.64 & 3.43 / 31.67 & & \best{1.71} / \best{32.18} & \best{2.06} / \best{32.15} \\
Measured PSF field & 8.68 / 23.85 & 8.66 / 24.22 & & \best{5.52} / \best{32.41} & \best{5.45} / \best{32.43} \\
\bottomrule
\end{tabular}
\end{table}

\paragraph{Truncation and balancing.}
Table~\ref{tab:app:kbp} reports the results obtained when differentiating
through more DPIR iterations, and when replacing the balancing rule by a fixed
weight.

\emph{Truncation.} These runs keep the learning rates and the balancing rule
of the reference configuration, which were chosen for $K_{\rm bp}=1$ and which
we did not tune again. With these settings, increasing $K_{\rm bp}$ does not
help. The computational cost
increases: on the $\sigma=1$ Gaussian, the peak memory grows linearly with
$K_{\rm bp}$, from $2.2$~GB for $K_{\rm bp}=1$ to $5.8$~GB for $K_{\rm bp}=4$,
and an operator update takes $1.4$ times longer for $K_{\rm bp}=2$ and $1.8$
times longer for $K_{\rm bp}=4$. At the same time, the kernels become less
accurate in the three settings we have considered. The PSNR decreases on the $\sigma=1$ Gaussian, and
it is essentially unchanged on \texttt{mbcomplex} and on the PSF field. For the
latter, $K_{\rm bp}=4$ yields $5.67$ / $32.49$ instead of $5.52$ / $32.41$ for
$K_{\rm bp}=1$, with a peak memory of $7.7$~GB instead of $3.2$~GB. Note that
with $K_{\rm bp}=1$, the restorations for the two uniform kernels are already
within $0.10$ and $0.01$~dB of those obtained with the true operator
(Table~\ref{tab:exp:uniform}). There is thus little room for a longer horizon
to improve them. The norm of
the reconstruction gradient also increases with $K_{\rm bp}$: on the
$\sigma=1$ Gaussian, the median value of $\|g_G\|/\|g_F\|$ after the warm-up is
$1.3$, $1.6$, and $2.2$ for $K_{\rm bp}=1$, $2$, and $4$, respectively. This is
the amplification that the balancing rule controls. The choice
$K_{\rm bp}=1$ corresponds to one-step differentiation~\citep{bolte2023onestep}, a cheap
estimate of the full derivative: the dependence of the restoration on the
operator is differentiated through the last data-fidelity step only, whose
input is held fixed. This is the step analyzed in
Appendix~\ref{app:theory:practice}. Using a
short backpropagation horizon is also common in other contexts, where it often
performs comparably to the exact gradient~\citep{shaban2019truncated}.

\emph{Balancing.} We replace the balancing rule and its warm-up by a constant
weight on $g_G$. We consider three values: the median of $c_k$ in the reference
run, and $0.3$ and $3$ times this median. On the $\sigma=1$ Gaussian, the
results obtained with the two largest weights are within the variations
observed across seeds for the reference (Table~\ref{tab:app:seeds}), whereas
the smallest weight loses part of the gain. On \texttt{mbcomplex}, the smallest
weight also loses part of the gain, and the two largest ones diverge. A
constant weight that is appropriate for one kernel may thus fail for another
one, whereas we use the balancing rule without any change in all our
experiments.

\begin{table}[h]
\revcolor
\caption{\textbf{Truncation and balancing.} Kernel error / PSNR. The peak
memory and the time per operator update (relative to the reference) are
measured on the $\sigma=1$ Gaussian.}
\label{tab:app:kbp}
\centering
\footnotesize
\setlength{\tabcolsep}{4pt}
\begin{tabular}{lcccc}
\toprule
Variant & Gaussian $\sigma{=}1$ & Motion \texttt{mbcomplex} & Peak memory & Time \\
\midrule
Reference ($K_{\rm bp}=1$, balancing) & 0.87 / 34.75 & 1.71 / 32.18 & 2.2 GB & $\times1$ \\
$K_{\rm bp}=2$ & 1.34 / 34.34 & 2.20 / 32.13 & 3.4 GB & $\times1.4$ \\
$K_{\rm bp}=4$ & 1.56 / 33.96 & 3.76 / 32.28 & 5.8 GB & $\times1.8$ \\
\midrule
Fixed weight, $0.3\times$ median & 1.64 / 34.08 & 3.02 / 31.91 & & \\
Fixed weight, median             & 1.18 / 34.66 & 17.30 / 13.57 & & \\
Fixed weight, $3\times$ median   & 1.20 / 34.74 & 16.52 / 13.23 & & \\
\bottomrule
\end{tabular}
\end{table}

\subsection{Restoration with other non-blind solvers}
\label{sec:appendix:exp:solvers}

Our operators are learned with DPIR with 8 iterations as a restorer. To check that they are
not specific to this setting, we restore the test images with other solvers,
without learning the operators again (Table~\ref{tab:app:solvers}). We consider
DPIR with $4$ or $16$ iterations instead of $8$, and
DiffPIR~\citep{zhu2023diffpir} with $100$ steps, which uses the FFHQ diffusion
model for uniform blur and the DRUNet denoiser for the PSF field. Recall that
on FFHQ, the DPIR solver used for evaluation already relies on a different
denoiser than the one used for operator learning
(Section~\ref{sec:experiments}). On average, our operators perform better than
those of measurement-only matching with every solver.

\begin{table}[h]
\revcolor
\caption{\textbf{Restoration with other solvers.} PSNR (dB). For uniform
blur, we report the mean over the nine degradations, and the number of
degradations for which our method is better than measurement-only matching.}
\label{tab:app:solvers}
\centering
\small
\begin{tabular}{llrrrc}
\toprule
& Solver & Meas.-only & Ours & True op. & Ours better \\
\midrule
Uniform blur & DPIR, $8$ it.\ (Table~\ref{tab:exp:uniform}) & 31.57 & \best{32.02} & 32.18 & 8/9 \\
             & DPIR, $4$ it.  & 31.81 & \best{32.02} & 32.21 & 6/9 \\
             & DPIR, $16$ it. & 30.95 & \best{31.51} & 31.59 & 9/9 \\
             & DiffPIR        & 29.58 & \best{30.04} & 30.06 & 9/9 \\
\midrule
PSF field    & DPIR, $8$ it.\ (Table~\ref{tab:exp:nu}) & 23.85 & \best{32.41} & 35.52 & \\
             & DPIR, $4$ it.  & 24.35 & \best{32.43} & 35.40 & \\
             & DPIR, $16$ it. & 23.39 & \best{32.28} & 35.49 & \\
             & DiffPIR        & 22.74 & \best{30.11} & 31.63 & \\
\bottomrule
\end{tabular}
\end{table}

\subsection{Comparison with cycle consistency}
\label{sec:appendix:exp:cycle}

We consider the variant of Section~\ref{sec:exp:ablations}, where $G$ is
replaced by the cycle loss $\|A_\theta R_\theta(Y)-Y\|^2$. A criterion of this
kind is used by~\citet{sanghvi2022photon} for single images, where the kernel
is updated by back-propagating an unsupervised loss through a non-blind solver.
The restorer
$R_\theta$ (DPIR with $K_{\rm bp}=1$), the warm-up, and the balancing rule are
the same as in our method. The two branches are still coupled through
$A_\theta$, without any inverse network, and only the clean-side criterion
changes.

The gradient of the cycle loss replaces $g_G$ in the update of
Eq.~(\ref{eq:method:balance}). Recall that the balancing rule of this equation
rescales this gradient whenever its norm is larger than that of $g_F$, so that
its contribution to the update never has a norm larger than $\lambda\|g_F\|$.
Since the scale of the cycle loss is arbitrary, we multiply it by $10^4$, which
makes the norm of its gradient larger than that of $g_F$ at every step. After
rescaling, the contribution of the cycle term to the update thus always has
norm $\lambda\|g_F\|$, which is the largest value allowed by the rule. With
$G$, the rescaling does not occur at every step: after the warm-up, the norm of
its contribution is on average between $0.79$ and $0.96$ times that of $g_F$
for \texttt{mbcomplex}, the $\sigma=1$ Gaussian, and the PSF field. The cycle
term and $G$ thus enter the update with comparable magnitudes, and neither
requires tuning a weight.

Table~\ref{tab:app:cycle} reports the results for each setting. The kernels
estimated with the cycle loss blur less than the true ones. On the $\sigma=1$
Gaussian for instance, the estimated kernel has half of its mass on the central
pixel, against $16\%$ for the true kernel, and its frequency response is larger
than the true one at high frequencies. With an operator that blurs less than
the true one, the solver does not deconvolve enough, and some blur remains in
the restoration. This is consistent with the results on the PSF field, where
the cycle variant has the largest kernel error of the three methods and a worse
LPIPS than our method, although its PSNR is similar.
Finally, the other cycle loss, $\|R_\theta(A_\theta X+\varepsilon)-X\|^2$,
would favor operators that are close to the identity even more directly, since
restoration then reduces to denoising.

\begin{table}[h]
\revcolor
\caption{\textbf{Cycle consistency in place of $G$.} Kernel error
($\times10^{-4}$), PSNR, and LPIPS for measurement-only matching (Meas.), our
method (Ours), and the variant where $G$ is replaced by the cycle loss
(Cycle).}
\label{tab:app:cycle}
\centering
\footnotesize
\setlength{\tabcolsep}{3.5pt}
\begin{tabular}{lrrrcrrrcrrr}
\toprule
& \multicolumn{3}{c}{Kernel error $\downarrow$} & & \multicolumn{3}{c}{PSNR $\uparrow$} & & \multicolumn{3}{c}{LPIPS $\downarrow$} \\
\cmidrule(lr){2-4}\cmidrule(lr){6-8}\cmidrule(lr){10-12}
Degradation & Meas. & Ours & Cycle & & Meas. & Ours & Cycle & & Meas. & Ours & Cycle \\
\midrule
Motion \texttt{mb}          & 3.49 & \best{1.74} & 5.87 & & 31.23 & \best{31.63} & 31.06 & & 0.087 & \best{0.077} & 0.086 \\
Motion \texttt{mb2}         & \best{2.02} & 2.94 & 7.92 & & 34.12 & \best{34.56} & 33.82 & & 0.040 & \best{0.035} & 0.056 \\
Motion \texttt{mbcomplex}   & 3.67 & \best{1.71} & 5.66 & & 31.64 & \best{32.18} & 31.06 & & 0.082 & \best{0.070} & 0.084 \\
Gaussian $\sigma{=}1$       & 1.74 & \best{0.87} & 8.89 & & 33.19 & \best{34.75} & 34.14 & & 0.049 & \best{0.047} & 0.091 \\
Gaussian $\sigma{=}3$       & \best{1.20} & 1.25 & 11.42 & & \best{29.23} & 29.21 & 26.50 & & \best{0.206} & 0.220 & 0.388 \\
Defocus disk                & 2.50 & \best{2.16} & 12.35 & & 30.39 & \best{30.76} & 26.66 & & 0.143 & \best{0.129} & 0.336 \\
Anisotropic Gaussian        & 1.77 & \best{1.47} & 10.71 & & 30.88 & \best{31.05} & 28.48 & & 0.134 & \best{0.132} & 0.261 \\
Per-channel anisotropic     & 1.71 & \best{0.95} & 10.17 & & 31.87 & \best{32.26} & 29.37 & & \best{0.115} & 0.117 & 0.233 \\
Chromatic aberration        & 1.37 & \best{0.99} & 10.57 & & 31.54 & \best{31.77} & 29.43 & & \best{0.130} & 0.131 & 0.242 \\
\midrule
Mean (uniform)              & 2.16 & \best{1.56} & 9.29 & & 31.57 & \best{32.02} & 30.06 & & 0.110 & \best{0.106} & 0.197 \\
\midrule
Measured PSF field          & 8.68 & \best{5.52} & 10.99 & & 23.85 & \best{32.41} & 32.30 & & 0.207 & \best{0.132} & 0.186 \\
\bottomrule
\end{tabular}
\end{table}

\subsection{Effect of specifying a wrong noise level}
\label{sec:appendix:exp:noise}

In practice, the noise level can be estimated with high accuracy
(Appendix~\ref{sec:appendix:exp:noise-est}): two of the three standard
estimators that we tested recover the level of our synthetic measurements
within $1\%$, and the third one within $5\%$; the noise of a real camera can be
calibrated from its raw files. In this section, we nevertheless study the
effect of giving a wrong noise level to the methods.

We consider the $\sigma=1$ Gaussian kernel, for which the true noise level
is $0.02$, and we run both methods with an assumed level of $0.01$ or $0.04$. This
level is used wherever the noise level is needed: to generate the synthetic
measurements of $F$, in the DPIR restorer of $G$, and for the final
restoration. Table~\ref{tab:app:noise} reports the kernel error and two PSNR
values. The first one is obtained when restoring with the assumed level, as a
user would do. The second one is obtained when restoring with the true level,
which isolates the effect of the learned kernel. Since a wrong level also
degrades DPIR when it is run with the true operator, the table reports this
case as a reference. With the true level, the true operator gives $34.85$~dB.

\begin{table}[h]
\revcolor
\caption{\textbf{Effect of a wrong noise level} on the $\sigma=1$ Gaussian kernel
(true level $0.02$). Kernel error, and PSNR when restoring with the assumed
noise level or with the true one. \emph{True op.} uses the ground-truth
operator, with DPIR run at the assumed level. The second row uses the level
estimated from the training measurements
(Appendix~\ref{sec:appendix:exp:noise-est}).}
\label{tab:app:noise}
\centering
\small
\setlength{\tabcolsep}{3.5pt}
\begin{tabular}{lrrrcrrrcr}
\toprule
& \multicolumn{3}{c}{Measurement-only} & & \multicolumn{3}{c}{Ours} & & True op. \\
\cmidrule(lr){2-4}\cmidrule(lr){6-8}\cmidrule(lr){10-10}
Assumed $\sigma$ & Kernel err. & PSNR & PSNR, true $\sigma$ & & Kernel err. & PSNR & PSNR, true $\sigma$ & & PSNR \\
\midrule
$0.02$ (true)          & 1.74 & 33.19 & 33.19 & & \best{0.87} & \best{34.75} & \best{34.75} & & 34.85 \\
$0.02002$ (estimated)  & 1.74 & 33.19 & 33.19 & & \best{0.99} & \best{34.68} & \best{34.68} & & 34.85 \\
$0.01$ ($\times0.5$)   & \best{3.45} & 24.41 & 33.35 & & 8.20 & \best{29.97} & \best{34.33} & & 26.17 \\
$0.04$ ($\times2$)     & 2.84 & 30.50 & 31.66 & & \best{1.50} & \best{31.50} & \best{33.67} & & 31.91 \\
\bottomrule
\end{tabular}
\end{table}

When the noise level is overestimated, most
of the loss of restoration quality comes from the solver: the true operator
loses $2.9$~dB, and our restorations remain within $0.41$~dB of those obtained
with it, against $1.41$~dB for measurement-only matching. Our kernel also
remains more accurate than the one obtained by measurement-only matching with
the correct level. When the noise level is underestimated, DPIR does not
regularize enough, and it loses $8.7$~dB with the true operator. The
restorations of real measurements then contain amplified noise, which $G$
penalizes by favoring operators that deconvolve less. Our kernel is thus
degraded more than that of measurement-only matching. Since such a kernel
partly compensates for the lack of regularization, our restorations then have
a higher PSNR than those obtained with the true operator. In all cases, our
restorations remain better than those of measurement-only matching under the
same error. When the noise level is uncertain, a larger value is thus the
safer choice. Finally, with the level estimated from the training
measurements, $0.02002$, measurement-only matching is unchanged and the results
of our method are within the variation across seeds of
Table~\ref{tab:app:seeds} (second row of Table~\ref{tab:app:noise}).

\subsection{Single-image blind deblurring}
\label{sec:appendix:exp:blinddps}
We compare with BlindDPS~\citep{chung2023blinddps} on the first $100$ test
images of the $\sigma=1$ Gaussian and motion (\texttt{mb}) settings. We use its
official code and pretrained models with $1000$ steps and $64\times64$ kernels,
and the kernel regularizer of its deblurring script. BlindDPS estimates a
kernel and restores each image from this image alone. For measurement-only
matching, our method, and the true operator, the same images are restored with
DPIR as in Table~\ref{tab:exp:uniform-summary}, with independent noise
realizations of the same level. When a BlindDPS sample diverged (non-finite
values, or a kernel with non-positive mass), we restarted it with a new seed,
at most twice. This was needed for $32$ and $13$ of the $100$ images,
respectively, and $3$ and $1$ samples still diverged after two restarts. We
keep these samples, and removing them changes the mean PSNR of BlindDPS by less
than $0.05$~dB. Since blind methods recover the image up to a translation, we
evaluate every method on the image interior, after removing a border of $4$
pixels and applying the integer shift of at most $4$ pixels that maximizes the
PSNR. Without this alignment, BlindDPS loses $0.5$ and $2.2$~dB, whereas the
other methods change by less than $0.1$~dB. The results are reported in
Table~\ref{tab:app:blinddps}. Our method is better than BlindDPS on each of the
$100$ images for both blurs, by $4.7$ and $7.7$~dB on average. Note that
BlindDPS samples from a posterior distribution, and that it improves the
perceptual quality (LPIPS) much more than the PSNR.

\begin{table}[h]
\revcolor
\caption{\textbf{Single-image blind deblurring.} PSNR and LPIPS on the first
$100$ test images, after alignment.}
\label{tab:app:blinddps}
\centering
\small
\begin{tabular}{lrrcrr}
\toprule
& \multicolumn{2}{c}{Gaussian $\sigma=1$} & & \multicolumn{2}{c}{Motion \texttt{mb}} \\
\cmidrule(lr){2-3}\cmidrule(lr){5-6}
Method & PSNR $\uparrow$ & LPIPS $\downarrow$ & & PSNR $\uparrow$ & LPIPS $\downarrow$ \\
\midrule
Blurred input     & 30.05 & 0.150 & & 22.86 & 0.417 \\
BlindDPS          & 30.05 & 0.094 & & 24.05 & 0.157 \\
Measurement-only  & 33.13 & 0.048 & & 31.32 & 0.084 \\
Ours              & \best{34.79} & \best{0.046} & & \best{31.77} & \best{0.074} \\
True operator     & 34.94 & 0.045 & & 31.86 & 0.076 \\
\bottomrule
\end{tabular}
\end{table}

\subsection{Estimation and calibration of the noise level}
\label{sec:appendix:exp:noise-est}

Our method takes the noise level as an input, and
Appendix~\ref{sec:appendix:exp:noise} shows that an underestimated level
degrades the learned operator. In practice, this level does not have to be set
by hand: it can be estimated from the measurements and even calibrated on the sensor for a real camera.

\paragraph{Gaussian noise.}
Estimating the level of white Gaussian noise from a single image is a classical
problem, for which standard estimators report errors of a few percent on natural
images~\citep{liu2013single,pyatykh2013image}. Blurred measurements are a
favorable case, since blur leaves little image content at the highest
frequencies. We apply three estimators to the training measurements of our ten
deblurring settings: the Laplacian-based filter
of~\citet{immerkaer1996fast}, the median absolute deviation of the finest
wavelet coefficients~\citep{donoho1994ideal}, and the principal component
analysis of patches of~\citet{chen2015efficient}.
Table~\ref{tab:app:noise-est} reports the median of each estimate over the
measurements of a setting. The first and the third estimators recover the level
within $1\%$ in all settings, and the second one within $5\%$. On the
$\sigma=1$ Gaussian kernel, we then train both methods with the level given by
the first estimator, $2.002\cdot10^{-2}$, used wherever the level is needed
(Table~\ref{tab:app:noise}). Measurement-only matching is unchanged, and the
results of our method are within the variation observed across seeds with the
true level (Table~\ref{tab:app:seeds}).

\begin{table}[h]
\revcolor
\caption{\textbf{Noise level estimated from the training measurements.} Median
over the measurements of a setting, for the nine uniform blurs (true level
$0.02$; range over the settings) and for the PSF field (true level $0.01$), and
largest relative error over the ten settings (computed before rounding).}
\label{tab:app:noise-est}
\centering
\small
\begin{tabular}{lccc}
\toprule
Estimator & Uniform blurs & PSF field & Largest error \\
\midrule
\citet{immerkaer1996fast} & $0.0200$ to $0.0201$ & $0.0100$ & $0.7\%$ \\
\citet{donoho1994ideal}   & $0.0200$ to $0.0206$ & $0.0105$ & $4.8\%$ \\
\citet{chen2015efficient} & $0.0199$ to $0.0201$ & $0.0100$ & $0.5\%$ \\
\bottomrule
\end{tabular}
\end{table}

\paragraph{Real camera.}
Calibrating the noise of a sensor from raw data is a standard procedure: it is covered by an
industry standard~\citep{emva1288}, the DNG raw format stores the two parameters
of the noise model~\citep{dng2023}, the benchmarks for real-noise removal rely
on it~\citep{plotz2017benchmarking,abdelhamed2018sidd}, and it can be done from
a single raw image~\citep{foi2008practical}. For the capture of Section~\ref{sec:exp:real}, we proceed in two
steps from the fifteen raw files. (i)~In each color plane of the raw mosaic,
the noise variance is modeled as an affine function of the intensity, which
accounts for shot and read noise. We fit this function on smooth regions, jointly for all
frames. (ii)~Our method takes developed frames (demosaicing, white balance,
brightness scaling, and tone curve), in which the noise is no longer white, so
that the estimators above do not apply. We therefore pass synthetic raw data
with the fitted noise model through the same development and measure the
standard deviation of the result. It depends only mildly on the intensity,
since the tone curve compensates for most of the signal dependence: it stays
within $25\%$ of its average for at least $89\%$ of the pixels of each f/16
frame. This average, a root mean square over the pixels of a frame, has a
median of $0.007$ to $0.008$ over the seven frames, depending on the color
channel, and it is below $0.01$ for every frame. We therefore use a level of
$8\cdot10^{-3}$ on this capture, for the forward model and for the final
restorations.

\subsection{Comparison with diffraction theory on the real capture}
\label{sec:appendix:exp:real-mtf}

No ground-truth PSF is available for the lens of Section~\ref{sec:exp:real} because
measuring one is a complex procedure requiring dedicated hardware~\citet{bauer2018automatic}. The amount of blur nevertheless has a
physical reference, because the blur of this lens at f/16 is dominated by
diffraction, which physics predicts in closed form. In this section, we compare the learned
kernels with this prediction.

\paragraph{Measuring the amount of blur.}
A standard way to quantify the blur of a PSF $k$ is its modulation transfer
function (MTF), the modulus $|\hat k(\nu)|$ of its Fourier transform. We express $\nu$ in cycles per
pixel: a frequency of $0.1$ corresponds to stripes with a period of $10$
pixels, and $0.5$ is the highest frequency that an image can represent. Since a
PSF sums to one, its MTF equals $1$ at $\nu=0$. Since the PSF pattern may also have an
orientation, we obtain a
single value per frequency by averaging over orientations. For $72$ angles
$\theta$ equally spaced, we evaluate the Fourier transform of
the PSF at the two-dimensional frequency $(\nu\cos\theta,\nu\sin\theta)$,
\begin{equation*}
\hat k(\nu,\theta)=\sum_{x,y}k[x,y]\,e^{-2i\pi\nu(x\cos\theta+y\sin\theta)},
\end{equation*}
where the sum is over the pixels of the kernel, and we average the $72$ moduli
$|\hat k(\nu,\theta)|$. In the following, the MTF of a kernel at frequency
$\nu$ refers to this average. Finally, we summarize an MTF by the frequency at
which it falls to $0.5$, called MTF50: the larger the MTF50, the sharper the
PSF. For instance, a Gaussian PSF with a standard deviation of $\sigma$ pixels
has an MTF50 of $0.19/\sigma$.

\paragraph{Prediction of diffraction theory.}
Diffraction limits the sharpness of any lens: even without aberrations, a lens
with a circular aperture images a point as a small spot, whose size is
proportional to the f-number $N$, the ratio between the focal length and the
diameter of the aperture. A lens whose blur is due to diffraction alone is said
to be diffraction-limited. The MTF of its PSF is the same for all orientations
and is known in closed form:
\begin{equation*}
M_N(\nu)=\tfrac{2}{\pi}\bigl(\arccos s-s\sqrt{1-s^2}\bigr)
\quad\text{with}\quad s=\min(\lambda_{\rm w} N\nu/p,\,1),
\end{equation*}
where $\lambda_{\rm w}$ is the wavelength of the light and $p$ the pixel pitch,
that is, the distance between two adjacent pixels of the sensor. Our operator
maps f/5.6 images to f/16 images, taken with the same sensor and developed in
the same way. Since the Fourier transform turns a convolution into a product,
the MTF of the camera at f/16 is the product of its MTF at f/5.6 and of the MTF
of our operator. The contributions of the sensor and of the demosaicing are the
same at both apertures and cancel, so that only the lens matters. If the lens
is diffraction-limited at both apertures, the MTF of our operator should thus
be close to $M_{16}/M_{5.6}$. At f/5.6, we assume this in the center of the
field only, where the sharp crops are taken. We evaluate this ratio with the
nominal f-numbers, $p=3.34\,\mu$m, and an average over the wavelengths of green
light ($500$ to $580$~nm), and we compare it with the MTF of the green channel
of the learned kernels.

\begin{table}[h]
\revcolor
\caption{\textbf{MTF of the kernels learned on the real capture, against
diffraction theory} (green channel). The five central columns give the MTF at
five spatial frequencies, that is, the fraction of the contrast that the blur
leaves to sinusoidal stripes with periods of $20$, $10$, $6.7$, $5$ and $4$
pixels ($1$: contrast preserved, $0$: stripes erased). For the learned kernels,
it is averaged over the $108$ kernels of the field. MTF50 is the frequency, in
cycles per pixel, at which the MTF falls to $0.5$: the larger, the sharper. The
last column is the share of the $108$ kernels whose own MTF50 is within $20\%$
of the theoretical one.}
\label{tab:app:real-mtf}
\centering
\small
\setlength{\tabcolsep}{4.5pt}
\begin{tabular}{llccccccc}
\toprule
& & \multicolumn{5}{c}{MTF at frequency $\nu$ (cycles/pixel)} & & Kernels \\
\cmidrule(lr){3-7}
Method & Noise level & $0.05$ & $0.10$ & $0.15$ & $0.20$ & $0.25$ & MTF50 & within $20\%$ \\
\midrule
Diffraction theory & --- & 0.89 & 0.76 & 0.63 & 0.48 & 0.33 & 0.194 & --- \\
\midrule
Measurement-only & $2\cdot10^{-2}$ & 0.71 & 0.39 & 0.30 & 0.20 & 0.10 & 0.075 & $6\%$ \\
Ours             & $2\cdot10^{-2}$ & 0.88 & 0.70 & 0.58 & 0.44 & 0.29 & 0.179 & $81\%$ \\
\midrule
Measurement-only & $8\cdot10^{-3}$ (calibrated) & 0.72 & 0.42 & 0.36 & 0.27 & 0.17 & 0.079 & $26\%$ \\
Ours             & $8\cdot10^{-3}$ (calibrated) & 0.88 & 0.71 & 0.60 & 0.48 & 0.34 & \best{0.191} & \best{89\%} \\
\bottomrule
\end{tabular}
\end{table}

\paragraph{Results.}
The first row of Table~\ref{tab:app:real-mtf} is the prediction
$M_{16}/M_{5.6}$, and the other rows are the MTF of the learned kernels. With
the calibrated noise level (Appendix~\ref{sec:appendix:exp:noise-est}), the MTF
of our kernels follows the prediction at the five frequencies. Their MTF50 is
within $2\%$ of the theoretical value, and $89\%$ of the $108$ kernels are
individually within $20\%$ of it. The kernels of measurement-only matching
attenuate stripes much more than predicted: for a period of $10$ pixels, they
keep $42\%$ of the contrast instead of $76\%$, and their MTF50 is $2.5$ times
too low. In terms of width, the theoretical MTF50 is that of a Gaussian PSF
with a standard deviation of about $1$ pixel, and the MTF50 of measurement-only
matching that of a Gaussian of $2.4$ pixels: its kernels are much too blurry,
which is consistent with Figure~\ref{fig:appendix:real}. With a level set by
hand to $2\cdot10^{-2}$, more than twice the calibrated one, the conclusions
are the same and the MTF50 of our kernels is $8\%$ below the theoretical value.
This agrees with Appendix~\ref{sec:appendix:exp:noise}: overestimating the
noise level is the benign direction.
Note that this comparison measures the amount of blur only, not the chromatic shifts between channels, for which
we have no reference.
}

\end{document}